\documentclass[10pt,conference,letterpaper]{IEEEtran}
\IEEEoverridecommandlockouts

\usepackage{times}
\usepackage{latexsym}

\usepackage[T1]{fontenc}
\usepackage[utf8]{inputenc}

\usepackage{microtype}
\usepackage{inconsolata}

\usepackage{graphicx}

\usepackage{mathtools}
\usepackage{amssymb}
\usepackage{amsthm}

\usepackage{siunitx}
\usepackage{threeparttable}
\usepackage{booktabs}
\usepackage{multirow}
\usepackage{tabularx}
\usepackage{array}
\usepackage{xcolor}
\usepackage{colortbl}
\usepackage{placeins}

\usepackage{float}
\usepackage{afterpage}
\usepackage{textcomp}
\usepackage{gensymb}
\usepackage{enumitem}
\usepackage{subcaption}

\usepackage{algorithm}
\usepackage[noend]{algpseudocode}

\newtheoremstyle{boldstyle}
  {\topsep}
  {\topsep}
  {\itshape}
  {}
  {\bfseries}
  {.}
  {.5em}
  {}

\definecolor{deepblue}{HTML}{355C7D}     
\definecolor{pinkaccent}{HTML}{F67280}   
\definecolor{darkgrayx}{HTML}{4C4C4C}    
\definecolor{steelblue}{HTML}{6C8EAD}    
\definecolor{lightsteel}{HTML}{A7C5DD}   

\colorlet{titlecol}{deepblue}
\colorlet{boxborder}{deepblue!75!black}
\colorlet{arrowcol}{deepblue}
\colorlet{accentcol}{pinkaccent!85!black}
\colorlet{codegray}{darkgrayx}

\colorlet{steponebg}{lightsteel!35}
\colorlet{steptwobg}{steelblue!18}
\colorlet{stepthreebg}{pinkaccent!12}
\colorlet{stepfourbg}{darkgrayx!10}

\usepackage{tikz}
\usepackage{xcolor}
\usepackage{amsmath}
\usepackage{enumitem}
\usetikzlibrary{positioning,fit,calc,backgrounds,arrows.meta,shapes.geometric}

\theoremstyle{boldstyle}
\newtheorem{theorem}{Theorem}
\newtheorem{lemma}[theorem]{Lemma}

\newtheorem{definition}[theorem]{Definition}

\newtheorem{assumption}{Assumption}
\newtheorem{corollary}{Corollary}

\usepackage{url}    
\newcolumntype{L}[1]{>{\raggedright\arraybackslash}m{#1}}
\newcolumntype{C}[1]{>{\centering\arraybackslash}m{#1}}
\newcolumntype{R}[1]{>{\raggedleft\arraybackslash}m{#1}}

\algrenewcommand\algorithmicrequire{\textbf{Require:}}
\algrenewcommand\algorithmicensure{\textbf{Ensure:}}
\usepackage{hyperref}
\hypersetup{
  colorlinks=true,
  linkcolor=black,
  citecolor=blue,
  urlcolor=blue,
  pdfborder={0 0 0}
}

\def\BibTeX{{\rm B\kern-.05em{\sc i\kern-.025em b}\kern-.08em
  T\kern-.1667em\lower.7ex\hbox{E}\kern-.125emX}}

\title{
Converse and Collision-Based Achievability for Node Localization with Hybrid Distance-Spectral Graph Positional Encodings
}

\author{%
{Zimo Yan\textsuperscript{\rm 1},
Yifan Li\textsuperscript{\rm 1},
Hao Li\textsuperscript{\rm 1},
Zheng Xie\textsuperscript{\rm 1}\thanks{Corresponding author: Zheng Xie (xiezheng81@nudt.edu.cn).},
Chang Liu\textsuperscript{\rm 1},
Zheming Tu\textsuperscript{\rm 1},
Yuan Wang\textsuperscript{\rm 2}}%
\vspace{1.6mm}\\
\fontsize{10}{10}\selectfont\itshape
\textsuperscript{\rm 1}National University of Defense Technology, Changsha, China\\
\textsuperscript{\rm 2}Wuhan University, Wuhan, China\\
\{yanzimo20, liyifan25, lihao22, xiezheng81, liuchang\_\}@nudt.edu.cn,
tzm\_nudt@163.com\\
2024282090042@whu.edu.cn
}

\fontsize{9}{9}\selectfont\ttfamily\upshape
\fontsize{10}{10}\selectfont\rmfamily\itshape

\fontsize{9}{9}\selectfont\ttfamily\upshape

\begin{document}
\maketitle

\begin{abstract}
Graph positional encodings are widely used in graph neural
networks and graph Transformers, yet it remains unclear when the code
itself can identify nodes.  We study a hybrid distance-spectral encoding
that combines anchor-distance profiles with quantized low-frequency
Laplacian-energy coordinates.  Treating the encoding as an observation
map yields a simplex-refined converse, an exact collision factorization
\(\kappa_H=\kappa_D\kappa_{S|D}\), and the collision information
\(I_H=-\log\kappa_D-\log\kappa_{S|D}\).  On random regular graphs, the
criterion is made explicit through a bounded-correlation Gaussian-wave
surrogate; for actual Laplacian-energy coordinates, we give the
distance-conditioned spectral collision condition sufficient for
conditional actual-coordinate achievability.  Experiments show that
\(I_H/\log n\) calibrates localization success, and PE-only structural
task probes on Universal Dependencies trees show that hybrid encodings
better recover syntactic-tree geometry than distance-only or
spectral-only baselines.
\end{abstract}

\section{Introduction}

Graph neural networks and graph Transformers are widely used for relational
data, including graph-based NLP and symbolic reasoning settings where nodes
represent tokens, entities, or reasoning states. Their effectiveness often
depends on whether such nodes can be distinguished by positional or structural
information in the underlying graph. This has made graph positional encoding a
central topic in modern graph learning
\cite{dwivedi2023benchmark,ying2021graphormer},
with connections to dependency-based NLP and text-based relational reasoning
\cite{marcheggiani2017encoding,sinha2019clutrr,demarneffe2021universal}.

Existing positional encodings can be roughly grouped into three classes.
Spectral methods use Laplacian eigenvectors, eigenvalues, or invariant spectral
features to encode global graph geometry
\cite{lim2023signnet}. Distance-based methods describe nodes
through shortest-path distances, anchor profiles, or related structural
descriptors
\cite{li2020distance,dwivedi2022learnable}. Learned and hybrid encoders further
combine multiple positional signals inside graph Transformer or GNN
architectures
\cite{rampasek2022gps,canturk2024gpse}.

However, downstream performance does not directly reveal how much
identifiability is supplied by the positional code itself. Model architecture,
optimization, node attributes, edge labels, and task-specific correlations may
all obscure the intrinsic resolution of an encoding. Spectral encodings also
face sign, basis, and stability issues, while distance encodings may collide
when different nodes share the same anchor-distance profile
\cite{huang2024stability}. Thus, a graph-dependent theory is needed to explain
when positional encodings can localize nodes, when ambiguity is
unavoidable, and which spectral collision bounds suffice for
actual-coordinate achievability.

\textbf{Motivation.}
This raises the central question: \textit{when does a hybrid graph positional
encoding contain enough information to localize nodes?}

We answer this question at the level of the encoding itself. The hybrid code
is treated as an observation map whose fibers are exactly the sets of nodes
that remain indistinguishable from positional information alone. This view separates intrinsic identifiability from downstream
architectures and leads to a unified theory of converse bounds,
collision-based achievability, and spectral collision bounds conditioned on distance collisions. Figure~\ref{fig:framework-overview} summarizes the hybrid code construction and the resulting two-sided localization framework.

\textbf{Contributions.}
The main contributions are as follows:
\begin{enumerate}
    \item We formulate node localization with hybrid distance-spectral positional encodings as an encoding-level observation-map problem. This viewpoint separates intrinsic positional identifiability from downstream architectures, attributes, labels, and optimization effects.

    \item The resulting theory provides a two-sided localization criterion. A simplex-refined converse rules out localization below the $\log n$ code-budget scale, while $\kappa_H=\kappa_D\kappa_{S|D}$ yields a collision-information condition for vanishing error. On random regular graphs, Gaussian-wave anti-concentration gives a closed surrogate result and a sufficient condition for conditional achievability with actual Laplacian coordinates.

    \item Synthetic diagnostics and PE-only probes on Universal Dependencies support the theory. The quantity $I_H/\log n$ tracks localization success across graph families, while hybrid encodings recover dependency-tree geometry better than distance-only or spectral-only baselines; derangement controls confirm the importance of node-code alignment.
\end{enumerate}
\begin{figure*}[t]
    \centering
    \includegraphics[width=\textwidth]{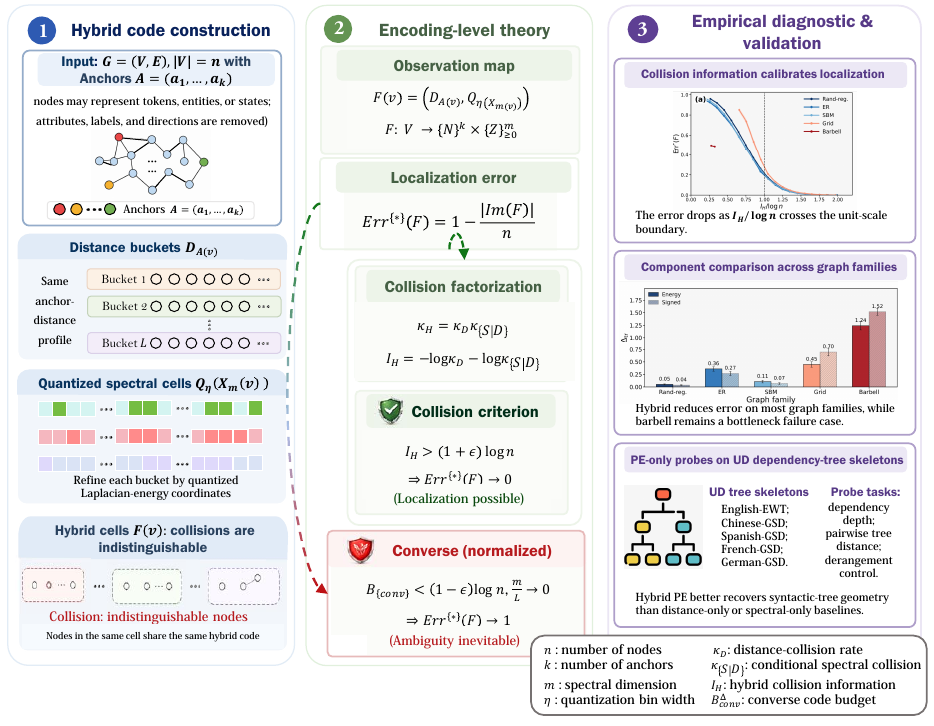}
    \caption{Overview of the hybrid positional encoding and the two-sided localization framework.}
    \label{fig:framework-overview}
\end{figure*}

\section{Literature Review}
\label{SE2}

Graph positional and structural encodings augment graph neural networks
and graph Transformers with information that cannot always be recovered
from local message passing alone. This motivation is closely related to
the expressivity limitations of message-passing GNNs under the
Weisfeiler-Leman framework
\cite{xu2019powerful,morris2019weisfeiler,morris2023wlstory}.
Such encodings are also important in graph-based NLP and relational reasoning,
where vertices may represent tokens, entities, dependency nodes, or reasoning
states \cite{marcheggiani2017encoding,sinha2019clutrr,demarneffe2021universal}.
Existing methods can be broadly grouped into spectral encodings,
distance- and structure-aware encodings, and learned or hybrid
positional/structural encoders, while related theory studies their effects on
expressivity, stability, generalization, and node distinguishability.

\subsection{Spectral Positional Encodings}

Spectral positional encodings use eigenvalues, eigenvectors, or functions of
graph operators to represent global geometry, with foundations in Laplacian
eigenmaps and spectral graph representation \cite{belkin2003laplacian}.
Recent graph-learning models incorporate Laplacian features into graph
Transformers and GNNs, while diffusion-wavelet embeddings and random feature
propagation provide related spectral or propagation-based structural signals
\cite{kreuzer2021rethinking,dwivedi2022learnable,donnat2018graphwave,
eliasof2023rfp}. Raw Laplacian eigenvectors suffer from sign ambiguity and
basis rotations in repeated eigenspaces; SignNet, BasisNet, PEG, Laplacian
Canonization, and stable spectral encodings address these issues through
invariance, equivariance, or perturbation-stability mechanisms
\cite{lim2023signnet,wang2022peg,ma2023laplacian,huang2024stability}.
Beyond vertex eigenvectors, CycleNet uses projector-based Hodge-Laplacian
cycle-space representations for basis-invariant structural encoding
\cite{yan2024cyclenet}. These methods provide rich global signals, but they do
not by themselves characterize when finite-dimensional and quantized spectral
codes uniquely identify vertices.

\subsection{Distance- and Structure-aware Encodings}

Distance-based encodings represent nodes through relative topology, including
shortest-path distances, anchor-distance profiles, random-walk statistics, and
structural attention biases. Position-aware GNNs and Distance Encoding use
anchor or distance features to improve node representations and expressive
power \cite{you2019position,li2020distance}, while Graphormer, GRPE, and
rewiring methods inject structural information into attention or graph
connectivity \cite{ying2021graphormer,park2022grpe,bruelgabrielsson2023rewiring}.
Other methods, such as struc2vec and the Structure-Aware Transformer, model
structural similarity or rooted subgraph structure without requiring nodes to
be close in the graph \cite{ribeiro2017struc2vec,chen2022sat}. Although these
encodings are interpretable, their resolution depends on anchor placement,
graph symmetry, receptive-field design, and the number of distinct structural
profiles induced on the graph.

\subsection{Learned and Hybrid Positional/Structural Encoders}

Learned and hybrid encoders combine multiple positional signals inside
trainable graph architectures. GraphGPS separates local message passing,
global attention, and positional or structural encodings, GRIT incorporates
graph inductive biases into Transformers, and GPSE learns transferable
positional and structural representations
\cite{rampasek2022gps,ma2023grit,canturk2024gpse}. Related graph
Transformers, including TokenGT, Exphormer, and direction-aware Transformers,
represent nodes and edges as tokens, use sparse expander-based attention, or
employ magnetic-Laplacian and directional random-walk encodings
\cite{kim2022tokengt,shirzad2023exphormer,geisler2023directed}. These
architectures are empirically effective, but their performance reflects the
joint effects of the encoding, model capacity, attention mechanism,
optimization objective, and dataset. Thus, downstream accuracy alone does not
show whether the positional code itself contains enough information to identify
nodes.

\subsection{Positional Identifiability and Graph Localization}

Identifying vertices from distances has a long history in graph theory through
resolving sets, landmarks, and metric dimension
\cite{slater1975leaves,harary1976metric,khuller1996landmarks,
chartrand2000resolvability}. A resolving set uniquely identifies every vertex
by its vector of distances to selected reference vertices, and related work
studies metric dimension in graph families and random graph models, including
Erd\H{o}s-R\'enyi graphs \cite{bollobas2013metric}. Graph reconstruction from
distance queries provides a nearby perspective, with results for random
regular graphs \cite{mathieu2021simple}. In graph learning, Weisfeiler-Leman
analyses and recent studies of graph Transformers clarify how structural
distinctions, positional encodings, and generalization interact
\cite{xu2019powerful,morris2019weisfeiler,morris2023wlstory,
keriven2023role,black2024comparing,li2024generalization}. These works motivate
an encoding-level analysis, but they do not directly characterize
finite-resolution hybrid codes that combine anchor distances with quantized
spectral coordinates.

\subsection{Relation to Closest Prior Work}

This work is closest to theoretical analyses of hybrid
distance-spectral positional encodings and to graph Transformer
positional-encoding methods. A recent information-theoretic study formulated
node localization through observation maps and established one-sided
image-size converse bounds for positional ambiguity
\cite{yan2026informationtheoreticlimitsnodelocalization}. Building on this identifiability viewpoint, we
address the complementary positive question of when distance and spectral code
collisions vanish fast enough to enable localization. Compared with that line of work, the new contribution is the
collision-achievability framework: we turn the factorization
\(\kappa_H=\kappa_D\kappa_{S\mid D}\) into a deterministic
criterion, the information measure \(I_H\), a closed
Gaussian-wave achievability theorem for random regular
graphs, and a a distance-conditioned spectral collision condition
for actual Laplacian-energy coordinates. Unlike architecture-level studies, we characterize the information in
the positional code itself rather than the performance of a particular
downstream model.  The UD experiments are therefore PE-only structural
task probes: they connect intrinsic localization to syntactic geometry,
but are not full dependency parsing or language-model fine-tuning. As summarized in Fig.~\ref{fig:framework-overview}, our analysis isolates the positional code itself and connects its construction to both converse and collision-based achievability results.

\section{Problem Formulation}
\label{SE3}

Let \(G=(V,E)\) be a finite connected simple undirected graph with
\(n=|V|\), and let \(d_G\) denote shortest-path distance. We study an
encoding-level localization problem: only graph structure is observed,
while node attributes, edge labels, edge directions, lexical information,
and task-specific features are removed. Therefore, vertices assigned the
same positional code are indistinguishable to any decoder using this code
alone.

Given an ordered anchor list \(\mathcal A=(a_1,\ldots,a_k)\), the
anchor-distance profile of \(v\in V\) is
\begin{equation}
D_{\mathcal A}(v)
=
\bigl(d_G(v,a_1),\ldots,d_G(v,a_k)\bigr).
\end{equation}
The realized profiles form
\begin{equation}
\mathcal T_{G,\mathcal A}
=
\{D_{\mathcal A}(v):v\in V\},
\quad
D(G,\mathcal A)=|\mathcal T_{G,\mathcal A}|,
\end{equation}
and each \(t\in\mathcal T_{G,\mathcal A}\) induces a distance bucket
\begin{equation}
B_t=\{v\in V:D_{\mathcal A}(v)=t\}.
\end{equation}

To refine these buckets, we use low-frequency Laplacian energy. Let
\(\phi_1,\ldots,\phi_n\in\mathbb R^V\) be an orthonormal eigenbasis of
the normalized graph Laplacian, ordered by nondecreasing eigenvalue. For
\(m\in\{1,\ldots,n-1\}\), set
\begin{equation}
X_m(v)
=
n\bigl(\phi_2(v)^2,\ldots,\phi_{m+1}(v)^2\bigr)
\in\mathbb R_{\ge 0}^{m}.
\end{equation}
The squared coordinates remove sign ambiguity. We use a fixed eigenbasis
convention; for repeated eigenspaces, a blockwise projector-energy version
can be used without changing the counting arguments.

For \(\eta>0\), define
\begin{equation}
Q_\eta(x_1,\ldots,x_m)
=
\bigl(\lfloor x_1/\eta\rfloor,\ldots,\lfloor x_m/\eta\rfloor\bigr),
\end{equation}
\begin{equation}
Z_{m,\eta}(v)=Q_\eta(X_m(v)).
\end{equation}
The hybrid positional observation map is
\begin{equation}
F^{(m,\eta)}_{G,\mathcal A}(v)
=
\bigl(D_{\mathcal A}(v),Z_{m,\eta}(v)\bigr).
\end{equation}

A source vertex \(v_\star\) is sampled uniformly from \(V\), and the
decoder observes only \(F^{(m,\eta)}_{G,\mathcal A}(v_\star)\). For any
positional map \(F:V\to\mathcal Y\), the optimal conditional localization
error is
\begin{equation}
\operatorname{Err}^{\star}(F)
=
\inf_{s:\operatorname{Im}(F)\to V}
\mathbb P\bigl(s(F(v_\star))\neq v_\star \mid G,\mathcal A\bigr).
\end{equation}
We ask when \(F^{(m,\eta)}_{G,\mathcal A}\) separates most vertices, and
when its preimages remain too large for reliable localization.

\section{Theoretical Analysis}
\label{SE5}

The analysis rests on two complementary views of the same observation
map. The converse side counts how many hybrid codes can be produced; the
achievability side controls how often two distinct vertices collide under
the same code. We use the observation-map identity
\begin{equation}
\operatorname{Err}^*(F)
=
1-\frac{|\operatorname{Im}(F)|}{|V|},
\label{eq:image-identity}
\end{equation}
valid for any finite \(V\), map \(F:V\to\mathcal Y\), and uniform source
vertex. Thus, small image size implies unavoidable ambiguity, while small
pairwise collision probability implies successful localization. All formal
proofs are deferred to Appendix~\ref{app:proofs}.

\subsection{A Normalized Converse}
\label{subsec:normalized-refined-converse}

The converse is a counting argument. The distance component contributes
at most \(D(G,\mathcal A)\) profiles, while the spectral component is
controlled by the total Laplacian-energy mass
\begin{equation}
\sum_{v\in V}\sum_{j=1}^{m} X_m(v)_j = mn ,
\label{eq:spectral-total-mass}
\end{equation}
which follows from orthonormality. Let
\(S_m(v):=\sum_{j=1}^{m}X_m(v)_j\). After trimming vertices
with \(S_m(v)>L\), the remaining spectral vectors lie in the
nonnegative simplex \(\{x\in\mathbb R_{\ge 0}^{m}:\sum_j x_j\le L\}\).
Their quantized codes therefore occupy only finitely many
distance-spectral cells. This gives the following image-size bound.

\begin{theorem}[Simplex-refined normalized converse]
\label{thm:normalized-refined-converse}
For any finite connected graph \(G=(V,E)\), anchor set
\(\mathcal A\subseteq V\), spectral dimension \(m\in\{1,\dots,n-1\}\),
quantization level \(\eta>0\), and trimming threshold \(L>0\),
\begin{equation}
\bigl|\operatorname{Im}(F_{G,\mathcal A}^{(m,\eta)})\bigr|
\le
D(G,\mathcal A)
\binom{\lfloor L/\eta\rfloor+m}{m}
+
\frac{mn}{L}.
\label{eq:normalized-converse-image}
\end{equation}
Consequently,
\begin{equation}
\operatorname{Err}^*
\bigl(F_{G,\mathcal A}^{(m,\eta)}\bigr)
\ge
1
-
\frac{D(G,\mathcal A)}{n}
\binom{\lfloor L/\eta\rfloor+m}{m}
-
\frac{m}{L}.
\label{eq:normalized-converse-error}
\end{equation}
\end{theorem}

This bound yields a simplex-refined subcritical principle. If, along a
graph sequence,
\begin{equation}
\log D(G_n,\mathcal A_n)
+
\log\binom{\lfloor L_n/\eta_n\rfloor+m_n}{m_n}
\le
(1-\varepsilon)\log n
\label{eq:subcritical-budget-general}
\end{equation}
for some \(\varepsilon>0\), and \(m_n/L_n\to0\), then
\(\operatorname{Err}^*(F_{G_n,\mathcal A_n}^{(m_n,\eta_n)})\to1\).
For random \(r\)-regular graphs with random anchors, the standard diameter
bound gives \(D(G_n,\mathcal A_n)\le (C_r\log n+1)^{k_n}\) with high
probability, yielding the corresponding simplex-refined random-regular
subcritical condition in Appendix~\ref{app:proofs}. The key message is
that localization is impossible when the joint anchor, spectral-dimension,
and quantization budget remains below the \(\log n\) scale.

\subsection{Collision-Based Achievability}
\label{subsec:collision-based-achievability}

We now turn to pairwise collisions. Let \(U,V\) be sampled uniformly from
\(V\) without replacement, and for any map \(F:V\to\mathcal Y\) write
\begin{equation}
\kappa_F=\mathbb P(F(U)=F(V)).
\end{equation}
For the hybrid map, a collision requires both a distance collision and a
spectral collision inside the same distance bucket. We write
\begin{align}
\kappa_{\mathrm H}
&=
\mathbb P\!\left(
F_{G,\mathcal A}^{(m,\eta)}(U)
=
F_{G,\mathcal A}^{(m,\eta)}(V)
\right),\\
\kappa_{\mathrm D}
&=
\mathbb P\!\left(
D_{\mathcal A}(U)=D_{\mathcal A}(V)
\right),\\
\kappa_{\mathrm{S}\mid\mathrm D}
&=
\mathbb P\!\left(
Z_{m,\eta}(U)=Z_{m,\eta}(V)
\,\middle|\,
D_{\mathcal A}(U)=D_{\mathcal A}(V)
\right).
\end{align}
The conditional probability in the third line is defined when
\(\kappa_{\mathrm D}>0\). When \(\kappa_{\mathrm D}=0\), no
distance-colliding pair exists, so the distance component already separates
all distinct ordered vertex pairs and necessarily
\(\kappa_{\mathrm H}=0\). In this degenerate case, we adopt the convention
\(\kappa_{\mathrm{S}\mid\mathrm D}=0\). This convention preserves the
factorization
\(\kappa_{\mathrm H}
=\kappa_{\mathrm D}\kappa_{\mathrm{S}\mid\mathrm D}\).

When \(\kappa_{\mathrm D}>0\), the conditional collision probability can be
written equivalently as
\begin{equation}
\kappa_{\mathrm{S}\mid\mathrm D}
=
\frac{
\sum_{t\in\mathcal T_{G,\mathcal A}}
\sum_{\substack{u,v\in B_t\\u\ne v}}
\mathbf 1\{Z_{m,\eta}(u)=Z_{m,\eta}(v)\}
}{
\sum_{t\in\mathcal T_{G,\mathcal A}} |B_t|(|B_t|-1)
}.
\label{eq:conditional-spectral-collision}
\end{equation}
The denominator counts all ordered pairs of distinct vertices that collide
under the anchor-distance code, whereas the numerator counts those pairs
that remain indistinguishable after spectral refinement. Thus,
\(\kappa_{\mathrm{S}\mid\mathrm D}\) measures how much quantized spectral
energy resolves the ambiguity left by anchor-distance profiles. When
\(\kappa_{\mathrm D}=0\), the ratio above is not evaluated, and the stated
zero convention is used instead.

\begin{lemma}[Collision factorization]
\label{lem:basic-collision-lemma}
For any observation map \(F:V\to\mathcal Y\) on a finite set \(|V|=n\),
\begin{equation}
\operatorname{Err}^*(F)\le (n-1)\kappa_F.
\label{eq:basic-collision-error}
\end{equation}
Moreover, the hybrid map satisfies the exact factorization
\begin{equation}
\kappa_{\mathrm H}
=
\kappa_{\mathrm D}\kappa_{\mathrm{S}\mid\mathrm D}.
\label{eq:hybrid-collision-factorization}
\end{equation}
Consequently,
\begin{equation}
\operatorname{Err}^*
\bigl(F_{G,\mathcal A}^{(m,\eta)}\bigr)
\le
(n-1)\kappa_{\mathrm D}\kappa_{\mathrm{S}\mid\mathrm D}.
\label{eq:hybrid-collision-error}
\end{equation}
\end{lemma}

\begin{theorem}[Deterministic collision-achievability criterion]
\label{thm:deterministic-collision-achievability}
For any deterministic sequence of graphs, anchor sets, spectral dimensions,
and quantization levels, if
\begin{equation}
n\kappa_{\mathrm D}\kappa_{\mathrm{S}\mid\mathrm D}\to0,
\end{equation}
then
\begin{equation}
\operatorname{Err}^*
\bigl(F_{G,\mathcal A}^{(m,\eta)}\bigr)\to0.
\end{equation}
The same conclusion holds in probability when the above quantities are
random and
\(n\kappa_{\mathrm D}\kappa_{\mathrm{S}\mid\mathrm D}\xrightarrow{\mathbb P}0\).
\end{theorem}

It is therefore natural to define the total collision information
\begin{equation}
\mathcal I_{\mathrm H}(G,\mathcal A,m,\eta)
:=
-\log\kappa_{\mathrm D}
-
\log\kappa_{\mathrm{S}\mid\mathrm D},
\label{eq:total-collision-information}
\end{equation}
with the convention \(-\log0=+\infty\). The achievability condition is then
\(\mathcal I_{\mathrm H}(G,\mathcal A,m,\eta)\ge (1+\varepsilon)\log n\),
up to lower-order terms.

\subsection{Random-Regular Achievability and Conditional Spectral Collisions}
\label{subsec:rrg-achievability-transfer}

We instantiate the deterministic collision criterion on random regular
graphs by combining a two-source distance distinguisher bound with a
bounded-correlation Gaussian-wave anti-concentration bound. Both ingredients
are proved in Appendix~\ref{app:proof-Random-Regular}. The Gaussian-wave
surrogate models normalized Laplacian coordinates, since
\begin{equation}
X_m(v)_j
=
n\phi_{j+1}(v)^2
=
\bigl(\sqrt n\,\phi_{j+1}(v)\bigr)^2 .
\end{equation}
The actual-coordinate specialization is stated separately
below. Rather than requiring a full distributional transfer
from the Gaussian-wave surrogate to actual eigenvectors, the
collision criterion only needs a spectral collision bound
conditioned on distance collisions.

\begin{definition}
\label{def:gw-surrogate-main}
Fix a graph \(G=(V,E)\), an anchor set \(\mathcal A\), a spectral
dimension \(m\), and a constant \(\rho_\star\in[0,1)\). A random field
\begin{equation}
\{\widetilde Y_j(v):v\in V,\ 1\le j\le m\}
\end{equation}
is called an admissible bounded-correlation Gaussian-wave energy surrogate
if, conditionally on \(G\) and \(\mathcal A\), the following conditions hold:
\begin{enumerate}
\item for each \(j\), \(\{\widetilde Y_j(v):v\in V\}\) is a centered
Gaussian field;
\item for every \(v\in V\) and every \(j\),
\begin{equation}
\mathbb E[\widetilde Y_j(v)^2]=1;
\end{equation}
\item for every \(u\neq v\) and every \(j\),
\begin{equation}
\left|
\operatorname{Corr}
\bigl(
\widetilde Y_j(u),
\widetilde Y_j(v)
\bigr)
\right|
\le
\rho_\star;
\end{equation}
\item the fields are independent across \(j=1,\ldots,m\).
\end{enumerate}
The associated Gaussian-wave energy coordinate is
\begin{equation}
\widetilde X_m(v)
:=
\bigl(
\widetilde Y_1(v)^2,\ldots,\widetilde Y_m(v)^2
\bigr),
\end{equation}
and its quantized spectral code is
\begin{equation}
\widetilde Z_{m,\eta}(v)
:=
Q_\eta(\widetilde X_m(v)).
\end{equation}
\end{definition}

The closed random-regular theorem uses the following two estimates.

\begin{lemma}[Random-regular distance collision decay]
\label{lem:rrg-distance-collision-decay-main}
Fix \(r\ge3\), and let \(G_n\sim\mathcal G_{n,r}\). Let
\(\mathcal A_n\subseteq V(G_n)\) be sampled uniformly without replacement,
independently of \(G_n\), with \(|\mathcal A_n|=k_n\). Assume
\begin{equation}
\frac{k_n}{\log n}\to\infty,
\quad
k_n=o((\log n)^2).
\end{equation}
Then there exists a constant \(c_r>0\), depending only on \(r\), such that,
with
\begin{equation}
I_{\mathrm d,n}(r)
:=
\frac{c_r}{4\log n},
\end{equation}
we have
\begin{equation}
\kappa_{\mathrm D}(G_n,\mathcal A_n)
\le
\exp\{-k_n I_{\mathrm d,n}(r)\}
\end{equation}
with probability tending to one.
\end{lemma}

\begin{lemma}[Gaussian-wave spectral collision decay]
\label{lem:gw-spectral-collision-decay-main}
Fix \(\rho_\star\in[0,1)\). There exists a constant
\(C_{\mathrm s}=C_{\mathrm s}(\rho_\star)>0\) such that, for any finite
connected graph \(G\), any anchor set \(\mathcal A\), and any admissible
bounded-correlation Gaussian-wave energy surrogate,
\begin{equation}
\mathbb E
\left[
\widetilde\kappa_{\mathrm{S}\mid\mathrm{D}}
\,\middle|\,
G,\mathcal A
\right]
\le
\bigl(C_{\mathrm s}\eta\log(e/\eta)\bigr)^m,
\quad
0<\eta<e^{-1}.
\end{equation}
Equivalently, with
\begin{equation}
I_{\mathrm s}^{\mathrm{gw}}(\eta)
:=
\log
\frac{1}{C_{\mathrm s}\eta\log(e/\eta)},
\end{equation}
for any \(\omega_{\mathrm s,n}\to\infty\),
\begin{equation}
\widetilde\kappa_{\mathrm{S}\mid\mathrm{D}}
\le
\exp
\left\{
-m_n I_{\mathrm s}^{\mathrm{gw}}(\eta_n)
+
\omega_{\mathrm s,n}
\right\}
\end{equation}
with conditional probability at least \(1-e^{-\omega_{\mathrm s,n}}\).
\end{lemma}

Let
\(\widetilde F_{G_n,\mathcal A_n}^{(m_n,\eta_n)}\)
denote the hybrid map whose spectral component is an
admissible bounded-correlation Gaussian-wave energy surrogate
with correlation bound \(\rho_\star\in[0,1)\). Define
\begin{equation}
I_{\mathrm d,n}(r)
:=
\frac{c_r}{4\log n},
\quad
I_{\mathrm s}^{\mathrm{gw}}(\eta)
:=
\log
\frac{1}{
C_{\mathrm s}\eta\log(e/\eta)
}.
\label{eq:gw-information-scales}
\end{equation}

\begin{theorem}[Closed Gaussian-wave hybrid achievability]
\label{thm:closed-gw-hybrid-achievability}
Fix \(r\geq 3\), and let \(G_n\sim\mathcal G_{n,r}\). Let
\(\mathcal A_n\subseteq V(G_n)\) be sampled uniformly without
replacement, independently of \(G_n\), with
\(|\mathcal A_n|=k_n\). Assume
\begin{equation}
\frac{k_n}{\log n}\to\infty,
\quad
k_n=o\bigl((\log n)^2\bigr).
\label{eq:closed-gw-anchor-regime}
\end{equation}
There exist constants \(c_r>0\) and
\(C_{\mathrm s}=C_{\mathrm s}(\rho_\star)>0\) such that, if
\begin{equation}
\eta_n\in(0,e^{-1}),
\quad
C_{\mathrm s}\eta_n\log(e/\eta_n)<1,
\label{eq:closed-gw-positive-info}
\end{equation}
and if, for some \(\varepsilon>0\),
\begin{equation}
\begin{aligned}
k_n I_{\mathrm d,n}(r)
+
m_n I_{\mathrm s}^{\mathrm{gw}}(\eta_n)
\ge
(1+\varepsilon)\log n
\end{aligned}
\label{eq:closed-gw-condition}
\end{equation}
for all sufficiently large \(n\), then
\begin{equation}
\operatorname{Err}^*
\bigl(
\widetilde F_{G_n,\mathcal A_n}^{(m_n,\eta_n)}
\bigr)
\xrightarrow{\mathbb P}0.
\label{eq:closed-gw-achievability}
\end{equation}
\end{theorem}

Theorem~\ref{thm:closed-gw-hybrid-achievability} is closed for
the Gaussian-wave energy model. For actual Laplacian-energy
coordinates, the remaining ingredient is a spectral collision
bound inside distance-collision buckets.

For the actual Laplacian-energy code, write
\begin{equation}
X_n(v)
:=
n\bigl(
\phi_2(v)^2,\ldots,\phi_{m_n+1}(v)^2
\bigr).
\label{eq:actual-lap-energy}
\end{equation}
For \(x,y\in\mathbb{R}_{\geq 0}^{m_n}\), define
\begin{equation}
H_{\eta_n}(x,y)
:=
\mathbf{1}
\bigl\{
Q_{\eta_n}(x)=Q_{\eta_n}(y)
\bigr\}.
\label{eq:hard-spectral-collision}
\end{equation}
When distance-collision pairs exist, let
\(\pi_{\mathrm D,n}\) be the uniform measure over ordered
pairs \((u,v)\) with \(u\neq v\) and
\(D_{\mathcal A_n}(u)=D_{\mathcal A_n}(v)\). If no such pair
exists, all \(\pi_{\mathrm D,n}\)-expectations are set to zero.

Let \(H^+_{\eta_n,\alpha_n}\) be the expanded-cell majorant
from Appendix~\ref{app:expanded-cell-majorant}, satisfying
\begin{equation}
0
\leq
H^+_{\eta_n,\alpha_n}(x,y)
\leq
1,
\quad
H_{\eta_n}(x,y)
\leq
H^+_{\eta_n,\alpha_n}(x,y).
\label{eq:spectral-majorant}
\end{equation}
We use
\begin{equation}
\alpha_n
:=
\eta_n(\log n)^{-2}.
\label{eq:smoothing-scale}
\end{equation}
Define the smoothed conditional spectral collision rate
\begin{equation}
\Gamma_n^+
:=
\mathbb{E}_{\pi_{\mathrm D,n}}
\Bigl[
H^+_{\eta_n,\alpha_n}
\bigl(
X_n(U),X_n(V)
\bigr)
\Bigr].
\label{eq:smoothed-spectral-rate}
\end{equation}

\begin{assumption}
\label{ass:actual-spectral-collision-bound}
Fix \(r\geq 3\), and let \(G_n\sim\mathcal{G}_{n,r}\).
Let \(\mathcal A_n\subseteq V(G_n)\) be sampled uniformly
without replacement and independently of \(G_n\). Assume
\begin{equation}
\begin{aligned}
m_n
&=
O(\log n),
&
\eta_n
&\in
[n^{-c_0},e^{-1}),
&
\alpha_n
&=
\eta_n(\log n)^{-2},
\end{aligned}
\label{eq:actual-spectral-bound-regime}
\end{equation}
where \(c_0>0\) is fixed. Suppose that there exist a constant
\(C_{\mathrm{sc},r}>0\), depending only on \(r\), and a
deterministic sequence \(\xi_n=o(\log n)\), such that, with
probability tending to one,
\begin{equation}
\Gamma_n^+
\leq
\exp
\Bigl\{
-m_n I_{\mathrm{sc},r}(\eta_n)
+
\xi_n
\Bigr\},
\label{eq:actual-spectral-collision-bound}
\end{equation}
where
\begin{equation}
I_{\mathrm{sc},r}(\eta)
:=
\log
\frac{1}{
C_{\mathrm{sc},r}\eta\log(e/\eta)
}.
\label{eq:spectral-collision-information-scale}
\end{equation}
\end{assumption}

\begin{corollary}[Conditional actual Laplacian hybrid achievability]
\label{cor:lap-hybrid-achievability}
Fix \(r\geq 3\), and let \(G_n\sim\mathcal{G}_{n,r}\).
Let \(\mathcal A_n\subseteq V(G_n)\) be sampled uniformly
without replacement and independently of \(G_n\), with
\(|\mathcal A_n|=k_n\). Assume
\begin{equation}
\frac{k_n}{\log n}\to\infty,
\quad
k_n=o\bigl((\log n)^2\bigr).
\label{eq:lap-anchor-regime-main}
\end{equation}
Suppose Assumption~\ref{ass:actual-spectral-collision-bound}
holds and
\begin{equation}
C_{\mathrm{sc},r}\eta_n\log(e/\eta_n)<1
\label{eq:lap-positive-spectral-information}
\end{equation}
for all sufficiently large \(n\). Let
\begin{equation}
F_n^{\mathrm{Lap}}(v)
:=
\bigl(
D_{\mathcal A_n}(v),
Q_{\eta_n}(X_n(v))
\bigr).
\label{eq:lap-hybrid-map-main}
\end{equation}
If there exists \(\varepsilon>0\) such that
\begin{equation}
\begin{aligned}
k_n I_{\mathrm d,n}(r)
+
m_n I_{\mathrm{sc},r}(\eta_n)
\geq
(1+\varepsilon)\log n
\end{aligned}
\label{eq:lap-achievability-budget-main}
\end{equation}
for all sufficiently large \(n\), then
\begin{equation}
\operatorname{Err}^*
\bigl(
F_n^{\mathrm{Lap}}
\bigr)
\xrightarrow{\mathbb P}0.
\label{eq:lap-achievability-main}
\end{equation}
\end{corollary}

\subsection{Two-Sided Localization Criterion}
\label{subsec:two-sided-localization-criterion}

The converse and collision bounds can be summarized by two graph-dependent
quantities. Writing
\(R_{L,\eta}:=\lfloor L/\eta\rfloor\), define the simplex-refined
converse code budget
\begin{equation}
\begin{aligned}
\mathcal B_{\mathrm{conv}}^{\Delta}
(G,\mathcal A,m,\eta,L)
:={}&
\log D(G,\mathcal A)\\
&+
\log\binom{R_{L,\eta}+m}{m},
\end{aligned}
\label{eq:converse-budget}
\end{equation}
and the collision information
\begin{equation}
\mathcal I_{\mathrm H}(G,\mathcal A,m,\eta)
=
-\log\kappa_{\mathrm D}
-
\log\kappa_{\mathrm{S}\mid\mathrm D}.
\label{eq:collision-information}
\end{equation}
Here, the superscript \(\Delta\) denotes the simplex-counting
converse budget. The conservative box budget used in the
finite-sample design diagnostics is denoted separately by
\(\mathcal B_{\mathrm{conv}}^{\Box}\) in Section~V.

Then Theorem~\ref{thm:normalized-refined-converse} and
Lemma~\ref{lem:basic-collision-lemma} imply
\begin{equation}
\begin{aligned}
&1-
\exp\!\left\{
\mathcal B_{\mathrm{conv}}^{\Delta}
(G,\mathcal A,m,\eta,L)-\log n
\right\}
-
\frac{m}{L}\\
\le{}&
\operatorname{Err}^*
\bigl(F_{G,\mathcal A}^{(m,\eta)}\bigr)
\le
(n-1)\exp\!\left\{
-\mathcal I_{\mathrm H}(G,\mathcal A,m,\eta)
\right\}.
\label{eq:two-sided-bound-compact}
\end{aligned}
\end{equation}

\begin{corollary}[Graph-dependent design principle]
\label{cor:graph-dependent-design-principle}
For a sequence of hybrid encodings
\(F_{G_n,\mathcal A_n}^{(m_n,\eta_n)}\), the following two
regimes hold.

If there exist \(L_n>0\) and \(\varepsilon>0\) such that
\begin{equation}
\mathcal B_{\mathrm{conv}}^{\Delta}
(G_n,\mathcal A_n,m_n,\eta_n,L_n)
\le
(1-\varepsilon)\log n
\end{equation}
and \(m_n/L_n\to0\), then
\begin{equation}
\operatorname{Err}^*
\bigl(F_{G_n,\mathcal A_n}^{(m_n,\eta_n)}\bigr)
\to1.
\end{equation}
If there exists \(\varepsilon>0\) such that
\begin{equation}
\mathcal I_{\mathrm H}(G_n,\mathcal A_n,m_n,\eta_n)
\ge
(1+\varepsilon)\log n,
\end{equation}
then
\begin{equation}
\operatorname{Err}^*
\bigl(F_{G_n,\mathcal A_n}^{(m_n,\eta_n)}\bigr)
\to0.
\end{equation}
\end{corollary}

Thus, the hybrid positional encoding admits a two-sided
interpretation. Localization is impossible when the
simplex-refined code budget after distance partitioning and
spectral quantization is below the \(\log n\) scale, while
localization is achievable when the total distance-spectral
collision information exceeds the same scale.

\section{Empirical Evaluation}
\label{SE7}
\subsection{Experimental Setup}
\label{subsec:experimental-setup}

All experiments instantiate the hybrid observation map in
Section~\ref{SE3}. To avoid overloading the scalar trimming
mass \(S_m(v)=\sum_{j=1}^m X_m(v)_j\) used in the converse
analysis, we denote by \(\Psi_m(v)\) the spectral coordinate
vector used in a given empirical run. For compactness, the
empirical map is written as
\begin{equation}
\label{eq:exp_observation_map}
F(v)=\bigl(D_{\mathcal A}(v),Q_{\eta}(\Psi_m(v))\bigr),
\end{equation}
where \(D_{\mathcal A}(v)\) denotes the anchor-distance
profile. Depending on the experiment, \(\Psi_m(v)\) is
instantiated as actual Laplacian-energy coordinates, signed
Laplacian coordinates, or Gaussian-wave surrogate coordinates.
Unless otherwise specified, anchors are sampled uniformly
without replacement, and quantization is performed
coordinatewise by floor binning,
\begin{equation}
Q_{\eta}(x)
=
\bigl(\lfloor x_1/\eta\rfloor,\ldots,\lfloor x_p/\eta\rfloor\bigr),
\qquad x\in\mathbb R^p .
\end{equation}

\paragraph{Synthetic graph diagnostics}
The synthetic experiments cover random regular graphs,
Erd\H{o}s--R\'enyi graphs, stochastic block models, grids,
and barbell graphs. The main localization sweep uses graph
sizes around \(n\in\{500,1000,2000\}\), with grids instantiated
at \(n\in\{529,1024,2025\}\). The parameter grid is
\[
\begin{aligned}
k &\in \{1,2,3,4,5,6,8,12,16,24,32\},\\
m &\in \{1,2,3,4,6,8,12,16\},\\
\eta &\in \{0.8,0.6,0.5,0.4,0.3,0.2,0.1\}.
\end{aligned}
\]
Distance-only, spectral-only, and hybrid encodings are compared under
both Laplacian and Gaussian-wave spectral coordinates. A separate
collision-transfer sweep uses
\[
k\in\{2,4,8,16,32\},
\]
\[
m\in\{2,4,8,16\},
\]
\[
\eta\in\{0.7,0.5,0.3,0.1\},
\]
and records the collision identity
\begin{equation}
\label{eq:exp_collision_identity}
\kappa_H=\kappa_D\kappa_{S|D},
\qquad
I_H=-\log\kappa_H .
\end{equation}
For rows with \(\kappa_D>0\), this is equivalently
\(I_H=-\log\kappa_D-\log\kappa_{S|D}\). When
\(\kappa_D=0\), there are no distance-colliding ordered
pairs, so the ordinary conditional estimator of
\(\kappa_{S|D}\) has no denominator. We mark these cases as
distance-saturated and use the convention
\[
\kappa_{S|D}=0,\qquad
\kappa_H=0,\qquad
I_H=+\infty .
\]
Thus, such rows represent successful distance-level saturation
rather than numerical failure. Finite discrepancy summaries
either exclude these infinite-information rows or report them
separately as saturated configurations.

For the main Laplacian-energy hybrid setting, we additionally
perform graph-clustered and held-out validation.  All
\((k,m,\eta)\) configurations generated from the same graph instance
are treated as one cluster.  Thresholds are calibrated only on
disjoint graph groups and evaluated under leave-one-family-out,
leave-one-setting-out, leave-one-size-out, and size-extrapolation
protocols.

\paragraph{Universal Dependencies structural task probes}
We conduct real-graph structural task probes on Universal Dependencies
trees.  Tokens are nodes, head-dependent arcs are converted into
undirected edges, and lexical forms, dependency labels, edge directions,
and token attributes are excluded.  We use English-EWT, Chinese-GSD,
Spanish-GSD, French-GSD, and German-GSD, filter sentences to
\(6\le n\le 80\), and use at most \(3000\) training sentences per
treebank together with all filtered development and test sentences.
These probes measure structural information in the positional code and
should not be interpreted as full dependency parsing.

The UD comparison includes four encoding families:
\[
\begin{aligned}
\text{NoPE}:&\quad F(v)=\mathrm{const},\\
\text{Distance}:&\quad F(v)=D_{\mathcal A}(v),\\
\text{SpectralEnergy}:&\quad F(v)=Q_\eta(\Psi_m(v)),\\
\text{HybridEnergy}:&\quad F(v)=\bigl(D_{\mathcal A}(v),Q_\eta(\Psi_m(v))\bigr).
\end{aligned}
\]
The UD grid is
\[
k\in\{1,2,4,8\},\quad
m\in\{2,4,8,16\},\quad
\eta=0.25,
\]
with three random anchor trials for each configuration.

\paragraph{Metrics}
The primary encoding-level metrics are the image-size success rate and
its induced localization error,
\begin{equation}
\label{eq:exp_success_error}
\mathrm{succ}(F)=\frac{|\mathrm{Im}(F)|}{n},
\quad
\mathrm{err}(F)=1-\mathrm{succ}(F).
\end{equation}
The synthetic diagnostics also report the normalized collision
information \(I_H/\log n\), the conservative box-budget diagnostic
\(B_{\rm conv}^{\Box}/\log n\), and Gaussian-wave surrogate
discrepancies, where
\[
B_{\rm conv}^{\Box}
=\log D(G,A)+m\log(\lceil L/\eta\rceil+1)
\]
is the budget used in the finite-sample sweeps. The simplex-refined
budget \(B_{\rm conv}^{\Delta}\) used in the theory is no larger than
\(B_{\rm conv}^{\Box}\). For UD trees, we report PE-only
structural-probe metrics in addition to encoding-level localization
quantities. The surface-position and dependency-depth probes use NMAE,
while the pairwise dependency-distance probe uses macro-F1. PE-row
derangement controls preserve sentence-level PE marginals but remove
node-code alignment. Details are given in
Appendix~\ref{app:ud_localization_diagnostics}.

\subsection{Mathematical Localization and Surrogate-Discrepancy Diagnostics}
\label{subsec:math_localization_transfer}

We first evaluate the localization mechanism underlying the theoretical
results.  For each graph and positional-encoding configuration, we compute
the optimal localization error \(\operatorname{Err}^{*}(F)\), the
conservative box-budget diagnostic \(B_{\rm conv}^{\Box}\), and the hybrid
collision information
\begin{equation}
I_H
=
-\log \kappa_D
-\log \kappa_{S|D}.
\end{equation}
All localization results are averaged over \(10\) independent trials.

\paragraph{Localization phase transition}
Figure~\ref{fig:exp1_summary} summarizes the main localization diagnostics.
The empirical error decreases as \(I_H/\log n\) crosses the unit scale, while
the surrogate discrepancy separates expander-like graphs from bottleneck graphs.  The
design map further shows that low-error configurations are organized most
clearly by the achievability-side information ratio \(I_H/\log n\).

\begin{figure*}[t]
    \centering
    \includegraphics[width=0.95\textwidth]{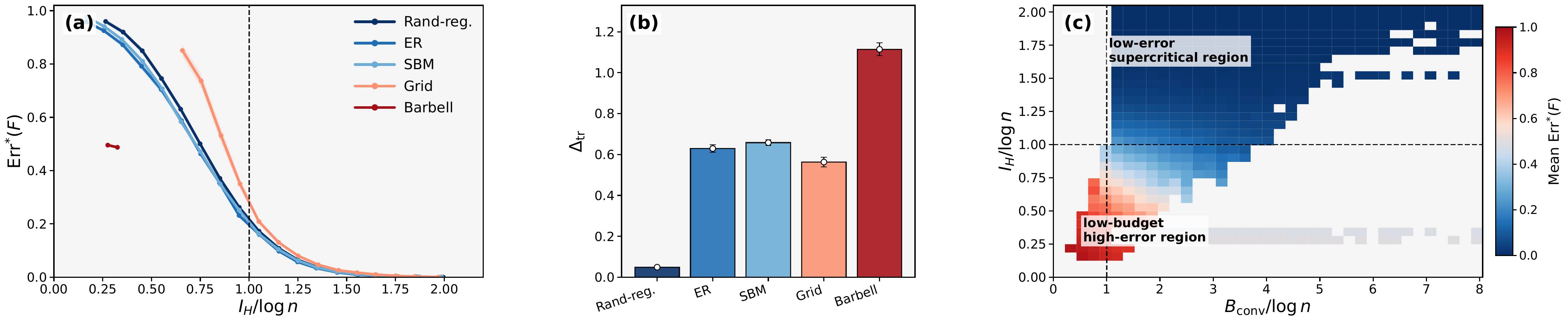}
    \caption{
    Localization diagnostics.
    (a) Error versus \(I_H/\log n\).
    (b) Gaussian-wave surrogate discrepancy \(\Delta_{\rm tr}\).
    (c) Design map in the
    \((B_{\rm conv}^{\Box}/\log n, I_H/\log n)\) plane.
    }
    \label{fig:exp1_summary}
\end{figure*}

Table~\ref{tab:exp1_summary} reports family-level statistics.  Hybrid
Laplacian-energy encodings substantially outperform both distance-only and
spectral-only encodings on random regular, ER, SBM, and grid graphs.  Their
average error reduction relative to distance-only encoding ranges from
\(63.9\%\) to \(80.6\%\).  Barbell graphs provide a contrasting failure case:
their hybrid error remains close to the distance-only error, their information
level remains subcritical, and their surrogate discrepancy is the largest.

\begin{table}[t]
\centering
\scriptsize
\setlength{\tabcolsep}{2.0pt}
\caption{
Family-level localization summary.
D, S, and H denote distance-only, spectral-only, and hybrid encodings.
}
\label{tab:exp1_summary}
\resizebox{\columnwidth}{!}{%
\begin{tabular}{lcccccc}
\toprule
Graph
& D Err.
& S Err.
& H Err.
& H \(I_H/\log n\)
& \(\Delta_{\rm tr}\)
& \(\Delta\)Err. \\
\midrule
Rand-reg.
& 0.409
& 0.364
& \textbf{0.079}
& 1.710
& \textbf{0.049}
& \textbf{0.001}
\\
ER
& 0.471
& 0.513
& \textbf{0.170}
& 1.443
& 0.629
& 0.073
\\
SBM
& 0.424
& 0.475
& \textbf{0.123}
& 1.522
& 0.658
& 0.043
\\
Grid
& 0.225
& 0.721
& \textbf{0.051}
& 1.651
& 0.563
& 0.024
\\
Barbell
& 0.495
& 0.886
& 0.490
& 0.309
& \textbf{1.114}
& \textbf{0.347}
\\
\bottomrule
\end{tabular}%
}
\end{table}

\paragraph{Collision decomposition and Gaussian-wave transfer}
We next examine the collision decomposition and the Gaussian-wave transfer
diagnostic used in the achievability theory. For each configuration, we
compute the distance collision rate \(\kappa_{\mathrm D}\), the weighted
conditional spectral collision rate \(\kappa_{\mathrm{S}\mid\mathrm D}\),
and the hybrid collision rate \(\kappa_{\mathrm H}\). The implementation
follows the deterministic identity
\begin{equation}
\kappa_{\mathrm H}
=
\kappa_{\mathrm D}\kappa_{\mathrm{S}\mid\mathrm D},
\quad
I_{\mathrm H}
=
-\log \kappa_{\mathrm H}
=
-\log \kappa_{\mathrm D}
-\log \kappa_{\mathrm{S}\mid\mathrm D},
\end{equation}
with the extended-real convention \(-\log 0=+\infty\). Across
\(33{,}600\) evaluated configurations, the maximum numerical factorization
error is \(9.918\times 10^{-17}\), and the maximum discrepancy in the
additive information identity is \(3.553\times 10^{-15}\) on finite-information
configurations. No violation of the collision-to-error inequality or the
conservative box-budget converse corollary is observed.

Figure~\ref{fig:exp2_collision_transfer} summarizes the two main diagnostics.
Panel~(a) shows that localization error is organized by the information ratio
\(I_H/\log n\): the mean error drops from \(0.515\) in the low-information
regime to \(0.003\) in the high-information regime.  Panel~(b) reports the
hard Gaussian-wave surrogate discrepancy
\begin{equation}
\Delta_{\rm tr}
=
\frac{
|I_{S|D}^{\rm Lap}-I_{S|D}^{\rm GW}|
}{\log n}.
\end{equation}
The gap is smallest on expander-like random regular graphs, where it is
\(0.050\) for energy coordinates and \(0.035\) for signed coordinates.  It is
substantially larger on bottleneck graphs, especially barbell graphs, where
the corresponding gaps increase to \(1.240\) and \(1.520\).  These hard-collision discrepancies indicate closer agreement between the
Gaussian-wave surrogate and actual Laplacian coordinates on expander-like
graphs, and weaker agreement on bottleneck graphs. They do not, by
themselves, establish the conditional spectral-collision assumption for
actual coordinates. The corresponding diagnostic based on the expanded-cell
majorant is reported in Appendix~\ref{app:exp2_collision_transfer}.

\begin{figure}[!t]
    \centering
    \includegraphics[width=0.98\columnwidth]{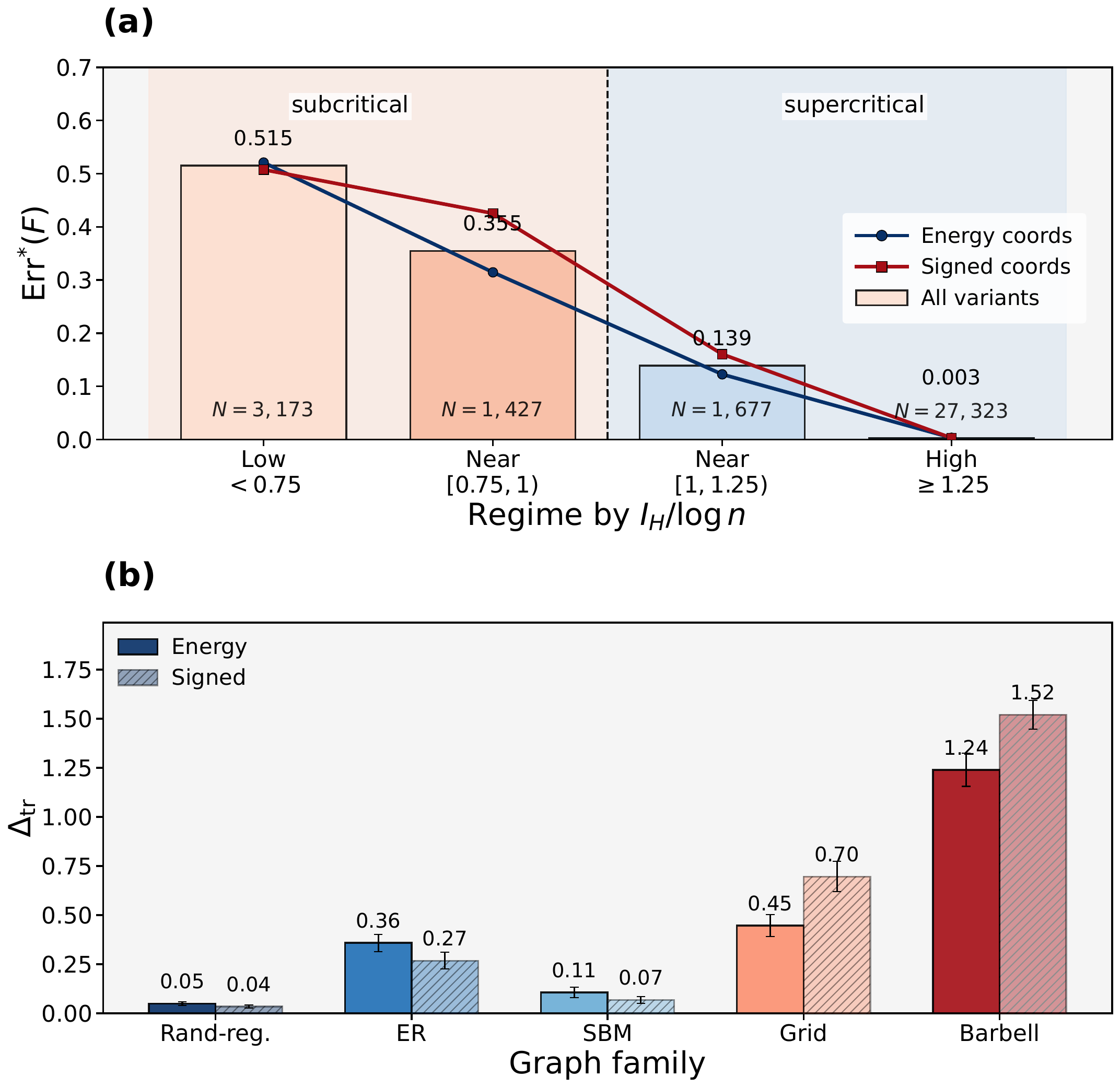}
    \vspace{-0.4em}
    \caption{
    Collision decomposition and hard Gaussian-wave transfer diagnostics.
    (a) Regime-binned localization error by \(I_H/\log n\).
    (b) Hard surrogate discrepancy
    \(\Delta_{\rm tr}
    =
    |I_{S|D}^{\rm Lap}-I_{S|D}^{\rm GW}|/\log n\)
    across graph families.
    }
    \label{fig:exp2_collision_transfer}
    \vspace{-1.0em}
\end{figure}

\paragraph{Unified graph-dependent design diagnostics}
Finally, we consolidate the localization and collision diagnostics into a
single graph-dependent design analysis.  For each configuration, we compare
the conservative converse-side box budget and the achievability-side
normalized collision information,
\begin{equation}
\frac{B_{\rm conv}^{\Box}}{\log n},
\quad
\frac{I_H}{\log n}.
\end{equation}
This directly evaluates the graph-dependent design principle: configurations
with insufficient effective code budget should remain difficult to localize,
whereas configurations with \(I_H>\log n\) should be localizing.

Table~\ref{tab:exp3_design_diagnostics} summarizes the main
Laplacian-energy hybrid setting over \(25{,}256\) configurations, covering
five graph families and seven graph settings, with graph sizes ranging from
\(n=494\) to \(n=2025\).  The achievability-side rule
\(I_H/\log n\ge1\) is highly predictive of localization success.  Using
\(\operatorname{Err}^*(F)\le0.1\) as success, it achieves precision \(0.930\)
and recall \(0.9998\).  With the relaxed threshold
\(\operatorname{Err}^*(F)\le0.2\), the precision and recall become \(0.990\)
and \(0.991\), respectively.

The converse-side rule is more conservative.  The region
\(B_{\rm conv}^{\Box}/\log n<1\) and \(I_H/\log n<1\) has mean localization
error \(0.829\) and median error \(0.850\).  The rule
\(B_{\rm conv}^{\Box}/\log n<1\) detects very high-error configurations
\(\operatorname{Err}^*(F)\ge0.9\) with recall \(0.975\), although its
precision is \(0.340\).  Overall, \(I_H/\log n\) provides a sharp
success-side diagnostic, while \(B_{\rm conv}^{\Box}/\log n\) identifies a
conservative failure-side regime.  The corresponding design-map
visualizations are reported in Appendix~\ref{app:exp3_design_diagnostics}.

\begin{table}[!t]
\centering
\caption{
Unified graph-dependent design diagnostics for the Laplacian-energy hybrid
positional encoding. We write
\(b=B_{\rm conv}^{\Box}/\log n\) and \(h=I_H/\log n\).
Succ. denotes \(\operatorname{Err}^*(F)\le0.1\), and High denotes
\(\operatorname{Err}^*(F)\ge0.9\).
}
\label{tab:exp3_design_diagnostics}

\footnotesize
\setlength{\tabcolsep}{0pt}
\renewcommand{\arraystretch}{1.08}

\begin{tabular*}{\columnwidth}{@{\extracolsep{\fill}}llrrrr@{}}
\toprule
\multicolumn{6}{@{}l}{\textbf{(a) Design-plane regions}} \\
\midrule
\(b\) & \(h\) & Count & Mean & Med. & Succ./High \\
\midrule
\(<1\)   & \(<1\)   & \(1{,}150\)  & \(0.829\) & \(0.850\) & \(0.000/0.340\) \\
\(<1\)   & \(\ge1\) & \(0\)        & \multicolumn{3}{c@{}}{No configurations observed} \\
\(\ge1\) & \(<1\)   & \(5{,}033\)  & \(0.463\) & \(0.488\) & \(0.001/0.002\) \\
\(\ge1\) & \(\ge1\) & \(19{,}073\) & \(0.019\) & \(0.001\) & \(0.930/0.000\) \\
\bottomrule
\end{tabular*}

\vspace{0.45em}

\begin{tabular*}{\columnwidth}{@{\extracolsep{\fill}}lrrrrr@{}}
\toprule
\multicolumn{6}{@{}l}{\textbf{(b) Prediction quality of design rules}} \\
\midrule
Rule / target & TP & FP & FN & Prec. & Rec. \\
\midrule
\(h\ge1,\ \operatorname{Err}^*\le0.1\)
& \(17{,}745\) & \(1{,}328\) & \(3\)   & \(0.930\) & \(0.9998\) \\
\(h\ge1,\ \operatorname{Err}^*\le0.2\)
& \(18{,}883\) & \(190\)     & \(166\) & \(0.990\) & \(0.991\) \\
\(b<1,\ \operatorname{Err}^*\ge0.9\)
& \(391\)      & \(759\)     & \(10\)  & \(0.340\) & \(0.975\) \\
\bottomrule
\end{tabular*}

\vspace{-0.5em}
\end{table}

\paragraph{Held-out and graph-clustered validation}
Because the design sweep reuses graph instances across parameter settings,
we further evaluate the Laplacian-energy hybrid setting with graph-clustered
and held-out validation.  The expanded sweep contains \(129{,}360\)
configurations from \(210\) graph-instance clusters.  Configurations from
the same graph are grouped as one cluster, and held-out thresholds are
calibrated only on disjoint training graph groups.

\begin{table}[!t]
\centering
\caption{
Held-out and graph-clustered validation for the Laplacian-energy hybrid
encoding.  The target is
\(\operatorname{Err}^{*}(F)\leq0.1\), and \(h=I_H/\log n\).
All reported scores are graph-cluster macro metrics.  ``Avg.'' denotes
the unweighted mean over held-out folds.  The barbell negative-control
fold is excluded from family and setting averages because it contains no
positive success cases.
}
\label{tab:exp3_heldout_clustered}
\scriptsize
\setlength{\tabcolsep}{2.4pt}
\resizebox{\columnwidth}{!}{%
\begin{tabular}{lccccc}
\toprule
Protocol & Threshold & AUC & Prec. & Rec. & F1 \\
\midrule
Clustered, theory rule
& \(1.0000\) & \(0.9995\) & \(0.9393\) & \(0.9999\) & \(0.9685\) \\
Clustered, calibration reference
& \(1.1545\) & \(0.9995\) & \(0.9928\) & \(0.9925\) & \(0.9926\) \\
Leave-one-family-out, Avg.
& \(1.1546\) & \(0.9993\) & \(0.9914\) & \(0.9903\) & \(0.9908\) \\
Leave-one-setting-out, Avg.
& \(1.1562\) & \(0.9995\) & \(0.9929\) & \(0.9919\) & \(0.9924\) \\
Leave-one-size-out, Avg.
& \(1.1619\) & \(0.9995\) & \(0.9933\) & \(0.9892\) & \(0.9912\) \\
Size extrapolation, Avg.
& \(1.1876\) & \(0.9994\) & \(0.9996\) & \(0.9754\) & \(0.9873\) \\
\bottomrule
\end{tabular}%
}
\vspace{-0.8em}
\end{table}

Table~\ref{tab:exp3_heldout_clustered} shows that \(I_H/\log n\)
remains a stable success-side diagnostic after clustering by graph instance
and transferring thresholds across unseen graph families, structural
settings, and graph sizes.

\subsection{Real-graph structural task probes on Universal Dependencies}
\label{subsec:ud_localization}

This experiment evaluates whether the localization advantage observed on
synthetic graphs also translates into task-relevant structural recovery on
real dependency-tree graphs.  The goal is not full dependency parsing or
language-model fine-tuning, because the positional codes are computed from
gold dependency trees.  Instead, UD is used as a controlled PE-only
structural task-probe benchmark: tokens are nodes, undirected dependency
arcs are edges, and lexical forms, dependency labels, edge directions, and
token attributes are excluded.  Thus, any predictive signal must come from
the positional code computed on the unlabeled undirected dependency-tree
skeleton.

This design keeps the experiment aligned with the theory.  The theoretical
claims concern the intrinsic resolution of a positional observation map,
whereas end-to-end downstream accuracy would mix positional information
with lexical features, model architecture, optimization, and task-specific
correlations.  The UD probes therefore serve as a middle ground between
pure localization diagnostics and full downstream benchmarks: they test
whether a code that better localizes nodes also supports recovery of
syntactic geometry on real graphs.

We use five UD treebanks, with at most \(3000\) training sentences per
treebank and the full official development and test splits after length
filtering.  The comparison includes NoPE, distance-only anchor codes,
Laplacian-energy codes, and the hybrid distance-energy code.

\paragraph{PE-only structural task protocol}
We evaluate three supervised probes using only positional codes as input.
The surface-position probe predicts normalized token order and tests
whether the undirected dependency-tree skeleton carries a weak word-order
signal.  The dependency-depth probe predicts
\begin{equation}
y_{\mathrm{depth}}(v)
=
\frac{d_G(v,r)}{\max_{u\in V} d_G(u,r)},
\end{equation}
where \(r\) is the dependency root.  This probe measures whether the code
captures node-level syntactic geometry.  The pairwise dependency-distance
probe predicts a clipped bucket of \(d_G(u,v)\) from the pair representation
\begin{equation}
\bigl[
F(u)+F(v),\,
|F(u)-F(v)|,\,
F(u)\odot F(v)
\bigr],
\end{equation}
and tests whether pairwise tree geometry can be recovered from the code.

For the added dependency-depth and pairwise-distance probes, no new
hyperparameter search is performed.  The configurations selected by the
original surface-position development protocol are frozen and reused.  This
makes the additional probes confirmatory: they test whether configurations
chosen from a weak surface-order signal also support more direct
tree-structural recovery.

\paragraph{Main UD structural-probe results}
Table~\ref{tab:ud_structural_probes} summarizes the results.  The
surface-position probe shows only modest gains: HybridEnergy reduces the
average NMAE from \(0.262\) to \(0.257\).  This is expected, since linear
word order is only partially determined by the undirected dependency-tree
skeleton.  In contrast, the depth and pairwise-distance probes show a much
clearer separation.  HybridEnergy achieves the best depth NMAE and the
best pairwise-distance macro-F1, outperforming both distance-only and
spectral-energy encodings.  This indicates that anchor distances and
Laplacian-energy coordinates capture complementary aspects of
dependency-tree geometry.

The PE-row derangement control separates node-code alignment from marginal
PE statistics.  Within each sentence, PE rows are deranged while targets
are kept fixed.  This preserves sentence-level PE statistics, feature
dimension, and code budget, but removes the alignment between a token and
its own graph code.  The alignment margin
\(\Delta_{\rm align}\) in Table~\ref{tab:ud_structural_probes} is the
real-condition gain over NoPE minus the corresponding deranged-condition
gain.  HybridEnergy has the larger alignment-specific margin across all
three probes, showing that its gains are not explained solely by the
marginal distribution of PE values.

\begin{table}[t]
\centering
\scriptsize
\setlength{\tabcolsep}{2.4pt}
\caption{
UD structural task probes.  Position and depth use NMAE; pairwise distance
uses macro-F1.  D, S, and H denote distance-only, spectral-energy, and
hybrid-energy encodings.  \(\Delta_{\rm align}\) is the real-minus-null
alignment margin from PE-row derangement.
}
\label{tab:ud_structural_probes}
\resizebox{\columnwidth}{!}{%
\begin{tabular}{llcccccc}
\toprule
Probe
& Metric
& NoPE
& D
& S
& H
& \(\Delta_{\rm align}^{S}\)
& \(\Delta_{\rm align}^{H}\)
\\
\midrule
Position
& NMAE \(\downarrow\)
& 0.262
& 0.260
& 0.258
& \textbf{0.257}
& 0.00011
& \textbf{0.00097}
\\
Depth
& NMAE \(\downarrow\)
& 0.224
& 0.190
& 0.196
& \textbf{0.173}
& 0.0347
& \textbf{0.0509}
\\
Pairwise dist.
& Macro-F1 \(\uparrow\)
& 0.074
& 0.642
& 0.532
& \textbf{0.812}
& 0.3919
& \textbf{0.5881}
\\
\bottomrule
\end{tabular}%
}
\end{table}

\paragraph{Design ablations on UD localization}
Table~\ref{tab:exp4_ud_ablation} reports UD design diagnostics.  These
ablations are not additional downstream benchmarks; rather, they test how
the structural-probe behavior changes under the design variables suggested
by the theory.

Anchor placement has the strongest effect.  Root and farthest anchors both
improve over random anchors, with farthest anchors giving the best
localization and the best surface-position recovery.  This is consistent
with the role of anchor profiles in reducing distance-bucket ambiguity.
The signed-coordinate block separates within-tree collision control from
cross-sentence probe stability.  Signed Laplacian coordinates greatly
reduce collisions and increase \(I_H/\log n\), but they do not improve
prediction accuracy.  This suggests that low collision inside a single
tree is not sufficient for cross-sentence probing: the coordinate system
must also be stable enough across different trees.

The quantization block follows the expected diagnostic trend.  Coarser
\(\eta\) reduces \(I_H/\log n\) and slightly increases collision error,
while probe performance remains stable for \(\eta\leq 0.5\) and weakens
mildly at \(\eta=1.0\).  Additional protocol details,
configuration-level diagnostics, and the compact derangement table are
provided in Appendix~\ref{app:ud_localization_diagnostics}.

\begin{table}[t]
\centering
\scriptsize
\setlength{\tabcolsep}{2.8pt}
\caption{
UD localization ablations.  Anchor and sign blocks use three treebanks.
Quantization reports probe metrics on two treebanks and diagnostic metrics on
five treebanks.  NMAE and \(\mathrm{Err}^*(F)\) are lower better; other
metrics are higher better.
}
\label{tab:exp4_ud_ablation}
\resizebox{\columnwidth}{!}{%
\begin{tabular}{llcccccc}
\toprule
Ablation
& Setting
& Lang.
& NMAE
& \(\tau\)
& \(I_H/\log n\)
& \(\mathrm{Err}^*(F)\)
& \(B_{\mathrm{conv}}/\log n\)
\\
\midrule
Anchor
& random-H
& 3
& 0.261
& 0.088
& 3.687
& 0.030
& 13.90
\\
& root-H
& 3
& 0.241
& 0.214
& 3.612
& 0.032
& 13.89
\\
& farthest-H
& 3
& \textbf{0.188}
& \textbf{0.528}
& 3.846
& \textbf{0.024}
& 13.90
\\
\midrule
Sign
& S-energy
& 3
& \textbf{0.260}
& \textbf{0.100}
& 1.994
& 0.133
& 12.95
\\
& S-signed
& 3
& 0.261
& 0.077
& 4.889
& 0.005
& 12.95
\\
& H-energy
& 3
& 0.261
& 0.088
& 3.687
& 0.030
& 13.90
\\
& H-signed
& 3
& 0.262
& 0.082
& \textbf{4.901}
& \textbf{0.004}
& 13.90
\\
\midrule
Quant.
& H, \(\eta=0.125\)
& 2/5
& \textbf{0.260}
& 0.096
& \textbf{3.338}
& \textbf{0.042}
& 17.17
\\
& H, \(\eta=0.25\)
& 2/5
& 0.260
& \textbf{0.097}
& 3.335
& 0.042
& 14.09
\\
& H, \(\eta=0.5\)
& 2/5
& 0.260
& 0.095
& 3.329
& 0.043
& 11.14
\\
& H, \(\eta=1.0\)
& 2/5
& 0.261
& 0.087
& 3.313
& 0.044
& 8.41
\\
\bottomrule
\end{tabular}%
}
\end{table}

\section{Conclusions and Limitations}
\label{SE9}

\subsection{Discussions}

This paper studied node localization from hybrid graph positional
encodings that combine anchor-distance profiles with quantized
low-frequency Laplacian-energy coordinates.  By viewing the encoding as
an observation map, we separated intrinsic positional identifiability
from downstream architectures, attributes, and task-specific features.
The normalized converse identifies subcritical code budgets, while the
collision analysis gives \(\kappa_H=\kappa_D\kappa_{S|D}\) and the
information measure \(I_H\).  Experiments show that \(I_H/\log n\)
calibrates localization success, and UD structural task probes further
show that this localization information supports recovery of
syntactic-tree geometry, especially dependency depth and pairwise tree
distance.

\subsection{Limitations}

The main theoretical limitation is that the actual-Laplacian
specialization is conditional on
Assumption~\ref{ass:actual-spectral-collision-bound}, which requires a
distance-conditioned spectral collision bound for low-frequency
Laplacian-energy coordinates.  The formulation also focuses on finite,
connected, undirected, and unweighted graphs, whereas practical graph
learning and NLP systems often include directions, labels, attributes,
and task-specific features.  Accordingly, the UD experiments should be
interpreted as PE-only structural task probes on unlabeled undirected
dependency-tree skeletons, not as full dependency parsing or
language-model fine-tuning.  Repeated or nearly repeated eigenspaces may
also introduce basis-stability issues, even though squared energy
coordinates remove sign ambiguity.

\subsection{Future Extensions}

Future work should establish
Assumption~\ref{ass:actual-spectral-collision-bound} from quantitative
eigenvector universality or related local laws, and extend the
actual-coordinate analysis beyond random regular graphs to other random
and structured families, including bottleneck regimes. Another direction
is to develop basis-invariant projector-energy or canonical spectral codes
for repeated and nearly repeated eigenspaces, together with non-asymptotic,
\(I_H\)-guided choices of anchors, spectral dimension, and quantization.
Extending the framework to directed, labeled, weighted, and attributed
graphs, and testing whether encoding-level collision information predicts
gains in full downstream architectures, would connect the theory to
dependency parsing, relation classification, and graph reasoning.

\bibliographystyle{IEEEtran} 
\bibliography{references} 

@inproceedings{wang2022peg,
  title     = {Equivariant and Stable Positional Encoding for More Powerful
               Graph Neural Networks},
  author    = {Wang, Haorui and Yin, Haoteng and Zhang, Muhan and Li, Pan},
  booktitle = {International Conference on Learning Representations},
  year      = {2022}
}

@inproceedings{ma2023laplacian,
  title     = {Laplacian Canonization: A Minimalist Approach to Sign and Basis Invariant Spectral Embedding},
  author    = {Ma, George and Wang, Yifei and Wang, Yisen},
  booktitle = {Advances in Neural Information Processing Systems},
  volume    = {36},
  pages     = {11296--11337},
  year      = {2023}
}

@inproceedings{eliasof2023rfp,
  title     = {Graph Positional Encoding via Random Feature Propagation},
  author    = {Eliasof, Moshe and Frasca, Fabrizio and Bevilacqua, Beatrice
               and Treister, Eran and Chechik, Gal and Maron, Haggai},
  booktitle = {Proceedings of the 40th International Conference on Machine Learning},
  series    = {Proceedings of Machine Learning Research},
  volume    = {202},
  pages     = {9202--9223},
  year      = {2023},
  publisher = {PMLR}
}

@inproceedings{yan2024cyclenet,
  title     = {Cycle Invariant Positional Encoding for Graph Representation Learning},
  author    = {Yan, Zuoyu and Ma, Tengfei and Gao, Liangcai and Tang, Zhi
               and Chen, Chao and Wang, Yusu},
  booktitle = {Proceedings of the Second Learning on Graphs Conference},
  series    = {Proceedings of Machine Learning Research},
  volume    = {231},
  pages     = {4:1--4:21},
  year      = {2024},
  publisher = {PMLR}
}

@inproceedings{donnat2018graphwave,
  title     = {Learning Structural Node Embeddings via Diffusion Wavelets},
  author    = {Donnat, Claire and Zitnik, Marinka and Hallac, David
               and Leskovec, Jure},
  booktitle = {Proceedings of the 24th ACM SIGKDD International Conference
               on Knowledge Discovery and Data Mining},
  pages     = {1320--1329},
  year      = {2018},
  doi       = {10.1145/3219819.3220025}
}

@inproceedings{ribeiro2017struc2vec,
  title     = {struc2vec: Learning Node Representations from Structural
               Identity},
  author    = {Ribeiro, Leonardo F. R. and Saverese, Pedro H. P.
               and Figueiredo, Daniel R.},
  booktitle = {Proceedings of the 23rd ACM SIGKDD International Conference
               on Knowledge Discovery and Data Mining},
  pages     = {385--394},
  year      = {2017},
  doi       = {10.1145/3097983.3098061}
}

@inproceedings{chen2022sat,
  title     = {Structure-Aware Transformer for Graph Representation Learning},
  author    = {Chen, Dexiong and O'Bray, Leslie and Borgwardt, Karsten},
  booktitle = {Proceedings of the 39th International Conference on
               Machine Learning},
  series    = {Proceedings of Machine Learning Research},
  volume    = {162},
  pages     = {3469--3489},
  year      = {2022},
  publisher = {PMLR}
}

@inproceedings{kim2022tokengt,
  title     = {Pure Transformers are Powerful Graph Learners},
  author    = {Kim, Jinwoo and Nguyen, Dat and Min, Seonwoo and Cho, Sungjun
               and Lee, Moontae and Lee, Honglak and Hong, Seunghoon},
  booktitle = {Advances in Neural Information Processing Systems},
  volume    = {35},
  pages     = {14582--14595},
  year      = {2022}
}

@inproceedings{shirzad2023exphormer,
  title     = {Exphormer: Sparse Transformers for Graphs},
  author    = {Shirzad, Hamed and Velingker, Ameya and Venkatachalam, Balaji
               and Sutherland, Danica J. and Sinop, Ali Kemal},
  booktitle = {Proceedings of the 40th International Conference on
               Machine Learning},
  series    = {Proceedings of Machine Learning Research},
  volume    = {202},
  pages     = {31613--31632},
  year      = {2023},
  publisher = {PMLR}
}

@inproceedings{geisler2023directed,
  title     = {Transformers Meet Directed Graphs},
  author    = {Geisler, Simon and Li, Yujia and Mankowitz, Daniel J.
               and Cemgil, Ali Taylan and G{\"u}nnemann, Stephan
               and Paduraru, Cosmin},
  booktitle = {Proceedings of the 40th International Conference on
               Machine Learning},
  series    = {Proceedings of Machine Learning Research},
  volume    = {202},
  pages     = {11144--11172},
  year      = {2023},
  publisher = {PMLR}
}

@inproceedings{xu2019powerful,
  title     = {How Powerful are Graph Neural Networks?},
  author    = {Xu, Keyulu and Hu, Weihua and Leskovec, Jure
               and Jegelka, Stefanie},
  booktitle = {International Conference on Learning Representations},
  year      = {2019}
}

@inproceedings{morris2019weisfeiler,
  title     = {Weisfeiler and Leman Go Neural: Higher-Order Graph Neural
               Networks},
  author    = {Morris, Christopher and Ritzert, Martin and Fey, Matthias
               and Hamilton, William L. and Lenssen, Jan Eric
               and Rattan, Gaurav and Grohe, Martin},
  booktitle = {Proceedings of the AAAI Conference on Artificial Intelligence},
  volume    = {33},
  number    = {1},
  pages     = {4602--4609},
  year      = {2019},
  doi       = {10.1609/aaai.v33i01.33014602}
}

@article{morris2023wlstory,
  title   = {Weisfeiler and Leman Go Machine Learning:
             The Story so Far},
  author  = {Morris, Christopher and Lipman, Yaron and Maron, Haggai
             and Rieck, Bastian and Kriege, Nils M. and Grohe, Martin
             and Fey, Matthias and Borgwardt, Karsten},
  journal = {Journal of Machine Learning Research},
  volume  = {24},
  number  = {333},
  pages   = {1--59},
  year    = {2023}
}

@article{khuller1996landmarks,
  title   = {Landmarks in Graphs},
  author  = {Khuller, Samir and Raghavachari, Balaji
             and Rosenfeld, Azriel},
  journal = {Discrete Applied Mathematics},
  volume  = {70},
  number  = {3},
  pages   = {217--229},
  year    = {1996},
  doi     = {10.1016/0166-218X(95)00106-2}
}

@article{chartrand2000resolvability,
  title   = {Resolvability in Graphs and the Metric Dimension of a Graph},
  author  = {Chartrand, Gary and Eroh, Linda and Johnson, Mark A.
             and Oellermann, Ortrud R.},
  journal = {Discrete Applied Mathematics},
  volume  = {105},
  number  = {1--3},
  pages   = {99--113},
  year    = {2000},
  doi     = {10.1016/S0166-218X(00)00198-0}
}

@article{bollobas2013metric,
  title   = {Metric Dimension for Random Graphs},
  author  = {Bollob{\'a}s, B{\'e}la and Mitsche, Dieter and Pra{\l}at, Pawe{\l}},
  journal = {The Electronic Journal of Combinatorics},
  volume  = {20},
  number  = {4},
  pages   = {P1},
  year    = {2013},
  doi     = {10.37236/2639}
}

@article{dwivedi2023benchmark,
  title = {Benchmarking Graph Neural Networks},
  author = {Dwivedi, Vijay Prakash and Joshi, Chaitanya K. and Luu, Anh Tuan and Laurent, Thomas and Bengio, Yoshua and Bresson, Xavier},
  journal = {Journal of Machine Learning Research},
  volume = {24},
  number = {43},
  pages = {1--48},
  year = {2023}
}

@inproceedings{ying2021graphormer,
  title = {Do Transformers Really Perform Badly for Graph Representation?},
  author = {Ying, Chengxuan and Cai, Tianle and Luo, Shengjie and Zheng, Shuxin and Ke, Guolin and He, Di and Shen, Yanming and Liu, Tie-Yan},
  booktitle = {Advances in Neural Information Processing Systems},
  volume = {34},
  pages = {28877--28888},
  year = {2021}
}

@inproceedings{rampasek2022gps,
  title = {Recipe for a General, Powerful, Scalable Graph Transformer},
  author = {Rampasek, Ladislav and Galkin, Mikhail and Dwivedi, Vijay Prakash and Luu, Anh Tuan and Wolf, Guy and Beaini, Dominique},
  booktitle = {Advances in Neural Information Processing Systems},
  volume = {35},
  pages = {14501--14515},
  year = {2022}
}

@inproceedings{canturk2024gpse,
  title     = {Graph Positional and Structural Encoder},
  author    = {Cant{\"u}rk, Semih and Liu, Renming and Lapointe-Gagn{\'e}, Olivier
               and L{\'e}tourneau, Vincent and Wolf, Guy and Beaini, Dominique
               and Ramp{\'a}{\v{s}}ek, Ladislav},
  booktitle = {Proceedings of the 41st International Conference on Machine Learning},
  series    = {Proceedings of Machine Learning Research},
  volume    = {235},
  pages     = {5533--5566},
  year      = {2024},
  publisher = {PMLR}
}

@inproceedings{lim2023signnet,
  title = {Sign and Basis Invariant Networks for Spectral Graph Representation Learning},
  author = {Lim, Derek and Robinson, Joshua and Zhao, Lingxiao and Smidt, Tess and Sra, Suvrit and Maron, Haggai and Jegelka, Stefanie},
  booktitle = {International Conference on Learning Representations},
  year = {2023}
}

@inproceedings{huang2024stability,
  title = {On the Stability of Expressive Positional Encodings for Graphs},
  author = {Huang, Yinan and Lu, William and Robinson, Joshua and Yang, Yu and Zhang, Muhan and Jegelka, Stefanie and Li, Pan},
  booktitle = {International Conference on Learning Representations},
  year = {2024}
}

@inproceedings{li2020distance,
  title = {Distance Encoding: Design Provably More Powerful Neural Networks for Graph Representation Learning},
  author = {Li, Pan and Wang, Yanbang and Wang, Hongwei and Leskovec, Jure},
  booktitle = {Advances in Neural Information Processing Systems},
  volume = {33},
  pages = {4465--4478},
  year = {2020}
}

@inproceedings{dwivedi2022learnable,
  title = {Graph Neural Networks with Learnable Structural and Positional Representations},
  author = {Dwivedi, Vijay Prakash and Luu, Anh Tuan and Laurent, Thomas and Bengio, Yoshua and Bresson, Xavier},
  booktitle = {International Conference on Learning Representations},
  year = {2022}
}

@inproceedings{keriven2023role,
  title = {What Functions Can Graph Neural Networks Compute on Random Graphs? The Role of Positional Encoding},
  author = {Keriven, Nicolas and Vaiter, Samuel},
  booktitle = {Advances in Neural Information Processing Systems},
  volume = {36},
  year = {2023}
}

@inproceedings{li2024generalization,
  title = {What Improves the Generalization of Graph Transformers? A Theoretical Dive into the Self-Attention and Positional Encoding},
  author = {Li, Hongkang and Wang, Meng and Ma, Tengfei and Liu, Sijia and Zhang, Zaixi and Chen, Pin-Yu},
  booktitle = {Proceedings of the 41st International Conference on Machine Learning},
  series = {Proceedings of Machine Learning Research},
  volume = {235},
  pages = {28784--28829},
  year = {2024}
}

@inproceedings{marcheggiani2017encoding,
  title = {Encoding Sentences with Graph Convolutional Networks for Semantic Role Labeling},
  author = {Marcheggiani, Diego and Titov, Ivan},
  booktitle = {Proceedings of the 2017 Conference on Empirical Methods in Natural Language Processing},
  pages = {1506--1515},
  year = {2017}
}

@inproceedings{sinha2019clutrr,
  title = {{CLUTRR}: A Diagnostic Benchmark for Inductive Reasoning from Text},
  author = {Sinha, Koustuv and Sodhani, Shagun and Dong, Jin and Pineau, Joelle and Hamilton, William L.},
  booktitle = {Proceedings of the 2019 Conference on Empirical Methods in Natural Language Processing and the 9th International Joint Conference on Natural Language Processing},
  pages = {4506--4515},
  year = {2019}
}

@article{demarneffe2021universal,
  title = {Universal Dependencies},
  author = {de Marneffe, Marie-Catherine and Manning, Christopher D. and Nivre, Joakim and Zeman, Daniel},
  journal = {Computational Linguistics},
  volume = {47},
  number = {2},
  pages = {255--308},
  year = {2021}
}

@article{belkin2003laplacian,
  title = {Laplacian Eigenmaps for Dimensionality Reduction and Data Representation},
  author = {Belkin, Mikhail and Niyogi, Partha},
  journal = {Neural Computation},
  volume = {15},
  number = {6},
  pages = {1373--1396},
  year = {2003},
  doi = {10.1162/089976603321780317}
}

@inproceedings{kreuzer2021rethinking,
  title = {Rethinking Graph Transformers with Spectral Attention},
  author = {Kreuzer, Devin and Beaini, Dominique and Hamilton, William L. and L{\'e}tourneau, Vincent and Tossou, Prudencio},
  booktitle = {Advances in Neural Information Processing Systems},
  volume = {34},
  pages = {21618--21629},
  year = {2021},
  url = {https://proceedings.neurips.cc/paper_files/paper/2021/hash/b4fd1d2cb085390fbbadae65e07876a7-Abstract.html}
}

@inproceedings{you2019position,
  title = {Position-aware Graph Neural Networks},
  author = {You, Jiaxuan and Ying, Rex and Leskovec, Jure},
  booktitle = {Proceedings of the 36th International Conference on Machine Learning},
  series = {Proceedings of Machine Learning Research},
  volume = {97},
  pages = {7134--7143},
  publisher = {PMLR},
  year = {2019},
  url = {https://proceedings.mlr.press/v97/you19b.html}
}

@article{park2022grpe,
  title = {{GRPE}: Relative Positional Encoding for Graph Transformer},
  author = {Park, Wonpyo and Chang, Woonggi and Lee, Donggeon and Kim, Juntae and Hwang, Seung-won},
  journal = {arXiv preprint arXiv:2201.12787},
  year = {2022},
  url = {https://arxiv.org/abs/2201.12787}
}

@article{bruelgabrielsson2023rewiring,
  title = {Rewiring with Positional Encodings for Graph Neural Networks},
  author = {Br{\"u}el-Gabrielsson, Rickard and Yurochkin, Mikhail and Solomon, Justin},
  journal = {Transactions on Machine Learning Research},
  year = {2023},
  url = {https://openreview.net/forum?id=dn3ZkqG2YV}
}

@inproceedings{ma2023grit,
  title = {Graph Inductive Biases in Transformers without Message Passing},
  author = {Ma, Liheng and Lin, Chen and Lim, Derek and Romero-Soriano, Adriana and Dokania, Puneet K. and Coates, Mark and Torr, Philip H. S. and Lim, Ser-Nam},
  booktitle = {Proceedings of the 40th International Conference on Machine Learning},
  series = {Proceedings of Machine Learning Research},
  volume = {202},
  pages = {23321--23337},
  publisher = {PMLR},
  year = {2023},
  url = {https://proceedings.mlr.press/v202/ma23c.html}
}

@inproceedings{black2024comparing,
  title = {Comparing Graph Transformers via Positional Encodings},
  author = {Black, Mitchell and Wan, Zhengchao and Mishne, Gal and Nayyeri, Amir and Wang, Yusu},
  booktitle = {Proceedings of the 41st International Conference on Machine Learning},
  series = {Proceedings of Machine Learning Research},
  volume = {235},
  pages = {4103--4139},
  publisher = {PMLR},
  year = {2024},
  url = {https://proceedings.mlr.press/v235/black24b.html}
}

@inproceedings{mathieu2021simple,
  title = {A Simple Algorithm for Graph Reconstruction},
  author = {Mathieu, Claire and Zhou, Hang},
  booktitle = {29th Annual European Symposium on Algorithms},
  series = {Leibniz International Proceedings in Informatics},
  volume = {204},
  pages = {68:1--68:18},
  publisher = {Schloss Dagstuhl -- Leibniz-Zentrum f{\"u}r Informatik},
  year = {2021},
  doi = {10.4230/LIPIcs.ESA.2021.68},
  url = {https://drops.dagstuhl.de/entities/document/10.4230/LIPIcs.ESA.2021.68}
}

@inproceedings{slater1975leaves,
  title = {Leaves of Trees},
  author = {Slater, Peter J.},
  booktitle = {Proceedings of the Sixth Southeastern Conference on Combinatorics, Graph Theory, and Computing},
  series = {Congressus Numerantium},
  volume = {14},
  pages = {549--559},
  year = {1975}
}

@article{harary1976metric,
  title = {On the Metric Dimension of a Graph},
  author = {Harary, Frank and Melter, Robert A.},
  journal = {Ars Combinatoria},
  volume = {2},
  pages = {191--195},
  year = {1976}
}

@misc{yan2026informationtheoreticlimitsnodelocalization,
      title={Information-Theoretic Limits of Node Localization under Hybrid Graph Positional Encodings}, 
      author={Zimo Yan and Zheng Xie and Chang Liu and Yiqin Lv and Runfan Duan},
      year={2026},
      eprint={2603.25030},
      archivePrefix={arXiv},
      primaryClass={cs.IT},
      url={https://arxiv.org/abs/2603.25030}, 
}

\clearpage

\appendices

\newtheorem*{lemmaR}{Lemma}
\newtheorem*{theoremR}{Theorem}
\newtheorem*{propositionR}{Proposition}
\newtheorem*{corollaryR}{Corollary}
\newtheorem*{assumptionR}{Assumption}
\newtheorem*{definitionR}{Definition}
\section{Theoretical Derivations}
\label{app:proofs}

This appendix provides the full theoretical derivations for the results stated in Section~\ref{SE5}. 
For readability, each result is restated before its proof. 

\subsection{Proofs for the Normalized Refined Converse}
\label{app:conv}
We first prove
the deterministic image-size bound and then specialize it to random
regular graphs using the standard logarithmic diameter control.
\subsubsection{Proof of Theorem~\ref{thm:normalized-refined-converse}}
\label{app:proof-normalized-refined-converse}

\begin{theoremR}[Restatement of Theorem~\ref{thm:normalized-refined-converse}]
For any finite connected graph $G=(V,E)$, anchor set
$\mathcal A\subseteq V$, spectral dimension $m\in\{1,\dots,n-1\}$,
quantization level $\eta>0$, and trimming threshold $L>0$,
\begin{equation}
\bigl|\operatorname{Im}(F_{G,\mathcal A}^{(m,\eta)})\bigr|
\le
D(G,\mathcal A)
\binom{\lfloor L/\eta\rfloor+m}{m}
+
\frac{mn}{L}.
\end{equation}
Consequently,
\begin{equation}
\operatorname{Err}^*
\bigl(F_{G,\mathcal A}^{(m,\eta)}\bigr)
\ge
1
-
\frac{D(G,\mathcal A)}{n}
\binom{\lfloor L/\eta\rfloor+m}{m}
-
\frac{m}{L}.
\end{equation}
\end{theoremR}

\begin{proof}
Let
\begin{equation}
S_m(v):=\sum_{j=1}^{m}X_m(v)_j .
\end{equation}
By orthonormality of the Laplacian eigenbasis,
\begin{equation}
\sum_{v\in V}S_m(v)=mn.
\end{equation}
Decompose
\begin{equation}
V_{\mathrm{in}}
:=
\{v\in V:S_m(v)\le L\},
\quad
V_{\mathrm{out}}
:=
V\setminus V_{\mathrm{in}}.
\end{equation}
Then
\begin{equation}
\bigl|\operatorname{Im}(F_{G,\mathcal A}^{(m,\eta)})\bigr|
\le
\bigl|F_{G,\mathcal A}^{(m,\eta)}(V_{\mathrm{in}})\bigr|
+
|V_{\mathrm{out}}|.
\end{equation}

We first bound the trimmed part. Since $S_m(v)>L$ for
$v\in V_{\mathrm{out}}$,
\begin{equation}
L|V_{\mathrm{out}}|
\le
\sum_{v\in V_{\mathrm{out}}}S_m(v)
\le
\sum_{v\in V}S_m(v)
=
mn.
\end{equation}
Hence
\begin{equation}
|V_{\mathrm{out}}|\le \frac{mn}{L}.
\end{equation}

It remains to count the possible codes on $V_{\mathrm{in}}$.
For $v\in V_{\mathrm{in}}$, write
\begin{equation}
q(v):=Q_\eta(X_m(v))\in\mathbb Z_{\ge0}^{m}.
\end{equation}
Since $X_m(v)_j\ge0$ and $\sum_{j=1}^{m}X_m(v)_j\le L$,
\begin{equation}
\sum_{j=1}^{m}q(v)_j
=
\sum_{j=1}^{m}
\left\lfloor\frac{X_m(v)_j}{\eta}\right\rfloor
\le
\left\lfloor
\frac{1}{\eta}\sum_{j=1}^{m}X_m(v)_j
\right\rfloor
\le
\left\lfloor\frac{L}{\eta}\right\rfloor .
\end{equation}
Set
\begin{equation}
R:=\left\lfloor\frac{L}{\eta}\right\rfloor .
\end{equation}
Thus the quantized spectral code belongs to
\begin{equation}
\mathcal Q_{m,R}
:=
\left\{
q\in\mathbb Z_{\ge0}^{m}:
\sum_{j=1}^{m}q_j\le R
\right\}.
\end{equation}
By the standard stars-and-bars count,
\begin{equation}
|\mathcal Q_{m,R}|=\binom{R+m}{m}.
\end{equation}
The distance component takes at most $D(G,\mathcal A)$
values. Therefore,
\begin{equation}
\bigl|F_{G,\mathcal A}^{(m,\eta)}(V_{\mathrm{in}})\bigr|
\le
D(G,\mathcal A)
\binom{\lfloor L/\eta\rfloor+m}{m}.
\end{equation}
Combining the estimates gives
\begin{equation}
\bigl|\operatorname{Im}(F_{G,\mathcal A}^{(m,\eta)})\bigr|
\le
D(G,\mathcal A)
\binom{\lfloor L/\eta\rfloor+m}{m}
+
\frac{mn}{L}.
\end{equation}

Finally, the observation-map identity gives
\begin{equation}
\operatorname{Err}^*
\bigl(F_{G,\mathcal A}^{(m,\eta)}\bigr)
=
1-
\frac{
\bigl|\operatorname{Im}(F_{G,\mathcal A}^{(m,\eta)})\bigr|
}{n}.
\end{equation}
Substituting the image-size bound proves the stated error
lower bound.
\end{proof}

\subsubsection{Random-regular subcritical regime}
\label{app:proof-random-regular-subcritical}

\begin{corollaryR}[Random-regular subcritical regime]
\label{cor:random-regular-subcritical}
Fix $r\ge 3$. Suppose $G_n\sim\mathcal{G}_{n,r}$ is the
uniform random $r$-regular graph on $n$ vertices, and
$\mathcal A_n\subseteq V(G_n)$ is sampled uniformly without
replacement, independently of $G_n$, with $|\mathcal A_n|=k_n$.
Then there exists a constant $C_r>0$, depending only on $r$,
such that for any $L_n>0$,
\begin{equation}
\resizebox{\columnwidth}{!}{$
\operatorname{Err}^*
\bigl(F_{G_n,\mathcal A_n}^{(m_n,\eta_n)}\bigr)
\ge
1
-
\frac{(C_r\log n+1)^{k_n}}{n}
\binom{\lfloor L_n/\eta_n\rfloor+m_n}{m_n}
-
\frac{m_n}{L_n}
$}
\end{equation}
with probability tending to one as $n\to\infty$.

In particular, if
\begin{equation}
\frac{m_n}{L_n}\to 0
\end{equation}
and
\begin{equation}
(C_r\log n+1)^{k_n}
\binom{\lfloor L_n/\eta_n\rfloor+m_n}{m_n}
=
o(n),
\end{equation}
then
\begin{equation}
\operatorname{Err}^*
\bigl(F_{G_n,\mathcal A_n}^{(m_n,\eta_n)}\bigr)
\xrightarrow{\mathbb P}1.
\end{equation}

Equivalently, choosing $L_n=\log\log n$, the preceding
condition is implied by
\begin{equation}
m_n=o(\log\log n)
\end{equation}
and, for some $\varepsilon>0$,
\begin{equation}
\resizebox{\columnwidth}{!}{$
k_n\log(C_r\log n+1)
+
\log
\binom{
\left\lfloor
\frac{\log\log n}{\eta_n}
\right\rfloor
+
m_n
}{m_n}
\le
(1-\varepsilon)\log n
$}
\end{equation}
for all sufficiently large $n$.

Moreover, if $0<\eta_n\le 1$ for all sufficiently large $n$,
then the above logarithmic budget is implied, up to lower-order
terms absorbed by the slack $\varepsilon\log n$, by
\begin{equation}
k_n\log\log n
+
m_n\log\frac{1}{\eta_n}
+
m_n\log\log\log n
\le
(1-\varepsilon)\log n.
\end{equation}
\end{corollaryR}

\begin{proof}
We use the standard diameter estimate for fixed-degree random
regular graphs. Since $r\ge 3$ is fixed, there exists a constant
$C_r>0$, depending only on $r$, such that
\begin{equation}
\mathbb P\bigl(
G_n \text{ is connected and } \operatorname{diam}(G_n)\le C_r\log n
\bigr)
\to 1 .
\end{equation}
Let this high-probability event be denoted by $\mathcal E_n$.

On $\mathcal E_n$, each coordinate of the anchor-distance
profile takes values in
\begin{equation}
\{0,1,\dots,\operatorname{diam}(G_n)\}.
\end{equation}
Hence
\begin{equation}
D(G_n,\mathcal A_n)
\le
\bigl(\operatorname{diam}(G_n)+1\bigr)^{k_n}
\le
(C_r\log n+1)^{k_n}.
\end{equation}
Applying Theorem~\ref{thm:normalized-refined-converse} on
$\mathcal E_n$ gives
\begin{equation}
\resizebox{\columnwidth}{!}{$
\operatorname{Err}^*
\bigl(F_{G_n,\mathcal A_n}^{(m_n,\eta_n)}\bigr)
\ge
1
-
\frac{(C_r\log n+1)^{k_n}}{n}
\binom{\lfloor L_n/\eta_n\rfloor+m_n}{m_n}
-
\frac{m_n}{L_n}.
$}
\end{equation}
Since $\mathbb P(\mathcal E_n)\to1$, the first claim follows.

For the convergence statement, define
\begin{equation}
a_n
:=
\frac{(C_r\log n+1)^{k_n}}{n}
\binom{\lfloor L_n/\eta_n\rfloor+m_n}{m_n},
\quad
b_n
:=
\frac{m_n}{L_n}.
\end{equation}
The assumptions give $a_n\to0$ and $b_n\to0$. Thus, for any
fixed $\delta>0$, with probability tending to one,
\begin{equation}
\operatorname{Err}^*
\bigl(F_{G_n,\mathcal A_n}^{(m_n,\eta_n)}\bigr)
\ge
1-a_n-b_n
>
1-\delta .
\end{equation}
This proves
\begin{equation}
\operatorname{Err}^*
\bigl(F_{G_n,\mathcal A_n}^{(m_n,\eta_n)}\bigr)
\xrightarrow{\mathbb P}1.
\end{equation}

It remains to justify the logarithmic sufficient conditions.
Set $L_n=\log\log n$. Then $m_n=o(\log\log n)$ implies
$m_n/L_n\to0$. Moreover,
\begin{equation}
\begin{aligned}
&(C_r\log n+1)^{k_n}
\binom{
\left\lfloor
\frac{\log\log n}{\eta_n}
\right\rfloor
+
m_n
}{m_n}\\
=&
\exp\left\{
k_n\log(C_r\log n+1)
+
\log
\binom{
\left\lfloor
\frac{\log\log n}{\eta_n}
\right\rfloor
+
m_n
}{m_n}
\right\}.
\end{aligned}
\end{equation}
Therefore, the displayed logarithmic budget with positive slack
implies the required $o(n)$ condition.

Finally, since
\begin{equation}
\binom{R+m}{m}\le (R+1)^m
\end{equation}
for all $R,m\ge0$, the simplex-refined budget is no larger than
the previous box budget. Hence, for $0<\eta_n\le1$ and
$L_n=\log\log n$,
\begin{equation}
\log
\binom{
\left\lfloor
\frac{\log\log n}{\eta_n}
\right\rfloor
+
m_n
}{m_n}
\le
m_n
\log
\left(
\left\lceil
\frac{\log\log n}{\eta_n}
\right\rceil
+1
\right).
\end{equation}
As in the box-counting bound,
\begin{equation}
\resizebox{\columnwidth}{!}{$
m_n
\log
\left(
\left\lceil
\frac{\log\log n}{\eta_n}
\right\rceil
+1
\right)
\le
m_n\log\frac{1}{\eta_n}
+
m_n\log\log\log n
+
O(m_n).
$}
\end{equation}
Also,
\begin{equation}
\log(C_r\log n+1)
\le
\log\log n+O(1).
\end{equation}
The remaining $O(k_n)+O(m_n)$ terms are $o(\log n)$ under
the stated budget and $m_n=o(\log\log n)$, and can be absorbed
by reducing the slack. This proves the simplified sufficient
condition.
\end{proof}

\subsection{Proofs for Collision-Based Achievability}
\label{app:collision}

This appendix proves the collision identities used in
Section~\ref{subsec:collision-based-achievability}. The key point is that
hybrid collisions decompose exactly into distance collisions and residual
spectral collisions inside distance buckets.
\subsubsection{Proof of Lemma~\ref{lem:basic-collision-lemma}}
\label{app:proof-basic-collision-lemma}

\begin{lemmaR}[Restatement of Lemma~\ref{lem:basic-collision-lemma}]
For any observation map $F:V\to\mathcal Y$ on a finite set $V$ with $|V|=n$,
\begin{equation}
\operatorname{Err}^*(F)
\le
(n-1)\kappa_F.
\end{equation}
Moreover, the hybrid map satisfies the exact factorization
\begin{equation}
\kappa_{\mathrm H}
=
\kappa_{\mathrm D}
\cdot
\kappa_{\mathrm{S}\mid\mathrm D}.
\end{equation}
Consequently,
\begin{equation}
\operatorname{Err}^*
\bigl(
F_{G,\mathcal A}^{(m,\eta)}
\bigr)
\le
(n-1)
\kappa_{\mathrm D}
\kappa_{\mathrm{S}\mid\mathrm D}.
\end{equation}
\end{lemmaR}

\begin{proof}
We first prove the general collision bound. Let
\begin{equation}
\mathcal Y_F:=\operatorname{Im}(F).
\end{equation}
For each \(y\in\mathcal Y_F\), define the fiber
\begin{equation}
C_y:=\{v\in V:F(v)=y\},
\end{equation}
and write
\begin{equation}
s_y:=|C_y|.
\end{equation}
Since \(\mathcal Y_F\) is the image of \(F\), each fiber is nonempty, and hence
\begin{equation}
s_y\ge 1,
\quad y\in\mathcal Y_F.
\end{equation}
Moreover, the fibers \(\{C_y\}_{y\in\mathcal Y_F}\) form a partition of \(V\), so
\begin{equation}
\sum_{y\in\mathcal Y_F}s_y=n.
\end{equation}

By the image-size identity,
\begin{equation}
\operatorname{Err}^*(F)
=
1-\frac{|\operatorname{Im}(F)|}{n}.
\end{equation}
Since
\begin{equation}
|\operatorname{Im}(F)|
=
|\mathcal Y_F|
=
\sum_{y\in\mathcal Y_F}1,
\end{equation}
we obtain
\begin{equation}
\begin{aligned}
\operatorname{Err}^*(F)
&=
1-\frac{|\mathcal Y_F|}{n}
\\
&=
\frac{n-|\mathcal Y_F|}{n}
\\
&=
\frac{
\sum_{y\in\mathcal Y_F}s_y
-
\sum_{y\in\mathcal Y_F}1
}{n}
\\
&=
\frac{1}{n}
\sum_{y\in\mathcal Y_F}(s_y-1).
\end{aligned}
\end{equation}

Next, because \(U,V\) are sampled uniformly from \(V\) without replacement, the collision rate of \(F\) is
\begin{equation}
\kappa_F
=
\mathbb P(F(U)=F(V))
=
\frac{1}{n(n-1)}
\sum_{\substack{u,v\in V\\u\neq v}}
\mathbf 1\{F(u)=F(v)\}.
\end{equation}
The ordered pairs \((u,v)\) with \(u\neq v\) and \(F(u)=F(v)=y\) are exactly the ordered pairs inside the fiber \(C_y\). Their number is
\begin{equation}
s_y(s_y-1).
\end{equation}
Therefore,
\begin{equation}
\kappa_F
=
\frac{1}{n(n-1)}
\sum_{y\in\mathcal Y_F}s_y(s_y-1).
\end{equation}
Multiplying both sides by \(n-1\), we get
\begin{equation}
(n-1)\kappa_F
=
\frac{1}{n}
\sum_{y\in\mathcal Y_F}s_y(s_y-1).
\end{equation}
Since \(s_y\ge 1\), we have
\begin{equation}
s_y-1
\le
s_y(s_y-1)
\end{equation}
for every \(y\in\mathcal Y_F\). Hence
\begin{equation}
\begin{aligned}
\operatorname{Err}^*(F)
&=
\frac{1}{n}
\sum_{y\in\mathcal Y_F}(s_y-1)
\\
&\le
\frac{1}{n}
\sum_{y\in\mathcal Y_F}s_y(s_y-1)
\\
&=
(n-1)\kappa_F.
\end{aligned}
\end{equation}
This proves the first claim.

We now prove the factorization for the hybrid map. Recall that
\begin{equation}
F_{G,\mathcal A}^{(m,\eta)}(v)
=
\bigl(D_{\mathcal A}(v),Z_{m,\eta}(v)\bigr).
\end{equation}
Thus two vertices \(u\neq v\) collide under the hybrid map if and only if
\begin{equation}
D_{\mathcal A}(u)=D_{\mathcal A}(v)
\end{equation}
and
\begin{equation}
Z_{m,\eta}(u)=Z_{m,\eta}(v).
\end{equation}
The first condition means that \(u\) and \(v\) lie in the same distance bucket \(B_t\) for some \(t\in\mathcal T_{G,\mathcal A}\). Therefore,
\begin{equation}
\begin{aligned}
\kappa_{\mathrm H}
&=
\mathbb P\bigl(
F_{G,\mathcal A}^{(m,\eta)}(U)
=
F_{G,\mathcal A}^{(m,\eta)}(V)
\bigr)
\\
&=
\frac{1}{n(n-1)}
\sum_{t\in\mathcal T_{G,\mathcal A}}
\sum_{\substack{u,v\in B_t\\u\neq v}}
\mathbf 1\{
Z_{m,\eta}(u)=Z_{m,\eta}(v)
\}.
\end{aligned}
\end{equation}

For each \(t\), by the definition of \(\kappa_{\mathrm S}(B_t)\), we have
\begin{equation}
\sum_{\substack{u,v\in B_t\\u\neq v}}
\mathbf 1\{
Z_{m,\eta}(u)=Z_{m,\eta}(v)
\}
=
|B_t|(|B_t|-1)\kappa_{\mathrm S}(B_t).
\end{equation}
This identity also holds when \(|B_t|\le 1\), because both sides are then equal to zero. Hence
\begin{equation}
\kappa_{\mathrm H}
=
\frac{1}{n(n-1)}
\sum_{t\in\mathcal T_{G,\mathcal A}}
|B_t|(|B_t|-1)\kappa_{\mathrm S}(B_t).
\end{equation}

Let
\begin{equation}
W
:=
\sum_{t\in\mathcal T_{G,\mathcal A}}
|B_t|(|B_t|-1).
\end{equation}
Then
\begin{equation}
\kappa_{\mathrm D}
=
\frac{W}{n(n-1)}.
\end{equation}

If \(W=0\), then every bucket has size at most one. Hence there are no distance collisions, so
\begin{equation}
\kappa_{\mathrm D}=0.
\end{equation}
There are also no hybrid collisions, and therefore
\begin{equation}
\kappa_{\mathrm H}=0.
\end{equation}
By convention,
\begin{equation}
\kappa_{\mathrm{S}\mid\mathrm D}=0.
\end{equation}
Thus
\begin{equation}
\kappa_{\mathrm H}
=
0
=
\kappa_{\mathrm D}\kappa_{\mathrm{S}\mid\mathrm D}. 
\end{equation}

Consider the case \(W>0\). By definition,
\begin{equation}
\kappa_{\mathrm{S}\mid\mathrm D}
=
\frac{
\sum_{t\in\mathcal T_{G,\mathcal A}}
|B_t|(|B_t|-1)\kappa_{\mathrm S}(B_t)
}{W}.
\end{equation}
Therefore,
\begin{equation}
\begin{aligned}
\kappa_{\mathrm D}\kappa_{\mathrm{S}\mid\mathrm D}
&=
\frac{W}{n(n-1)}
\cdot
\frac{
\sum_{t\in\mathcal T_{G,\mathcal A}}
|B_t|(|B_t|-1)\kappa_{\mathrm S}(B_t)
}{W}
\\
&=
\frac{1}{n(n-1)}
\sum_{t\in\mathcal T_{G,\mathcal A}}
|B_t|(|B_t|-1)\kappa_{\mathrm S}(B_t)
\\
&=
\kappa_{\mathrm H}.
\end{aligned}
\end{equation}
This proves the exact factorization
\begin{equation}
\kappa_{\mathrm H}
=
\kappa_{\mathrm D}
\kappa_{\mathrm{S}\mid\mathrm D}.
\end{equation}
This factorization is purely deterministic and does not rely on any independence assumption between the distance and spectral components.

Finally, applying the general collision bound to the hybrid observation map gives
\begin{equation}
\operatorname{Err}^*
\bigl(
F_{G,\mathcal A}^{(m,\eta)}
\bigr)
\le
(n-1)\kappa_{\mathrm H}.
\end{equation}
Using the factorization just proved, we obtain
\begin{equation}
\operatorname{Err}^*
\bigl(
F_{G,\mathcal A}^{(m,\eta)}
\bigr)
\le
(n-1)
\kappa_{\mathrm D}
\kappa_{\mathrm{S}\mid\mathrm D}.
\end{equation}
The proof is complete.
\end{proof}

\subsubsection{Equal-distance representation for random anchors}
\label{app:proof-bisector-representation}

\begin{lemmaR}[Equal-distance representation for random anchors]
\label{lem:bisector-representation}
Suppose \(\mathcal A\) consists of \(k\) vertices sampled uniformly from \(V\) without replacement. Then
\begin{equation}
\mathbb E_{\mathcal A}\kappa_{\mathrm D}
=
\mathbb E_{U,V}
\left[
\frac{(|\operatorname{Eq}_G(U,V)|)_k}{(n)_k}
\right],
\end{equation}
where \(U,V\) are sampled uniformly from \(V\) without replacement. In particular,
\begin{equation}
\mathbb E_{\mathcal A}\kappa_{\mathrm D}
\le
\mathbb E_{U,V}
\bigl[\beta_G(U,V)^k\bigr].
\end{equation}
\end{lemmaR}

\begin{proof}
Fix two distinct vertices \(u,v\in V\). By definition,
\begin{equation}
D_{\mathcal A}(u)=D_{\mathcal A}(v)
\end{equation}
if and only if
\begin{equation}
d_G(u,a)=d_G(v,a)
\end{equation}
for every anchor \(a\in\mathcal A\). Equivalently,
\begin{equation}
\mathcal A\subseteq \operatorname{Eq}_G(u,v),
\end{equation}
where
\begin{equation}
\operatorname{Eq}_G(u,v)
=
\{a\in V:d_G(a,u)=d_G(a,v)\}.
\end{equation}
Let
\begin{equation}
b(u,v):=|\operatorname{Eq}_G(u,v)|.
\end{equation}
Since \(\mathcal A\) is a uniformly sampled \(k\)-subset of \(V\), we have
\begin{equation}
\mathbb P_{\mathcal A}
\bigl(
D_{\mathcal A}(u)=D_{\mathcal A}(v)
\bigr)
=
\frac{\binom{b(u,v)}{k}}{\binom{n}{k}}
=
\frac{(b(u,v))_k}{(n)_k}.
\end{equation}
The equality remains valid when \(b(u,v)<k\), in which case both sides are zero.

Now recall that \(U,V\) are sampled uniformly from \(V\) without replacement. Hence
\begin{equation}
\begin{aligned}
\mathbb E_{\mathcal A}\kappa_{\mathrm D}
&=
\mathbb E_{\mathcal A}
\left[
\mathbb P_{U,V}
\bigl(
D_{\mathcal A}(U)=D_{\mathcal A}(V)
\bigr)
\right]
\\
&=
\mathbb E_{U,V}
\left[
\mathbb P_{\mathcal A}
\bigl(
D_{\mathcal A}(U)=D_{\mathcal A}(V)
\bigr)
\right]
\\
&=
\mathbb E_{U,V}
\left[
\frac{(|\operatorname{Eq}_G(U,V)|)_k}{(n)_k}
\right].
\end{aligned}
\end{equation}
This proves the claimed identity.

It remains to prove the inequality. Let
\begin{equation}
b:=|\operatorname{Eq}_G(u,v)|.
\end{equation}
If \(b<k\), then
\begin{equation}
\frac{(b)_k}{(n)_k}=0
\le
\left(\frac{b}{n}\right)^k.
\end{equation}
If \(b\ge k\), then
\begin{equation}
\frac{(b)_k}{(n)_k}
=
\prod_{i=0}^{k-1}\frac{b-i}{n-i}.
\end{equation}
For each \(0\le i\le k-1\), since \(b\le n\),
\begin{equation}
\frac{b-i}{n-i}
\le
\frac{b}{n}.
\end{equation}
Therefore,
\begin{equation}
\frac{(b)_k}{(n)_k}
\le
\left(\frac{b}{n}\right)^k.
\end{equation}
Using
\begin{equation}
\beta_G(u,v)
=
\frac{|\operatorname{Eq}_G(u,v)|}{n},
\end{equation}
we obtain
\begin{equation}
\frac{(|\operatorname{Eq}_G(u,v)|)_k}{(n)_k}
\le
\beta_G(u,v)^k.
\end{equation}
Taking expectation over \(U,V\) gives
\begin{equation}
\mathbb E_{\mathcal A}\kappa_{\mathrm D}
\le
\mathbb E_{U,V}
\bigl[\beta_G(U,V)^k\bigr].
\end{equation}
The proof is complete.
\end{proof}

\subsubsection{Proof of Theorem~\ref{thm:deterministic-collision-achievability}}
\label{app:proof-deterministic-collision-achievability}

\begin{theoremR}[Restatement of Theorem~\ref{thm:deterministic-collision-achievability}]
For a deterministic sequence of graphs, anchor sets, spectral dimensions, and quantization levels, if
\begin{equation}
n\,
\kappa_{\mathrm D}
\kappa_{\mathrm{S}\mid\mathrm D}
\to 0,
\end{equation}
then
\begin{equation}
\operatorname{Err}^*
\bigl(
F_{G,\mathcal A}^{(m,\eta)}
\bigr)
\to 0.
\end{equation}
Moreover, if the quantities above are random and
\begin{equation}
n\,
\kappa_{\mathrm D}
\kappa_{\mathrm{S}\mid\mathrm D}
\xrightarrow{\mathbb P}0,
\end{equation}
then
\begin{equation}
\operatorname{Err}^*
\bigl(
F_{G,\mathcal A}^{(m,\eta)}
\bigr)
\xrightarrow{\mathbb P}0.
\end{equation}
\end{theoremR}

\begin{proof}
For each graph in the sequence, apply Lemma~\ref{lem:basic-collision-lemma} to the hybrid observation map
\begin{equation}
F_{G,\mathcal A}^{(m,\eta)}.
\end{equation}
By the collision bound and the exact factorization proved there, we have
\begin{equation}
\operatorname{Err}^*
\bigl(
F_{G,\mathcal A}^{(m,\eta)}
\bigr)
\le
(n-1)
\kappa_{\mathrm D}
\kappa_{\mathrm{S}\mid\mathrm D}.
\end{equation}
Since \(n-1\le n\), it follows that
\begin{equation}
\operatorname{Err}^*
\bigl(
F_{G,\mathcal A}^{(m,\eta)}
\bigr)
\le
n
\kappa_{\mathrm D}
\kappa_{\mathrm{S}\mid\mathrm D}.
\end{equation}
If
\begin{equation}
n
\kappa_{\mathrm D}
\kappa_{\mathrm{S}\mid\mathrm D}
\to 0,
\end{equation}
then by the squeeze theorem,
\begin{equation}
\operatorname{Err}^*
\bigl(
F_{G,\mathcal A}^{(m,\eta)}
\bigr)
\to 0.
\end{equation}
The deterministic claim follows.

The convergence-in-probability claim is identical. If
\begin{equation}
n
\kappa_{\mathrm D}
\kappa_{\mathrm{S}\mid\mathrm D}
\xrightarrow{\mathbb P}0,
\end{equation}
then the deterministic inequality
\begin{equation}
0
\le
\operatorname{Err}^*
\bigl(
F_{G,\mathcal A}^{(m,\eta)}
\bigr)
\le
n
\kappa_{\mathrm D}
\kappa_{\mathrm{S}\mid\mathrm D}
\end{equation}
implies
\begin{equation}
\operatorname{Err}^*
\bigl(
F_{G,\mathcal A}^{(m,\eta)}
\bigr)
\xrightarrow{\mathbb P}0.
\end{equation}
The proof is complete.
\end{proof}

\subsection{Proofs for Random-Regular Achievability under Collision Decay}
\label{app:proof-Random-Regular}
This appendix provides the random-regular ingredients used in
Section~\ref{subsec:rrg-achievability-transfer}. We first control
distance collisions through a two-source distinguisher estimate, then
prove the Gaussian-wave spectral anti-concentration bound, and finally
combine the two estimates to obtain the closed Gaussian-wave
achievability theorem. The final part introduces the expanded-cell majorant and states the conditional spectral-collision assumption sufficient for
actual Laplacian-energy coordinates.
\subsubsection{Two-source distinguisher bound}
\label{app:proof-two-source-distinguisher-rrg}

\begin{theoremR}[Two-source distinguisher bound on random regular graphs]
\label{thm:two-source-distinguisher-rrg}
Fix \(r\ge 3\). Let \(G_n\sim\mathcal G_{n,r}\) be the uniform random
\(r\)-regular graph on \(n\) vertices. There exists a constant
\(c_r>0\), depending only on \(r\), such that with probability tending to
one,
\begin{equation}
\left|\mathcal D_{G_n}(u,v)\right|
\ge
\frac{c_r n}{\log n}
\end{equation}
for every pair \(u\neq v\) satisfying \(d_{G_n}(u,v)\ge 2\).
\end{theoremR}

\begin{proof}
We use the structural distinguisher lemma of Mathieu and
Zhou~\cite[Lemma~12]{mathieu2021simple}. Their definition of a
distinguisher for a vertex pair \(\{u,v\}\) is a vertex \(a\) satisfying
\begin{equation}
\bigl|d_G(a,u)-d_G(a,v)\bigr|>1 .
\end{equation}
Since graph distances are integer-valued, this is exactly the set
\begin{equation}
\mathcal D_G(u,v)
=
\left\{
a\in V(G):
\bigl|d_G(a,u)-d_G(a,v)\bigr|\ge 2
\right\}.
\end{equation}

We first work in the configuration model. Mathieu and Zhou prove that,
for fixed \(r\ge 3\), if \(G_n'\) is the multigraph generated by a
uniform random \(r\)-regular configuration, then for every fixed
unordered vertex pair \(\{u,v\}\),
\begin{equation}
\mathbb P
\left(
d_{G_n'}(u,v)\ge 2
\ \text{and}\
\left|\mathcal D_{G_n'}(u,v)\right|
\le
\frac{3n}{\log n}
\right)
=
o(n^{-2}).
\end{equation}
Equivalently, outside an event of probability \(o(n^{-2})\), every
non-adjacent fixed pair has at least \(3n/\log n\) strong
distinguishers.

There are at most \(n(n-1)/2\) unordered vertex pairs. Therefore, by a
union bound,
\begin{align}
&\mathbb P
\left(
\exists\, u\neq v:
d_{G_n'}(u,v)\ge 2
\ \text{and}\
\left|\mathcal D_{G_n'}(u,v)\right|
\le
\frac{3n}{\log n}
\right)
\\
&\quad\le
\frac{n(n-1)}{2}\cdot o(n^{-2})
=
o(1).
\end{align}
Hence, with probability tending to one in the configuration model,
\begin{equation}
\left|\mathcal D_{G_n'}(u,v)\right|
>
\frac{3n}{\log n}
\end{equation}
holds simultaneously for every pair \(u\neq v\) satisfying
\(d_{G_n'}(u,v)\ge 2\).

It remains to pass from the configuration model to the uniform simple
\(r\)-regular graph. For fixed \(r\ge 3\), the configuration model is
simple with probability bounded away from zero. Conditioning on the
simplicity event therefore transfers any \(1-o(1)\) event in the
configuration model to a \(1-o(1)\) event under the uniform simple
\(r\)-regular graph model. Thus the same simultaneous distinguisher
bound holds for \(G_n\sim\mathcal G_{n,r}\).

Taking, for instance, any fixed constant
\begin{equation}
0<c_r<3
\end{equation}
gives
\begin{equation}
\left|\mathcal D_{G_n}(u,v)\right|
\ge
\frac{c_r n}{\log n}
\end{equation}
simultaneously for all \(u\neq v\) with \(d_{G_n}(u,v)\ge 2\), with
probability tending to one. This proves the theorem.
\end{proof}

\subsubsection{Equal-distance moment decay}
\label{app:proof-equal-distance-moment-two-source}

\begin{propositionR}[Equal-distance moment decay from two-source competition]
\label{prop:equal-distance-moment-two-source}
Fix \(r\ge 3\). Let \(G_n\sim\mathcal G_{n,r}\), and let \(U,V\) be sampled
uniformly from \(V(G_n)\) without replacement. Suppose
\begin{equation}
\frac{k_n}{\log n}\to\infty
\end{equation}
and
\begin{equation}
k_n=o\bigl((\log n)^2\bigr).
\end{equation}
Let \(c_r>0\) be the constant in
Theorem~\ref{thm:two-source-distinguisher-rrg}, and define
\begin{equation}
I_{\mathrm d,n}(r)
:=
\frac{c_r}{4\log n},
\end{equation}
\begin{equation}
\omega_{\mathrm d,n}
:=
\frac{c_r k_n}{4\log n}.
\end{equation}
Then
\begin{equation}
\omega_{\mathrm d,n}\to\infty,
\end{equation}
and with probability tending to one over \(G_n\),
\begin{equation}
\mathbb E_{U,V}
\left[
\beta_{G_n}(U,V)^{k_n}
\,\middle|\,
G_n
\right]
\le
\exp\{-k_n I_{\mathrm d,n}(r)-\omega_{\mathrm d,n}\}.
\end{equation}
\end{propositionR}

\begin{proof}
Let \(\mathcal E_n\) denote the high-probability event from
Theorem~\ref{thm:two-source-distinguisher-rrg}. Thus
\begin{equation}
\mathbb P(\mathcal E_n)\to 1.
\end{equation}
We condition on a graph \(G_n\) for which \(\mathcal E_n\) holds.

Recall that
\begin{equation}
\operatorname{Eq}_{G_n}(u,v)
=
\left\{
a\in V(G_n):
d_{G_n}(a,u)=d_{G_n}(a,v)
\right\},
\end{equation}
and
\begin{equation}
\beta_{G_n}(u,v)
=
\frac{
|\operatorname{Eq}_{G_n}(u,v)|
}{n}.
\end{equation}
Since every vertex in \(\mathcal D_{G_n}(u,v)\) distinguishes \(u\) and \(v\), we have
\begin{equation}
\mathcal D_{G_n}(u,v)
\subseteq
V(G_n)\setminus \operatorname{Eq}_{G_n}(u,v).
\end{equation}
Therefore, on \(\mathcal E_n\), for every \(u\neq v\) with \(d_{G_n}(u,v)\ge 2\),
\begin{equation}
\begin{aligned}
|\operatorname{Eq}_{G_n}(u,v)|
&\le
n-
|\mathcal D_{G_n}(u,v)|
\\
&\le
n-\frac{c_r n}{\log n}.
\end{aligned}
\end{equation}
Consequently,
\begin{equation}
\beta_{G_n}(u,v)
\le
1-\frac{c_r}{\log n}
\end{equation}
for every non-adjacent pair \(u\neq v\).

It remains to control adjacent ordered pairs. Define
\begin{equation}
\mathsf{Adj}_n
:=
\left\{
(u,v)\in V(G_n)^2:
u\neq v,\ d_{G_n}(u,v)=1
\right\}.
\end{equation}
Since \(G_n\) is \(r\)-regular,
\begin{equation}
|\mathsf{Adj}_n|=rn.
\end{equation}
Because \(U,V\) are sampled uniformly without replacement,
\begin{align}
\mathbb P
\left(
(U,V)\in\mathsf{Adj}_n
\,\middle|\,
G_n
\right)
&=
\frac{|\mathsf{Adj}_n|}{n(n-1)}
\\
&=
\frac{rn}{n(n-1)}
\\
&=
\frac{r}{n-1}.
\end{align}

Using the trivial bound
\begin{equation}
\beta_{G_n}(u,v)\le 1
\end{equation}
on adjacent pairs and the two-source bound on non-adjacent pairs, we get
\begin{align}
\mathbb E_{U,V}
\left[
\beta_{G_n}(U,V)^{k_n}
\,\middle|\,
G_n
\right]
&\le
\frac{r}{n-1}
+
\left(1-\frac{c_r}{\log n}\right)^{k_n}.
\end{align}
For all sufficiently large \(n\), we have
\begin{equation}
0<\frac{c_r}{\log n}<1.
\end{equation}
Hence
\begin{equation}
\left(1-\frac{c_r}{\log n}\right)^{k_n}
\le
\exp\left\{
-\frac{c_r k_n}{\log n}
\right\}.
\end{equation}
Therefore,
\begin{equation}
\mathbb E_{U,V}
\left[
\beta_{G_n}(U,V)^{k_n}
\,\middle|\,
G_n
\right]
\le
\frac{r}{n-1}
+
\exp\left\{
-\frac{c_r k_n}{\log n}
\right\}.
\end{equation}

Define
\begin{equation}
A_n
:=
\exp\left\{
-\frac{c_r k_n}{2\log n}
\right\}.
\end{equation}
Since
\begin{equation}
\frac{k_n}{\log n}\to\infty,
\end{equation}
we have
\begin{equation}
\exp\left\{
-\frac{c_r k_n}{\log n}
\right\}
=
o(A_n).
\end{equation}
Moreover, since
\begin{equation}
k_n=o\bigl((\log n)^2\bigr),
\end{equation}
we have
\begin{equation}
\frac{c_r k_n}{2\log n}
=
o(\log n).
\end{equation}
Thus
\begin{align}
\frac{r}{n-1}
&=
\exp\{-\log n+O(1)\}
\\
&=
o(A_n).
\end{align}
Combining these two estimates gives, for all sufficiently large \(n\),
\begin{equation}
\frac{r}{n-1}
+
\exp\left\{
-\frac{c_r k_n}{\log n}
\right\}
\le
A_n.
\end{equation}
Hence, on \(\mathcal E_n\),
\begin{equation}
\mathbb E_{U,V}
\left[
\beta_{G_n}(U,V)^{k_n}
\,\middle|\,
G_n
\right]
\le
\exp\left\{
-\frac{c_r k_n}{2\log n}
\right\}.
\end{equation}

Finally, by the definitions
\begin{equation}
I_{\mathrm d,n}(r)
=
\frac{c_r}{4\log n}
\end{equation}
and
\begin{equation}
\omega_{\mathrm d,n}
=
\frac{c_r k_n}{4\log n},
\end{equation}
we have
\begin{align}
k_n I_{\mathrm d,n}(r)+\omega_{\mathrm d,n}
&=
k_n\cdot \frac{c_r}{4\log n}
+
\frac{c_r k_n}{4\log n}
\\
&=
\frac{c_r k_n}{2\log n}.
\end{align}
Therefore,
\begin{equation}
\mathbb E_{U,V}
\left[
\beta_{G_n}(U,V)^{k_n}
\,\middle|\,
G_n
\right]
\le
\exp\{-k_n I_{\mathrm d,n}(r)-\omega_{\mathrm d,n}\}.
\end{equation}
Since
\begin{equation}
\frac{k_n}{\log n}\to\infty,
\end{equation}
we also have
\begin{equation}
\omega_{\mathrm d,n}\to\infty.
\end{equation}
The proof is complete.
\end{proof}

\subsubsection{Distance collision decay}
\label{app:proof-distance-collision-decay-two-source}

\begin{corollaryR}[Distance collision decay from two-source competition]
\label{cor:distance-collision-decay-two-source}
Fix \(r\ge 3\). Let \(G_n\sim\mathcal G_{n,r}\), and let
\(\mathcal A_n\subseteq V(G_n)\) be sampled uniformly without replacement,
independently of \(G_n\), with
\begin{equation}
|\mathcal A_n|=k_n.
\end{equation}
Suppose
\begin{equation}
\frac{k_n}{\log n}\to\infty
\end{equation}
and
\begin{equation}
k_n=o\bigl((\log n)^2\bigr).
\end{equation}
Then, with
\begin{equation}
I_{\mathrm d,n}(r)
=
\frac{c_r}{4\log n},
\end{equation}
we have
\begin{equation}
\kappa_{\mathrm D}(G_n,\mathcal A_n)
\le
\exp\{-k_n I_{\mathrm d,n}(r)\}
\end{equation}
with probability tending to one as \(n\to\infty\).
\end{corollaryR}

\begin{proof}
Let
\begin{equation}
\resizebox{\columnwidth}{!}{$
\mathcal E_n
:=
\left\{
\mathbb E_{U,V}
\left[
\beta_{G_n}(U,V)^{k_n}
\,\middle|\,
G_n
\right]
\le
\exp\{-k_n I_{\mathrm d,n}(r)-\omega_{\mathrm d,n}\}
\right\}.
$}
\end{equation}
By Proposition~\ref{prop:equal-distance-moment-two-source},
\begin{equation}
\mathbb P(\mathcal E_n)\to 1.
\end{equation}

Condition on a graph \(G_n\) for which \(\mathcal E_n\) holds. By
Lemma~\ref{lem:bisector-representation}, applied conditionally on \(G_n\),
\begin{equation}
\begin{aligned}
\mathbb E_{\mathcal A_n}
\left[
\kappa_{\mathrm D}(G_n,\mathcal A_n)
\,\middle|\,
G_n
\right]
&\le
\mathbb E_{U,V}
\left[
\beta_{G_n}(U,V)^{k_n}
\,\middle|\,
G_n
\right]
\\
&\le
\exp\{-k_n I_{\mathrm d,n}(r)-\omega_{\mathrm d,n}\}.
\end{aligned}
\end{equation}

Since
\begin{equation}
\kappa_{\mathrm D}(G_n,\mathcal A_n)\ge 0,
\end{equation}
Markov's inequality gives
\begin{align}
&\mathbb P_{\mathcal A_n}
\left(
\kappa_{\mathrm D}(G_n,\mathcal A_n)
>
\exp\{-k_n I_{\mathrm d,n}(r)\}
\,\middle|\,
G_n
\right)
\\
&\quad\le
\frac{
\mathbb E_{\mathcal A_n}
\left[
\kappa_{\mathrm D}(G_n,\mathcal A_n)
\,\middle|\,
G_n
\right]
}{
\exp\{-k_n I_{\mathrm d,n}(r)\}
}
\\
&\quad\le
e^{-\omega_{\mathrm d,n}}.
\end{align}
Therefore,
\begin{align}
&\mathbb P
\left(
\kappa_{\mathrm D}(G_n,\mathcal A_n)
>
\exp\{-k_n I_{\mathrm d,n}(r)\}
\right)
\\
&\quad\le
\mathbb P(\mathcal E_n^c)+e^{-\omega_{\mathrm d,n}}.
\end{align}
Since
\begin{equation}
\mathbb P(\mathcal E_n^c)\to 0
\end{equation}
and
\begin{equation}
e^{-\omega_{\mathrm d,n}}\to 0,
\end{equation}
we conclude that
\begin{equation}
\mathbb P
\left(
\kappa_{\mathrm D}(G_n,\mathcal A_n)
\le
\exp\{-k_n I_{\mathrm d,n}(r)\}
\right)
\to 1.
\end{equation}
The proof is complete.
\end{proof}
\subsubsection{Gaussian-wave surrogate model}
\label{app:gw-surrogate-model}

The Gaussian-wave surrogate is used as a reference model for
\begin{equation}
X_m(v)_j
=
n\phi_{j+1}(v)^2
=
\bigl(\sqrt n\,\phi_{j+1}(v)\bigr)^2 .
\end{equation}

\begin{definitionR}[Admissible bounded-correlation Gaussian-wave surrogate]
\label{def:gw-surrogate-app}
Fix a finite graph \(G=(V,E)\), an anchor set \(\mathcal A\), a spectral
dimension \(m\), and a constant \(\rho_\star\in[0,1)\). A random field
\begin{equation}
\{\widetilde Y_j(v):v\in V,\ 1\le j\le m\}
\end{equation}
is an admissible bounded-correlation Gaussian-wave surrogate if,
conditionally on \(G\) and \(\mathcal A\), the following conditions hold:
\begin{enumerate}
\item for each \(j\), \(\{\widetilde Y_j(v):v\in V\}\) is a centered
Gaussian field;
\item for every \(v\in V\) and every \(j\),
\begin{equation}
\mathbb E[\widetilde Y_j(v)^2]=1;
\end{equation}
\item for every \(u\neq v\) and every \(j\),
\begin{equation}
\left|
\operatorname{Corr}
\bigl(
\widetilde Y_j(u),
\widetilde Y_j(v)
\bigr)
\right|
\le
\rho_\star;
\end{equation}
\item the fields are independent across \(j=1,\ldots,m\).
\end{enumerate}
The associated energy coordinate and quantized spectral code are
\begin{equation}
\widetilde X_m(v)
:=
\bigl(
\widetilde Y_1(v)^2,\ldots,\widetilde Y_m(v)^2
\bigr),
\end{equation}
and
\begin{equation}
\widetilde Z_{m,\eta}(v)
:=
Q_\eta(\widetilde X_m(v)).
\end{equation}
\end{definitionR}

\subsubsection{Expanded-cell majorant}
\label{app:expanded-cell-majorant}

For \(x,y\in\mathbb R_{\ge0}^{m}\), define
\begin{equation}
H_{\eta}(x,y)
:=
\mathbf 1\{Q_{\eta}(x)=Q_{\eta}(y)\}.
\end{equation}
For a smoothing scale \(\alpha\ge0\), define
\begin{equation}
\begin{aligned}
H^+_{\eta,\alpha}(x,y)
:=
\mathbf 1
\Bigl\{
&\exists q\in\mathbb Z_{\ge0}^{m}:                     \\
&x_i,y_i\in
[q_i\eta-\alpha,(q_i+1)\eta+\alpha],
\ \forall i
\Bigr\}.
\end{aligned}
\end{equation}

Then
\begin{equation}
H_{\eta}(x,y)
\le
H^+_{\eta,\alpha}(x,y)
\end{equation}
for all \(x,y\in\mathbb R_{\ge0}^{m}\).

\subsubsection{Gaussian-Wave Spectral Collision Estimate}
\label{app:gw-spectral-collision-estimate}

We first prove a one-dimensional estimate for squared correlated Gaussian
variables. This small-ball estimate is the main ingredient in the
Gaussian-wave spectral anti-concentration proof below.

\begin{lemma}[Quantized collision of correlated squared Gaussians]
\label{lem:correlated-gaussian-square-collision}
Fix \(\rho_\star\in[0,1)\). There exists a constant
\(C_{\rho_\star}>0\) such that the following holds. Let \((G,H)\) be a
centered bivariate Gaussian vector satisfying
\begin{equation}
\mathbb E[G^2]=\mathbb E[H^2]=1
\end{equation}
and
\begin{equation}
|\operatorname{Corr}(G,H)|\le \rho_\star .
\end{equation}
Then, for every \(\eta\in(0,e^{-1})\),
\begin{equation}
\mathbb P\left(
\left\lfloor \frac{G^2}{\eta}\right\rfloor
=
\left\lfloor \frac{H^2}{\eta}\right\rfloor
\right)
\le
C_{\rho_\star}\eta\log(e/\eta).
\end{equation}
\end{lemma}

\begin{proof}
Let
\begin{equation}
\rho:=\operatorname{Corr}(G,H).
\end{equation}
The joint density of \((G,H)\) is
\begin{equation}
p_\rho(g,h)
=
\frac{1}{2\pi\sqrt{1-\rho^2}}
\exp\left\{
-\frac{g^2-2\rho gh+h^2}{2(1-\rho^2)}
\right\}.
\end{equation}
Let
\begin{equation}
X:=G^2,
\quad
Y:=H^2.
\end{equation}
The joint law of \((X,Y)\) has a density \(f_\rho\) on
\((0,\infty)^2\). By summing over the four sign choices of \(G\) and \(H\),
and using \(|\rho|\le \rho_\star<1\), there exist constants
\(C_{\rho_\star},c_{\rho_\star}>0\), depending only on \(\rho_\star\), such
that
\begin{equation}
f_\rho(x,y)
\le
C_{\rho_\star}
(xy)^{-1/2}
\exp\{-c_{\rho_\star}(x+y)\},
\quad x,y>0.
\end{equation}

For \(\ell=0,1,2,\ldots\), define
\begin{equation}
I_\ell:=[\ell\eta,(\ell+1)\eta).
\end{equation}
Then
\begin{align}
\mathbb P\left(
\left\lfloor \frac{G^2}{\eta}\right\rfloor
=
\left\lfloor \frac{H^2}{\eta}\right\rfloor
\right)
&=
\sum_{\ell=0}^{\infty}
\int_{I_\ell}\int_{I_\ell}
f_\rho(x,y)\,dx\,dy .
\end{align}

For \(\ell=0\), using the density bound without the exponential factor,
\begin{align}
\int_0^\eta\int_0^\eta f_\rho(x,y)\,dx\,dy
&\le
C_{\rho_\star}
\left(\int_0^\eta x^{-1/2}\,dx\right)^2        \\
&\le
C_{\rho_\star}\eta .
\end{align}

For \(1\le \ell\le \lfloor 1/\eta\rfloor\), we have
\(x,y\ge \ell\eta\) on \(I_\ell\times I_\ell\). Hence
\begin{align}
\int_{I_\ell}\int_{I_\ell} f_\rho(x,y)\,dx\,dy
&\le
C_{\rho_\star}\eta^2(\ell\eta)^{-1}        \\
&=
C_{\rho_\star}\frac{\eta}{\ell}.
\end{align}
Therefore,
\begin{equation}
\sum_{\ell=1}^{\lfloor 1/\eta\rfloor}
\int_{I_\ell}\int_{I_\ell} f_\rho(x,y)\,dx\,dy
\le
C_{\rho_\star}\eta\log(e/\eta).
\end{equation}

Finally, for \(\ell>1/\eta\), the exponential decay gives
\begin{equation}
\int_{I_\ell}\int_{I_\ell} f_\rho(x,y)\,dx\,dy
\le
C_{\rho_\star}\eta^2 e^{-2c_{\rho_\star}\ell\eta}.
\end{equation}
Thus
\begin{equation}
\sum_{\ell>1/\eta}
\int_{I_\ell}\int_{I_\ell} f_\rho(x,y)\,dx\,dy
\le
C_{\rho_\star}\eta.
\end{equation}
Combining the three ranges yields
\begin{equation}
\mathbb P\left(
\left\lfloor \frac{G^2}{\eta}\right\rfloor
=
\left\lfloor \frac{H^2}{\eta}\right\rfloor
\right)
\le
C_{\rho_\star}\eta\log(e/\eta).
\end{equation}
The proof is complete.
\end{proof}

\subsubsection{Gaussian-wave spectral anti-concentration}
\label{app:proof-rw-spectral-anti-concentration}

\begin{theoremR}[Gaussian-wave spectral anti-concentration]
\label{thm:rw-spectral-anti-concentration}

Fix \(\rho_\star\in[0,1)\). There exists a constant
\begin{equation}
C_{\mathrm s}=C_{\mathrm s}(\rho_\star)>0
\end{equation}
such that the following holds. Let \(G=(V,E)\) be a finite connected graph,
let \(\mathcal A\subseteq V\) be any anchor set, and let
\begin{equation}
\{\widetilde Y_j(v):v\in V,\ 1\le j\le m\}
\end{equation}
be an admissible bounded-correlation Gaussian-wave surrogate with parameter
\(\rho_\star\). Then, for every \(\eta\in(0,e^{-1})\),
\begin{equation}
\mathbb E
\left[
\widetilde\kappa_{\mathrm{S}\mid\mathrm{D}}
\,\middle|\,
G,\mathcal A
\right]
\le
\bigl(C_{\mathrm s}\eta\log(e/\eta)\bigr)^m .
\end{equation}
Consequently, for any sequence \((G_n,\mathcal A_n,m_n,\eta_n)\) and any
sequence \(\omega_{\mathrm s,n}\to\infty\), we have
\begin{equation}
\widetilde\kappa_{\mathrm{S}\mid\mathrm{D}}
\le
\exp\left\{
-m_n I_{\mathrm s}^{\mathrm{rw}}(\eta_n)
+
\omega_{\mathrm s,n}
\right\}
\end{equation}
with probability at least \(1-e^{-\omega_{\mathrm s,n}}\), conditionally on
\(G_n,\mathcal A_n\), where
\begin{equation}
I_{\mathrm s}^{\mathrm{gw}}(\eta)
:=
\log\frac{1}{C_{\mathrm s}\eta\log(e/\eta)}.
\end{equation}
\end{theoremR}

\begin{proof}
Fix \(G\) and \(\mathcal A\), and condition on the induced distance buckets
\begin{equation}
\{B_t\}_{t\in\mathcal T_{G,\mathcal A}}.
\end{equation}
Let
\begin{equation}
W
:=
\sum_{t\in\mathcal T_{G,\mathcal A}} |B_t|(|B_t|-1).
\end{equation}
If \(W=0\), then there are no distance-collision pairs. By convention,
\begin{equation}
\widetilde\kappa_{\mathrm{S}\mid\mathrm{D}}=0,
\end{equation}
and the claim is immediate.

Assume \(W>0\). Fix an ordered pair \(u\ne v\). For each spectral coordinate
\(j\), admissibility of the Gaussian-wave surrogate gives
\begin{equation}
\mathbb E[\widetilde Y_j(u)^2]
=
\mathbb E[\widetilde Y_j(v)^2]
=
1
\end{equation}
and
\begin{equation}
\left|
\operatorname{Corr}
\bigl(
\widetilde Y_j(u),
\widetilde Y_j(v)
\bigr)
\right|
\le
\rho_\star.
\end{equation}
Therefore, by Lemma~\ref{lem:correlated-gaussian-square-collision},
\begin{equation}
\mathbb P\left(
\left\lfloor
\frac{\widetilde Y_j(u)^2}{\eta}
\right\rfloor
=
\left\lfloor
\frac{\widetilde Y_j(v)^2}{\eta}
\right\rfloor
\right)
\le
C_{\rho_\star}\eta\log(e/\eta).
\end{equation}

The Gaussian fields are independent across \(j\). Hence
\begin{align}
\mathbb P\left(
\widetilde Z_{m,\eta}(u)
=
\widetilde Z_{m,\eta}(v)
\right)
&\le
\bigl(C_{\rho_\star}\eta\log(e/\eta)\bigr)^m .
\end{align}
Set
\begin{equation}
C_{\mathrm s}:=C_{\rho_\star}.
\end{equation}
Using the definition of the weighted within-bucket spectral collision rate,
we obtain
\begin{align}
\mathbb E
\left[
\widetilde\kappa_{\mathrm{S}\mid\mathrm{D}}
\,\middle|\,
G,\mathcal A
\right]
&=
\frac{1}{W}
\sum_{t\in\mathcal T_{G,\mathcal A}}
\sum_{\substack{u,v\in B_t\\ u\ne v}}
\mathbb P\left(
\widetilde Z_{m,\eta}(u)
=
\widetilde Z_{m,\eta}(v)
\right)                                      \\
&\le
\frac{1}{W}
\sum_{t\in\mathcal T_{G,\mathcal A}}
|B_t|(|B_t|-1)
\bigl(C_{\mathrm s}\eta\log(e/\eta)\bigr)^m  \\
&=
\bigl(C_{\mathrm s}\eta\log(e/\eta)\bigr)^m.
\end{align}

For the high-probability statement, Markov's inequality gives
\begin{align}
&
\mathbb P\left(
\widetilde\kappa_{\mathrm{S}\mid\mathrm{D}}
>
e^{\omega_{\mathrm s,n}}
\bigl(C_{\mathrm s}\eta_n\log(e/\eta_n)\bigr)^{m_n}
\,\middle|\,
G_n,\mathcal A_n
\right)                                      \\
&\hspace{5em}\le
e^{-\omega_{\mathrm s,n}}.
\end{align}
By the definition of \(I_{\mathrm s}^{\mathrm{rw}}\), this is equivalent to
\begin{equation}
\widetilde\kappa_{\mathrm{S}\mid\mathrm{D}}
\le
\exp\left\{
-m_n I_{\mathrm s}^{\mathrm{rw}}(\eta_n)
+
\omega_{\mathrm s,n}
\right\}
\end{equation}
with conditional probability at least \(1-e^{-\omega_{\mathrm s,n}}\).
The proof is complete.
\end{proof}

\subsubsection{Proof of Theorem~\ref{thm:closed-gw-hybrid-achievability}}
\label{app:proof-closed-gw-hybrid-achievability}

Let
\(\widetilde F_{G_n,\mathcal A_n}^{(m_n,\eta_n)}\)
denote the hybrid observation map whose spectral component is
an admissible bounded-correlation Gaussian-wave energy
surrogate with correlation bound \(\rho_\star\in[0,1)\).
For constants \(c_r>0\) and
\(C_{\mathrm s}=C_{\mathrm s}(\rho_\star)>0\), write
\begin{equation}
I_{\mathrm d,n}(r)
:=
\frac{c_r}{4\log n},
\quad
I_{\mathrm s}^{\mathrm{gw}}(\eta)
:=
\log
\frac{1}{
C_{\mathrm s}\eta\log(e/\eta)
}.
\label{eq:app-gw-information-scales}
\end{equation}

\begin{theoremR}[Restatement of Theorem~\ref{thm:closed-gw-hybrid-achievability}]
Fix \(r\geq 3\), and let \(G_n\sim\mathcal G_{n,r}\). Let
\(\mathcal A_n\subseteq V(G_n)\) be sampled uniformly without
replacement, independently of \(G_n\), with
\(|\mathcal A_n|=k_n\). Assume
\begin{equation}
\frac{k_n}{\log n}\to\infty,
\quad
k_n=o\bigl((\log n)^2\bigr).
\label{eq:app-closed-gw-anchor-regime}
\end{equation}
There exist constants \(c_r>0\) and
\(C_{\mathrm s}=C_{\mathrm s}(\rho_\star)>0\) such that, if
\begin{equation}
\eta_n\in(0,e^{-1}),
\quad
C_{\mathrm s}\eta_n\log(e/\eta_n)<1
\label{eq:app-closed-gw-positive-info}
\end{equation}
for all sufficiently large \(n\), and if, for some
\(\varepsilon>0\),
\begin{equation}
k_n I_{\mathrm d,n}(r)
+
m_n I_{\mathrm s}^{\mathrm{gw}}(\eta_n)
\geq
(1+\varepsilon)\log n
\label{eq:app-closed-gw-condition}
\end{equation}
for all sufficiently large \(n\), then
\begin{equation}
\operatorname{Err}^*
\bigl(
\widetilde F_{G_n,\mathcal A_n}^{(m_n,\eta_n)}
\bigr)
\xrightarrow{\mathbb P}0.
\label{eq:app-closed-gw-achievability}
\end{equation}
\end{theoremR}

\begin{proof}
By Corollary~\ref{cor:distance-collision-decay-two-source},
with probability tending to one,
\begin{equation}
\kappa_{\mathrm D}(G_n,\mathcal A_n)
\le
\exp\{-k_n I_{\mathrm d,n}(r)\}.
\end{equation}
Let
\begin{equation}
E_{\mathrm D,n}
:=
\left\{
\kappa_{\mathrm D}(G_n,\mathcal A_n)
\le
\exp\{-k_n I_{\mathrm d,n}(r)\}
\right\}.
\end{equation}
Then
\begin{equation}
\mathbb P(E_{\mathrm D,n})\to 1.
\end{equation}

Choose any deterministic sequence
\(\omega_{\mathrm s,n}\to\infty\) such that
\begin{equation}
\omega_{\mathrm s,n}=o(\log n).
\end{equation}
By Theorem~\ref{thm:rw-spectral-anti-concentration},
conditionally on \(G_n,\mathcal A_n\),
\begin{equation}
\begin{aligned}
&\mathbb P
\Bigl(
\widetilde\kappa_{\mathrm S\mid\mathrm D}
\le
\exp
\bigl\{
-m_n I_{\mathrm s}^{\mathrm{gw}}(\eta_n)
+
\omega_{\mathrm s,n}
\bigr\}
\,\Bigm|\,
G_n,\mathcal A_n
\Bigr) \\
&\quad\ge
1-e^{-\omega_{\mathrm s,n}}.
\end{aligned}
\end{equation}
Therefore the spectral event
\begin{equation}
E_{\mathrm S,n}
:=
\left\{
\widetilde\kappa_{\mathrm S\mid\mathrm D}
\le
\exp
\bigl\{
-m_n I_{\mathrm s}^{\mathrm{gw}}(\eta_n)
+
\omega_{\mathrm s,n}
\bigr\}
\right\}
\end{equation}
satisfies
\begin{equation}
\mathbb P(E_{\mathrm S,n})
\ge
1-e^{-\omega_{\mathrm s,n}}
\to 1.
\end{equation}
Hence
\begin{equation}
\mathbb P(E_{\mathrm D,n}\cap E_{\mathrm S,n})\to 1.
\end{equation}

On \(E_{\mathrm D,n}\cap E_{\mathrm S,n}\), the exact hybrid
collision factorization gives
\begin{equation}
\widetilde\kappa_{\mathrm H}
=
\kappa_{\mathrm D}\,
\widetilde\kappa_{\mathrm S\mid\mathrm D}.
\end{equation}
Thus
\begin{equation}
\widetilde\kappa_{\mathrm H}
\le
\exp
\bigl\{
-k_n I_{\mathrm d,n}(r)
-m_n I_{\mathrm s}^{\mathrm{gw}}(\eta_n)
+
\omega_{\mathrm s,n}
\bigr\}.
\end{equation}
By the budget condition,
\begin{equation}
k_n I_{\mathrm d,n}(r)
+
m_n I_{\mathrm s}^{\mathrm{gw}}(\eta_n)
\ge
(1+\varepsilon)\log n.
\end{equation}
Since \(\omega_{\mathrm s,n}=o(\log n)\), for all sufficiently
large \(n\),
\begin{equation}
\widetilde\kappa_{\mathrm H}
\le
n^{-1-\varepsilon/2}.
\end{equation}
Consequently,
\begin{equation}
(n-1)\widetilde\kappa_{\mathrm H}
\le
n\widetilde\kappa_{\mathrm H}
\le
n^{-\varepsilon/2}
\to 0
\end{equation}
on an event whose probability tends to one.

Applying the deterministic collision-achievability criterion to
the Gaussian-wave hybrid observation map gives
\begin{equation}
\operatorname{Err}^*
\bigl(
\widetilde F_{G_n,\mathcal A_n}^{(m_n,\eta_n)}
\bigr)
\le
(n-1)\widetilde\kappa_{\mathrm H}.
\end{equation}
Therefore,
\begin{equation}
\operatorname{Err}^*
\bigl(
\widetilde F_{G_n,\mathcal A_n}^{(m_n,\eta_n)}
\bigr)
\xrightarrow{\mathbb P}0.
\end{equation}
The proof is complete.
\end{proof}

\subsubsection{Laplacian achievability under an actual spectral collision bound}
\label{app:proof-laplacian-spectral-collision-achievability}

\begin{assumptionR}[Restatement of Assumption~\ref{ass:actual-spectral-collision-bound}]
\label{assR:actual-spectral-collision-bound}
Fix \(r\geq 3\), and let \(G_n\sim\mathcal{G}_{n,r}\).
Let \(\mathcal A_n\subseteq V(G_n)\) be sampled uniformly
without replacement and independently of \(G_n\). Let
\begin{equation}
X_n(v)
:=
n\bigl(
\phi_2(v)^2,\ldots,\phi_{m_n+1}(v)^2
\bigr)
\in
\mathbb{R}_{\geq 0}^{m_n}.
\label{eq:app-actual-laplacian-energy}
\end{equation}
On the event
\begin{equation}
\sum_{t\in\mathcal{T}_{G_n,\mathcal{A}_n}}
|B_t|(|B_t|-1)>0,
\label{eq:app-positive-distance-collision-mass}
\end{equation}
define
\begin{equation}
\begin{aligned}
\pi_{\mathrm D,n}(u,v)
:=
\frac{
\mathbf{1}
\bigl\{
u\neq v,\,
D_{\mathcal A_n}(u)=D_{\mathcal A_n}(v)
\bigr\}
}{
\sum_{t\in\mathcal{T}_{G_n,\mathcal{A}_n}}
|B_t|(|B_t|-1)
}.
\end{aligned}
\label{eq:app-distance-collision-pair-measure}
\end{equation}
If the denominator is zero, all
\(\pi_{\mathrm D,n}\)-expectations below are set to zero.

Let
\begin{equation}
H_{\eta_n}(x,y)
:=
\mathbf{1}
\bigl\{
Q_{\eta_n}(x)=Q_{\eta_n}(y)
\bigr\}.
\label{eq:app-hard-spectral-collision}
\end{equation}
Let \(H^+_{\eta_n,\alpha_n}\) be the expanded-cell majorant
from Appendix~\ref{app:expanded-cell-majorant}, satisfying
\begin{equation}
0
\leq
H^+_{\eta_n,\alpha_n}(x,y)
\leq
1,
\quad
H_{\eta_n}(x,y)
\leq
H^+_{\eta_n,\alpha_n}(x,y).
\label{eq:app-expanded-majorant-dominates}
\end{equation}
Set
\begin{equation}
\alpha_n
:=
\eta_n(\log n)^{-2}.
\label{eq:app-smoothing-scale}
\end{equation}
Define
\begin{equation}
\Gamma_n^+
:=
\mathbb{E}_{\pi_{\mathrm D,n}}
\Bigl[
H^+_{\eta_n,\alpha_n}
\bigl(
X_n(U),X_n(V)
\bigr)
\Bigr].
\label{eq:app-smoothed-spectral-rate}
\end{equation}

Assume
\begin{equation}
\begin{aligned}
m_n
&=
O(\log n),
&
\eta_n
&\in
[n^{-c_0},e^{-1}),
&
\alpha_n
&=
\eta_n(\log n)^{-2},
\end{aligned}
\label{eq:app-actual-spectral-bound-regime}
\end{equation}
where \(c_0>0\) is fixed. Suppose that there exist a constant
\(C_{\mathrm{sc},r}>0\), depending only on \(r\), and a
deterministic sequence \(\xi_n=o(\log n)\), such that, with
probability tending to one,
\begin{equation}
\Gamma_n^+
\leq
\exp
\Bigl\{
-m_n I_{\mathrm{sc},r}(\eta_n)
+
\xi_n
\Bigr\},
\label{eq:app-actual-spectral-collision-bound}
\end{equation}
where
\begin{equation}
I_{\mathrm{sc},r}(\eta)
:=
\log
\frac{1}{
C_{\mathrm{sc},r}\eta\log(e/\eta)
}.
\label{eq:app-spectral-collision-information-scale}
\end{equation}
\end{assumptionR}

\begin{lemmaR}[Actual Laplacian spectral collision decay]
\label{lemR:actual-laplacian-spectral-collision-decay}
Under Assumption~\ref{assR:actual-spectral-collision-bound},
with probability tending to one,
\begin{equation}
\kappa_{\mathrm S\mid\mathrm D}^{\mathrm{Lap}}
\leq
\exp
\Bigl\{
-m_n I_{\mathrm{sc},r}(\eta_n)
+
\xi_n
\Bigr\}.
\label{eq:app-lap-spectral-collision-decay}
\end{equation}
\end{lemmaR}

\begin{proof}
If
\begin{equation}
\sum_{t\in\mathcal{T}_{G_n,\mathcal{A}_n}}
|B_t|(|B_t|-1)=0,
\label{eq:app-no-distance-collision-pairs}
\end{equation}
then \(\kappa_{\mathrm D}=0\), and the hybrid collision
probability is zero. The claim is immediate under the stated
convention.

We therefore work on the event where distance-collision
pairs exist. For the actual Laplacian-energy coordinates,
\begin{equation}
\begin{aligned}
\kappa_{\mathrm S\mid\mathrm D}^{\mathrm{Lap}}
&=
\mathbb{E}_{\pi_{\mathrm D,n}}
\Bigl[
H_{\eta_n}
\bigl(
X_n(U),X_n(V)
\bigr)
\Bigr].
\end{aligned}
\label{eq:app-lap-hard-collision-rate}
\end{equation}
Since
\(H_{\eta_n}\leq H^+_{\eta_n,\alpha_n}\),
we have
\begin{equation}
\begin{aligned}
\kappa_{\mathrm S\mid\mathrm D}^{\mathrm{Lap}}
&\leq
\mathbb{E}_{\pi_{\mathrm D,n}}
\Bigl[
H^+_{\eta_n,\alpha_n}
\bigl(
X_n(U),X_n(V)
\bigr)
\Bigr] \\
&=
\Gamma_n^+.
\end{aligned}
\label{eq:app-lap-smoothed-control}
\end{equation}
The result follows from
Assumption~\ref{assR:actual-spectral-collision-bound}.
\end{proof}

\begin{corollaryR}[Restatement of Corollary~\ref{cor:lap-hybrid-achievability}]
Assume the hypotheses of
Corollary~\ref{cor:distance-collision-decay-two-source}. Let
\begin{equation}
F_n^{\mathrm{Lap}}(v)
:=
\bigl(
D_{\mathcal A_n}(v),
Q_{\eta_n}(X_n(v))
\bigr).
\label{eq:app-lap-hybrid-map}
\end{equation}
Suppose Assumption~\ref{assR:actual-spectral-collision-bound}
holds and
\begin{equation}
C_{\mathrm{sc},r}\eta_n\log(e/\eta_n)<1
\label{eq:app-lap-positive-spectral-information}
\end{equation}
for all sufficiently large \(n\). If there exists
\(\varepsilon>0\) such that
\begin{equation}
\begin{aligned}
k_n I_{\mathrm d,n}(r)
+
m_n I_{\mathrm{sc},r}(\eta_n)
\geq
(1+\varepsilon)\log n
\end{aligned}
\label{eq:app-lap-information-budget}
\end{equation}
for all sufficiently large \(n\), then
\begin{equation}
\operatorname{Err}^*
\bigl(
F_n^{\mathrm{Lap}}
\bigr)
\xrightarrow{\mathbb P}0.
\label{eq:app-lap-achievability-conclusion}
\end{equation}
\end{corollaryR}

\begin{proof}
By Corollary~\ref{cor:distance-collision-decay-two-source},
with probability tending to one,
\begin{equation}
\kappa_{\mathrm D}
\leq
\exp
\bigl\{
-k_n I_{\mathrm d,n}(r)
\bigr\}.
\label{eq:app-lap-proof-distance-decay}
\end{equation}
By Lemma~\ref{lemR:actual-laplacian-spectral-collision-decay},
with probability tending to one,
\begin{equation}
\kappa_{\mathrm S\mid\mathrm D}^{\mathrm{Lap}}
\leq
\exp
\Bigl\{
-m_n I_{\mathrm{sc},r}(\eta_n)
+
\xi_n
\Bigr\}.
\label{eq:app-lap-proof-spectral-decay}
\end{equation}
The exact hybrid collision factorization gives
\begin{equation}
\kappa_{\mathrm H}^{\mathrm{Lap}}
=
\kappa_{\mathrm D}
\kappa_{\mathrm S\mid\mathrm D}^{\mathrm{Lap}}.
\label{eq:app-lap-proof-factorization}
\end{equation}
Therefore, with probability tending to one,
\begin{equation}
\begin{aligned}
\kappa_{\mathrm H}^{\mathrm{Lap}}
&\leq
\exp
\Bigl\{
-k_n I_{\mathrm d,n}(r)
-m_n I_{\mathrm{sc},r}(\eta_n)
+
\xi_n
\Bigr\}.
\end{aligned}
\label{eq:app-lap-proof-hybrid-collision}
\end{equation}
Using
\eqref{eq:app-lap-information-budget} and
\(\xi_n=o(\log n)\), we obtain
\begin{equation}
\begin{aligned}
(n-1)
\kappa_{\mathrm H}^{\mathrm{Lap}}
&\leq
\exp
\Bigl\{
\log n
-(1+\varepsilon)\log n
+
o(\log n)
\Bigr\} \\
&=
\exp
\Bigl\{
-\varepsilon\log n
+
o(\log n)
\Bigr\}.
\end{aligned}
\label{eq:app-lap-proof-vanishing-product}
\end{equation}
Thus
\begin{equation}
(n-1)
\kappa_{\mathrm H}^{\mathrm{Lap}}
\xrightarrow{\mathbb P}0.
\label{eq:app-lap-proof-product-to-zero}
\end{equation}
The deterministic collision-achievability criterion then gives
\begin{equation}
\operatorname{Err}^*
\bigl(
F_n^{\mathrm{Lap}}
\bigr)
\xrightarrow{\mathbb P}0.
\label{eq:app-lap-proof-final}
\end{equation}
The proof is complete.
\end{proof}

\subsection{Proofs for Two-Sided Localization Bounds}
\label{app:two-sided}

This appendix proves the graph-dependent design principle stated in
Section~\ref{subsec:two-sided-localization-criterion}. The impossibility
side follows from the simplex-refined normalized image-size converse,
while the achievability side follows from the hybrid collision
factorization.

\subsubsection{Proof of Corollary~\ref{cor:graph-dependent-design-principle}}
\label{app:proof-graph-dependent-design-principle}

\begin{proof}
We first prove the impossibility regime. By
Theorem~\ref{thm:normalized-refined-converse},
\begin{equation}
\begin{aligned}
&\operatorname{Err}^*
\bigl(F_{G_n,\mathcal A_n}^{(m_n,\eta_n)}\bigr)\\
\ge{}&
1
-
\frac{D(G_n,\mathcal A_n)}{n}
\binom{\lfloor L_n/\eta_n\rfloor+m_n}{m_n}
-
\frac{m_n}{L_n}.
\end{aligned}
\end{equation}
Using the definition
\begin{equation}
\begin{aligned}
&\mathcal B_{\mathrm{conv}}^{\Delta}
(G_n,\mathcal A_n,m_n,\eta_n,L_n)\\
:={}&
\log D(G_n,\mathcal A_n)
+
\log\binom{\lfloor L_n/\eta_n\rfloor+m_n}{m_n},
\end{aligned}
\end{equation}
the preceding bound becomes
\begin{equation}
\begin{aligned}
&\operatorname{Err}^*
\bigl(F_{G_n,\mathcal A_n}^{(m_n,\eta_n)}\bigr)\\
\ge{}&
1
-
\exp\!\left\{
\mathcal B_{\mathrm{conv}}^{\Delta}
(G_n,\mathcal A_n,m_n,\eta_n,L_n)
-
\log n
\right\}
-
\frac{m_n}{L_n}.
\end{aligned}
\end{equation}
If
\begin{equation}
\mathcal B_{\mathrm{conv}}^{\Delta}
(G_n,\mathcal A_n,m_n,\eta_n,L_n)
\le
(1-\varepsilon)\log n
\end{equation}
and \(m_n/L_n\to0\), then
\begin{equation}
\begin{aligned}
&\exp\!\left\{
\mathcal B_{\mathrm{conv}}^{\Delta}
(G_n,\mathcal A_n,m_n,\eta_n,L_n)
-
\log n
\right\}\\
&\le
n^{-\varepsilon}
\to0.
\end{aligned}
\end{equation}
Therefore,
\begin{equation}
\operatorname{Err}^*
\bigl(F_{G_n,\mathcal A_n}^{(m_n,\eta_n)}\bigr)
\to1.
\end{equation}

We next prove the achievability regime. By
Lemma~\ref{lem:basic-collision-lemma},
\begin{equation}
\operatorname{Err}^*
\bigl(F_{G_n,\mathcal A_n}^{(m_n,\eta_n)}\bigr)
\le
(n-1)
\kappa_{\mathrm D}
\kappa_{\mathrm{S}\mid\mathrm D}.
\end{equation}
Using
\begin{equation}
\mathcal I_{\mathrm H}(G_n,\mathcal A_n,m_n,\eta_n)
=
-\log\kappa_{\mathrm D}
-
\log\kappa_{\mathrm{S}\mid\mathrm D},
\end{equation}
we obtain
\begin{equation}
\operatorname{Err}^*
\bigl(F_{G_n,\mathcal A_n}^{(m_n,\eta_n)}\bigr)
\le
(n-1)
\exp\!\left\{
-\mathcal I_{\mathrm H}(G_n,\mathcal A_n,m_n,\eta_n)
\right\}.
\end{equation}
If
\begin{equation}
\mathcal I_{\mathrm H}(G_n,\mathcal A_n,m_n,\eta_n)
\ge
(1+\varepsilon)\log n,
\end{equation}
then
\begin{equation}
\begin{aligned}
&(n-1)
\exp\!\left\{
-\mathcal I_{\mathrm H}(G_n,\mathcal A_n,m_n,\eta_n)
\right\}\\
&\le
(n-1)n^{-1-\varepsilon}\\
&\le
n^{-\varepsilon}
\to0.
\end{aligned}
\end{equation}
Thus,
\begin{equation}
\operatorname{Err}^*
\bigl(F_{G_n,\mathcal A_n}^{(m_n,\eta_n)}\bigr)
\to0.
\end{equation}
The proof is complete.
\end{proof}

\section{Detailed Experimental Protocol}
\label{app:exp-details}

\subsection{Synthetic diagnostics}

The synthetic localization experiment uses five graph families: random
\(r\)-regular graphs with \(r\in\{3,6,10\}\), Erd\H{o}s-R\'enyi graphs
with average degree \(6\), stochastic block models with \(4\) blocks and
\((c_{\mathrm{in}},c_{\mathrm{out}})=(14,2)\), two-dimensional grids,
and barbell graphs.  For non-grid graphs,
\(n\in\{500,1000,2000\}\); for grids,
\(n\in\{529,1024,2025\}\).  Each configuration uses \(10\) graph trials.

We evaluate
\begin{equation}
k\in\{1,2,3,4,5,6,8,12,16,24,32\},
\end{equation}
and
\begin{equation}
\begin{aligned}
m&\in\{1,2,3,4,6,8,12,16\},\\
\eta&\in\{0.8,0.6,0.5,0.4,0.3,0.2,0.1\}.
\end{aligned}
\end{equation}
The spectral variants are
\begin{equation}
\texttt{lap\_energy},\quad
\texttt{gw\_energy},\quad
\texttt{lap\_signed},\quad
\texttt{gw\_signed}.
\end{equation}
For eigenvector \(\phi_j\),
\begin{equation}
S^{\mathrm{energy}}_j(v)=n\phi_j(v)^2,
\quad
S^{\mathrm{signed}}_j(v)=\sqrt n\,\phi_j(v).
\end{equation}

The collision-transfer experiment uses the same graph families with
\begin{equation}
\begin{aligned}
n&\in\{500,1000,2000\},\\
k&\in\{2,4,8,16,32\},\\
m&\in\{2,4,8,16\},\\
\eta&\in\{0.7,0.5,0.3,0.1\}.
\end{aligned}
\end{equation}
Each setting uses \(5\) graph trials.  Diagnostics based on the
expanded-cell majorant sample \(20{,}000\) ordered pairs from
distance-collision buckets and use
\(\alpha/\eta\in\{0.05,0.10\}\).

\subsection{Universal Dependencies structural task probes}

Each UD sentence is converted into an unlabeled undirected dependency-tree
skeleton.  Tokens are nodes, each head-dependent relation gives one
undirected edge, and multiword tokens and empty nodes are skipped.  More
specifically, only integer-token CoNLL-U rows are retained; non-integer
IDs corresponding to multiword tokens or empty nodes are ignored, and
edges whose heads are not retained are skipped.  We keep sentences with a
valid dependency root whose resulting skeleton is connected and satisfies
\(6\le n\le80\).  The requested treebanks are
\begin{equation}
\begin{gathered}
\text{UD\_English-EWT},\quad
\text{UD\_Chinese-GSD},\quad
\text{UD\_Spanish-GSD},\\
\text{UD\_French-GSD},\quad
\text{UD\_German-GSD}.
\end{gathered}
\end{equation}
The UD root directory is discovered automatically.  The implementation
first searches for a local \texttt{ud-treebanks-v2.18} directory and then
falls back to nearby directories containing \texttt{UD\_*} treebank
folders.  The resolved path is recorded in the experiment log.  If a
requested treebank is unavailable, a predefined alias is used.  The aliases
are Spanish-AnCora for Spanish-GSD, German-HDT for German-GSD,
French-Sequoia and French-ParisStories for French-GSD, English-GUM for
English-EWT, and Chinese-GSDSimp for Chinese-GSD.  Alias replacements and
unavailable treebanks are written to the log.

Splits are inferred from the CoNLL-U filenames.  Files containing
\texttt{-train.}, \texttt{-dev.}, or \texttt{-test.} are assigned to the
corresponding official split; other files are assigned to an unknown split.
Each treebank uses at most \(3000\) training sentences and all available
development and test sentences after filtering.  If a split exceeds its
limit, sentences are sampled without replacement using a deterministic
treebank-specific seed.  Unknown-split sentences, when present, are capped
at \(20\).  If an official train/dev/test partition is not available after
loading, the probe code falls back to a deterministic random
\(70\%/15\%/\)remaining train/dev/test split.

For each sentence graph, we precompute all-pairs shortest-path distances,
normalized-Laplacian energy coordinates \(n\phi_j(v)^2\), signed
coordinates \(\sqrt n\,\phi_j(v)\), and basic graph statistics.  The
normalized Laplacian is
\(\mathcal L=I-D^{-1/2}AD^{-1/2}\), and the nontrivial eigenvectors
\(j=2,\ldots,m+1\) are used.  For signed coordinates, eigenvector signs
are fixed by making the largest-magnitude entry positive.  Energy
coordinates are sign-invariant.  Lexical forms, dependency labels, edge
directions, POS tags, and token attributes are not used by the supervised
probes.  Dependency labels and POS tags are stored only for
encoding-level ambiguity diagnostics.

\subsection{UD probe protocol}

The UD experiments use only PE codes as input.  The surface-position
probe predicts normalized token position
\begin{equation}
y_{\mathrm{pos}}(v)=\frac{i(v)}{n-1}.
\end{equation}
The dependency-depth probe predicts normalized distance to the dependency
root,
\begin{equation}
y_{\mathrm{depth}}(v)
=
\frac{d_G(v,r)}{\max_{u\in V}d_G(u,r)}.
\end{equation}
The pairwise dependency-distance probe predicts a clipped bucket of
\(d_G(u,v)\) from
\begin{equation}
\bigl[
F(u)+F(v),\,
|F(u)-F(v)|,\,
F(u)\odot F(v)
\bigr].
\end{equation}

The probe input is the code representation produced by the PE module.
Distance features are divided by the sentence-graph diameter.  Quantized
spectral bins are cast to floating-point features.  NoPE is represented by
a constant feature.  For HybridEnergy, the input concatenates normalized
anchor-distance features and quantized Laplacian-energy bins.  Features
are standardized using the training split mean and standard deviation for
each treebank and configuration, and the same transformation is applied to
development and test examples.

The node-level probes use a three-layer MLP with hidden dimension \(128\),
dropout \(0.10\), AdamW optimizer, learning rate \(10^{-3}\), weight
decay \(10^{-4}\), batch size \(1024\), and SmoothL1 loss.  The regression
head uses a sigmoid output for normalized targets.  Training uses early
stopping with patience \(8\) and at most \(50\) epochs.  The model seeds
are \(\{42,43,44\}\).  The pairwise-distance probe uses the same training
budget on the pair representation above.

The original surface-position grid is
\begin{equation}
\begin{aligned}
k&\in\{1,2,4,8\},\\
m&\in\{2,4,8,16\},\\
\eta&=0.25.
\end{aligned}
\end{equation}
For Distance and HybridEnergy, random anchors are sampled without
replacement within each sentence graph.  Three random anchor trials are
used.  The anchor RNG is deterministic and depends on the global seed, the
graph identifier, the encoding, the anchor strategy, and the configuration
indices.  NoPE uses a single constant-code configuration, and
SpectralEnergy uses no anchors.

For the depth and pairwise-distance probes, no new hyperparameter search
is performed.  We freeze the HybridEnergy configurations selected by
development-set NMAE in the surface-position protocol.  The selected
\((k,m)\) values for English, Chinese, Spanish, French, and German are
\begin{equation}
(1,16),\quad (4,4),\quad (8,4),\quad (4,4),\quad (1,16),
\end{equation}
respectively, with \(\eta=0.25\).

We report NMAE, RMSE, ordinal pair accuracy, mean Kendall-\(\tau\), and
4-bin/8-bin accuracy for the surface-position probe; NMAE for the depth
probe; and macro-F1 for the pairwise-distance probe.  Ordinal pair
accuracy is computed within each sentence by comparing the predicted and
true order of token pairs; prediction ties are counted as half-correct.
Mean Kendall-\(\tau\) is averaged over sentence graphs, with undefined
values set to zero.  The 4-bin and 8-bin accuracies use equal-width bins
over the normalized interval.

\subsection{Encoding-level UD diagnostics}

For each encoding, we compute
\begin{equation}
\operatorname{Err}^{*}(F)
=
1-\frac{|\operatorname{Im}(F)|}{n},
\end{equation}
as well as
\(\kappa(F)\),
\(I(F)/\log n\),
\(I_D/\log n\),
\(I_H/\log n\),
\(I_{S|D}/\log n\), and
\(B_{\mathrm{conv}}/\log n\).
The collision rates are computed from code-fiber counts over ordered
vertex pairs sampled without replacement.  The distance-conditioned
quantity is computed as \(\kappa_{S|D}=\kappa_H/\kappa_D\) when
\(\kappa_D>0\).  When \(\kappa_D=0\), we use the theoretical convention
\(\kappa_{S|D}=0\) and \(I_{S|D}=+\infty\), while capped information
ratios use the numerical floor \(\epsilon=10^{-300}\).  Information
ratios are capped at \(5\) for aggregation.

The converse diagnostic uses \(L=4.0\) and the implemented conservative
box-budget version of \(B_{\mathrm{conv}}\).  In addition to localization
diagnostics, we compute code-fiber task ambiguity for stored categorical
labels, including exact surface position, 4-bin and 8-bin surface
position, root-depth buckets, UPOS, and dependency relation labels.  These
diagnostics measure encoding-level localization and label ambiguity and
are reported separately from the supervised UD structural-probe metrics.

\subsection{UD ablations}

We run three design-variable ablation groups, with the quantization group
split into a diagnostics-only block and a smaller probe block:
\begin{itemize}
    \item \textbf{Anchor strategy:}
    random, root, and farthest anchors with
    \(k\in\{1,4,8\}\), \(m\in\{8,16\}\), and \(\eta=0.25\).
    This block uses English-EWT, Chinese-GSD, and German-GSD.
    Random anchors use three trials.  Root anchors use the dependency root
    as the first anchor; for \(k>1\), remaining anchors are sampled
    randomly.  Farthest anchors start from the dependency root and then
    greedily add the node maximizing the minimum distance to the current
    anchor set.  Since farthest \(k=1\) duplicates root \(k=1\), it is
    skipped.
    \item \textbf{Spectral type:}
    energy versus signed coordinates with
    \(k\in\{4,8\}\), \(m\in\{8,16\}\), and \(\eta=0.25\).
    This block compares SpectralEnergy and HybridEnergy against
    SpectralSigned and HybridSigned on English-EWT, Chinese-GSD, and
    German-GSD.
    \item \textbf{Quantization:}
    \(\eta\in\{0.125,0.25,0.5,1.0\}\) with
    \(k\in\{4,8\}\) and \(m\in\{8,16\}\).
    The quantization diagnostics are computed on all five treebanks,
    while the reduced quantization probe is run on English-EWT and
    Chinese-GSD.
\end{itemize}
The reduced ablation probe uses model seed \(42\), at most \(30\) epochs,
and patience \(6\).  All ablation blocks use the same PE construction,
feature standardization, optimizer, and MLP architecture as the main
surface-position probe unless otherwise stated.

\subsection{PE-row derangement control}

The PE-row derangement control computes PE features on the true dependency
tree and then deranges feature rows within each sentence while keeping
targets fixed.  This preserves sentence length, feature dimension, and PE
marginals, but breaks token-level PE-target alignment.  The control
wrapper intercepts PE-computer outputs and applies the same within-sentence
row index to row-indexed feature arrays for a given graph, configuration,
control mode, and shuffle seed.  Square all-pairs matrices are not row
shuffled.  The default control mode is derangement, which enforces no
fixed points when \(n>1\); if random attempts fail, a cyclic shift is used.
The implementation also supports ordinary row permutation and row
resampling with replacement excluding each token's original row, but these
are not the default paper setting.

The default grid is
\begin{equation}
\begin{aligned}
k&\in\{4,8\},\\
m&\in\{8,16\},\\
\eta&=0.25.
\end{aligned}
\end{equation}
It compares NoPE, Distance, SpectralEnergy, and HybridEnergy.  Probe
results are run on English-EWT, Chinese-GSD, and German-GSD; diagnostics
are computed using the loaded five-treebank graph store.  The default
seeds are
\begin{equation}
\begin{aligned}
\text{model seeds}&=\{42,43,44\},\\
\text{derangement seeds}&=\{20260611,20260612,20260613\}.
\end{aligned}
\end{equation}
The null-control runs use the same MLP architecture and optimizer as the
main probe, with a shorter training budget of at most \(20\) epochs and
patience \(5\).  Shuffle audits record fixed-point counts, unique-row
fractions, and duplicate selected rows for reproducibility.  The same
row-derangement definition is used wherever a derangement control is
reported for the surface-position, dependency-depth, and pairwise-distance
probe features.

For more details on datasets and hyperparameter configurations, please refer to
\url{https://anonymous.4open.science/r/Converse-and-Collision-Based-Achievability-0DDC}.

\textbf{Use of AI assistants.}
AI assistants were used to support code development, including debugging
and boilerplate generation. All experimental results, analyses, and
claims were produced and verified by the authors.

\section{Supplementary Empirical Diagnostics}
\label{app:empirical_results}

\subsection{Additional diagnostics for mathematical localization and transfer}
\label{app:math_localization_transfer}
\label{app:exp1_math_localization}
\label{app:exp2_collision_transfer}
\label{app:exp3_design_diagnostics}

This appendix provides supplementary diagnostics for the mathematical
localization, collision-decomposition, Gaussian-wave transfer, and unified
design-map experiments in
Section~\ref{subsec:math_localization_transfer}.  The results verify the
implementation invariants used in the theory, report the full sensitivity and
transfer diagnostics omitted from the main text, and provide additional
visualizations for the graph-dependent design map.

\paragraph{Implementation invariants}
For the localization experiment, the \(566{,}790\) evaluated observation maps
and \(2{,}267{,}160\) conservative box-budget rows exactly match their
expected counts. The maximum numerical discrepancy in
\begin{equation}
\operatorname{Err}^{*}(F)
=
1-\frac{|\operatorname{Im}(F)|}{n}
\end{equation}
is \(2.220\times10^{-16}\), while the collision-to-error inequality has no
observed violation.  The maximum factorization error in
\begin{equation}
\kappa_H
=
\kappa_D\kappa_{S|D}
\end{equation}
is \(2.776\times10^{-17}\). Hybrid encoding is never worse than its distance-only component, and no
violation of either the collision-to-error inequality or the conservative
box-budget converse corollary is observed.

The empirical conditional collision estimator before applying the above
convention is non-estimable for \(19.39\%\) of the evaluated hybrid
configurations. These cases occur only when \(\kappa_{\mathrm D}=0\),
meaning that the distance code has already separated all nodes and no
residual distance-collision measure remains. Following the convention above,
we set
$\kappa_{\mathrm{S}\mid\mathrm D}=0,
\kappa_{\mathrm H}=0,
I_{\mathrm H}=+\infty$
for these distance-saturated configurations. They therefore represent
successful saturation rather than numerical failure. In finite diagnostic
plots, such infinite-information rows are either marked as saturated or
excluded from finite discrepancy summaries.

For the collision-decomposition experiment, the raw experiment contains
\(33{,}600\) evaluated configurations, matching the expected count exactly.
The expanded-cell diagnostic contains \(13{,}752\) rows, below the
theoretical upper bound of \(18{,}900\) because configurations with
\(\kappa_D=0\) have no residual distance-collision measure and are already
separated by the distance code. Table~\ref{tab:app_exp2_sanity} reports the
corresponding implementation checks. All deterministic identities hold up to
machine precision.

\begin{table}[!t]
\centering
\caption{
Sanity checks for the collision-decomposition and transfer diagnostics.
}
\label{tab:app_exp2_sanity}

\footnotesize
\setlength{\tabcolsep}{0pt}
\renewcommand{\arraystretch}{1.08}

\begin{tabular*}{\columnwidth}{@{\extracolsep{\fill}}lrl@{}}
\toprule
Diagnostic & Value & Status \\
\midrule
Raw rows / expected rows
& \(33{,}600 / 33{,}600\) & complete \\
Smoothed rows / upper bound
& \(13{,}752 / 18{,}900\) & valid \\
\(\max |\kappa_H-\kappa_D\kappa_{S|D}|\)
& \(9.918\times 10^{-17}\) & numerical zero \\
\(\max |I_H-I_D-I_{S|D}|\)
& \(3.553\times 10^{-15}\) & numerical zero \\
Collision-to-error violations
& \(0\) & none \\
Box-budget converse violations
& \(0\) & none \\
Rows with \(H_{\eta,\alpha}^{+}<H_\eta\)
& \(0\) & none \\
\bottomrule
\end{tabular*}

\vspace{-0.5em}
\end{table}

\paragraph{Localization sensitivity}
Figure~\ref{fig:app_exp1_k_sensitivity} provides the complete sensitivity
analysis for the number of anchors.  Increasing the number of anchors
monotonically improves the distance code, while larger spectral dimension and
finer quantization generally reduce the residual within-bucket collision
rate.  The qualitative conclusion is also stable across the evaluated truncation
choices used to compute the conservative box-budget diagnostic
\(B_{\rm conv}^{\Box}\). In the low-complexity
configuration search, every non-barbell setting that reaches
\(\operatorname{Err}^{*}(F)\le 0.1\) has \(I_H/\log n>1\); no evaluated
barbell configuration reaches this threshold.

\begin{figure*}[htbp]
    \centering
    \includegraphics[width=0.75\textwidth]{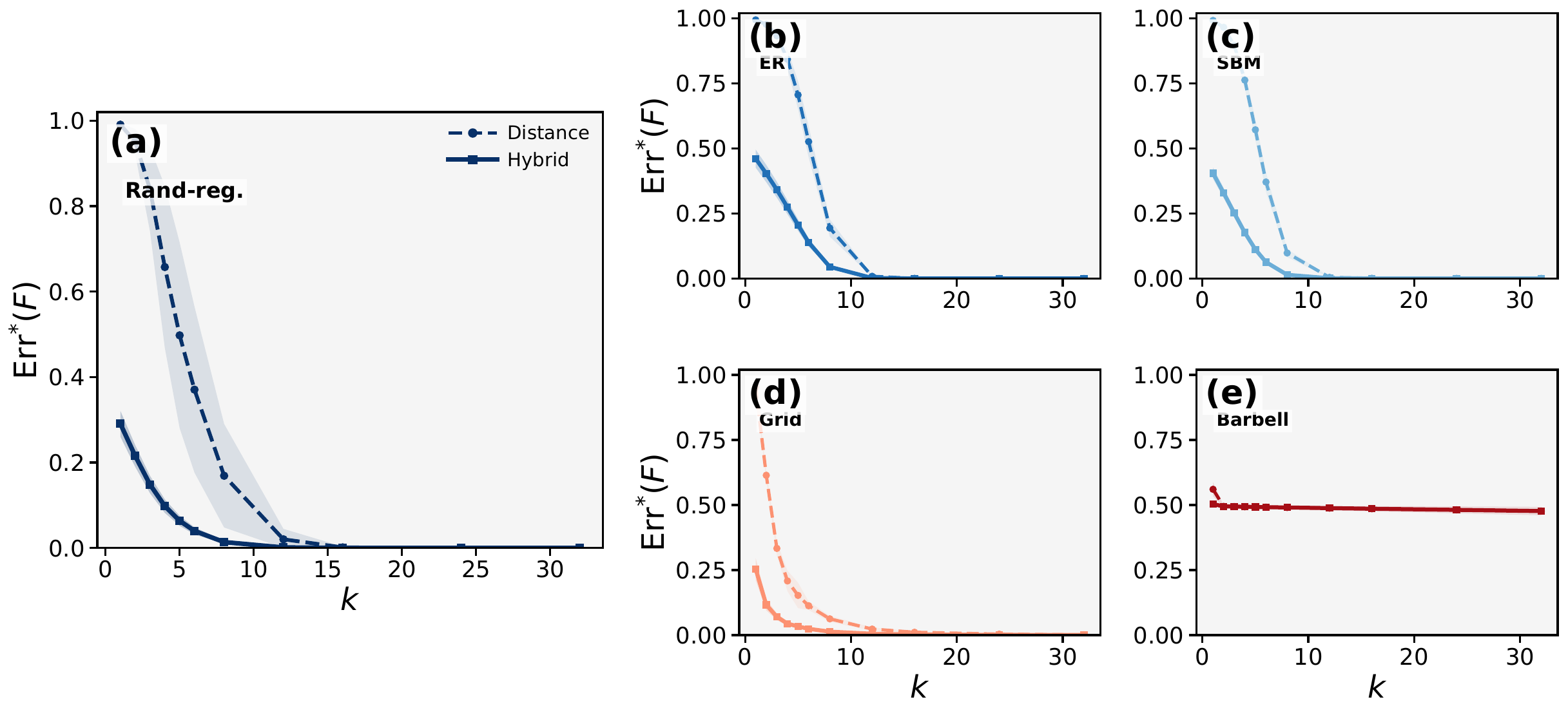}
    \caption{
    Sensitivity to the number of anchors \(k\).
    The random-regular setting is enlarged because it is the graph family
    directly aligned with the main theoretical achievability result.  The four
    smaller panels show ER, SBM, grid, and barbell graphs.  Solid curves denote
    hybrid Laplacian-energy encodings, while dashed curves denote distance-only
    encodings.  Shaded regions are \(95\%\) confidence intervals.
    }
    \label{fig:app_exp1_k_sensitivity}
\end{figure*}

\paragraph{Information regimes and Gaussian-wave transfer}
Table~\ref{tab:app_exp2_regime} gives a regime-binned localization summary.
The four regimes are defined by
\begin{equation}
I_H/\log n
\in
(-\infty,0.75),\quad
[0.75,1),\quad
[1,1.25),\quad
[1.25,\infty).
\end{equation}
The mean localization error decreases monotonically across these regimes,
from \(0.515\) in the low-information regime to \(0.003\) in the
high-information regime.  This supports the use of \(I_H/\log n\) as a
finite-sample diagnostic for the achievability side of the theory.

\begin{table}[htbp]
\centering
\scriptsize
\setlength{\tabcolsep}{4pt}
\caption{
Regime-binned localization summary over all spectral variants.
}
\label{tab:app_exp2_regime}
\begin{tabular}{lcccc}
\toprule
Regime
& Rows
& Mean Err.
& Median Err.
& Mean \(I_H/\log n\) \\
\midrule
Low \((<0.75)\)
& \(3{,}173\)
& \(0.515\)
& \(0.493\)
& \(0.386\) \\
Near-below \([0.75,1)\)
& \(1{,}427\)
& \(0.355\)
& \(0.342\)
& \(0.879\) \\
Near-above \([1,1.25)\)
& \(1{,}677\)
& \(0.139\)
& \(0.129\)
& \(1.127\) \\
High \((\ge 1.25)\)
& \(27{,}323\)
& \(0.003\)
& \(0.000\)
& \(1.919\) \\
\bottomrule
\end{tabular}
\end{table}

Table~\ref{tab:app_exp2_hard_transfer} reports the hard surrogate discrepancy
\begin{equation}
\Delta_{\rm tr}
=
\frac{
|I_{S|D}^{\rm Lap}-I_{S|D}^{\rm GW}|
}{\log n}.
\end{equation}
Random regular graphs have the smallest gap, confirming that the
Gaussian-wave surrogate is most accurate on expander-like graphs.  The gap is
moderate for SBM graphs, larger for ER and grid graphs, and largest for
barbell graphs.  The barbell case is a useful negative control: its
low-frequency Laplacian eigenvectors are dominated by the global bottleneck
and are therefore poorly approximated by independent Gaussian-wave
coordinates.

\begin{table}[htbp]
\centering
\scriptsize
\setlength{\tabcolsep}{4pt}
\caption{
Hard Gaussian-wave surrogate discrepancy.
Values are averaged over graph sizes, anchors, spectral dimensions,
quantization levels, and trials.
}
\label{tab:app_exp2_hard_transfer}
\begin{tabular}{lcc}
\toprule
Graph family
& Energy gap
& Signed gap \\
\midrule
Random regular
& \(0.050\)
& \(0.035\) \\
SBM
& \(0.106\)
& \(0.068\) \\
ER
& \(0.359\)
& \(0.268\) \\
Grid
& \(0.448\)
& \(0.696\) \\
Barbell
& \(1.240\)
& \(1.520\) \\
\bottomrule
\end{tabular}
\end{table}

We also evaluate the smoothed surrogate discrepancy for the expanded-cell majorant
\(H_{\eta,\alpha}^{+}\),
\begin{equation}
\Delta_{\rm tr}^{+}(\alpha)
=
\frac{
|I_{S|D}^{+,{\rm Lap}}(\alpha)
-
I_{S|D}^{+,{\rm GW}}(\alpha)|
}{\log n}.
\end{equation}
Table~\ref{tab:app_exp2_smooth_transfer} reports the family-level averages,
while Figure~\ref{fig:app_exp2_expanded_cell} shows the full distribution
over graph sizes, anchors, spectral dimensions, quantization levels, and
trials.  The same qualitative ordering persists under smoothing: random
regular and SBM graphs have small gaps, ER and grid graphs have intermediate
gaps, and barbell graphs have the largest gap.  This is consistent with
Corollary~\ref{cor:lap-hybrid-achievability}: actual Laplacian achievability requires Gaussian-wave anti-concentration
together with a distance-conditioned spectral-collision bound for the actual
Laplacian-energy coordinates.

\begin{table}[htbp]
\centering
\scriptsize
\setlength{\tabcolsep}{4pt}
\caption{
Smoothed Gaussian-wave surrogate discrepancy for energy coordinates.
}
\label{tab:app_exp2_smooth_transfer}
\begin{tabular}{lcc}
\toprule
Graph family
& \(\alpha/\eta=0.05\)
& \(\alpha/\eta=0.10\) \\
\midrule
Random regular
& \(0.086\)
& \(0.082\) \\
SBM
& \(0.141\)
& \(0.144\) \\
ER
& \(0.482\)
& \(0.478\) \\
Grid
& \(0.646\)
& \(0.693\) \\
Barbell
& \(1.602\)
& \(1.565\) \\
\bottomrule
\end{tabular}
\end{table}

\begin{figure}[htbp]
    \centering
    \includegraphics[width=0.98\columnwidth]{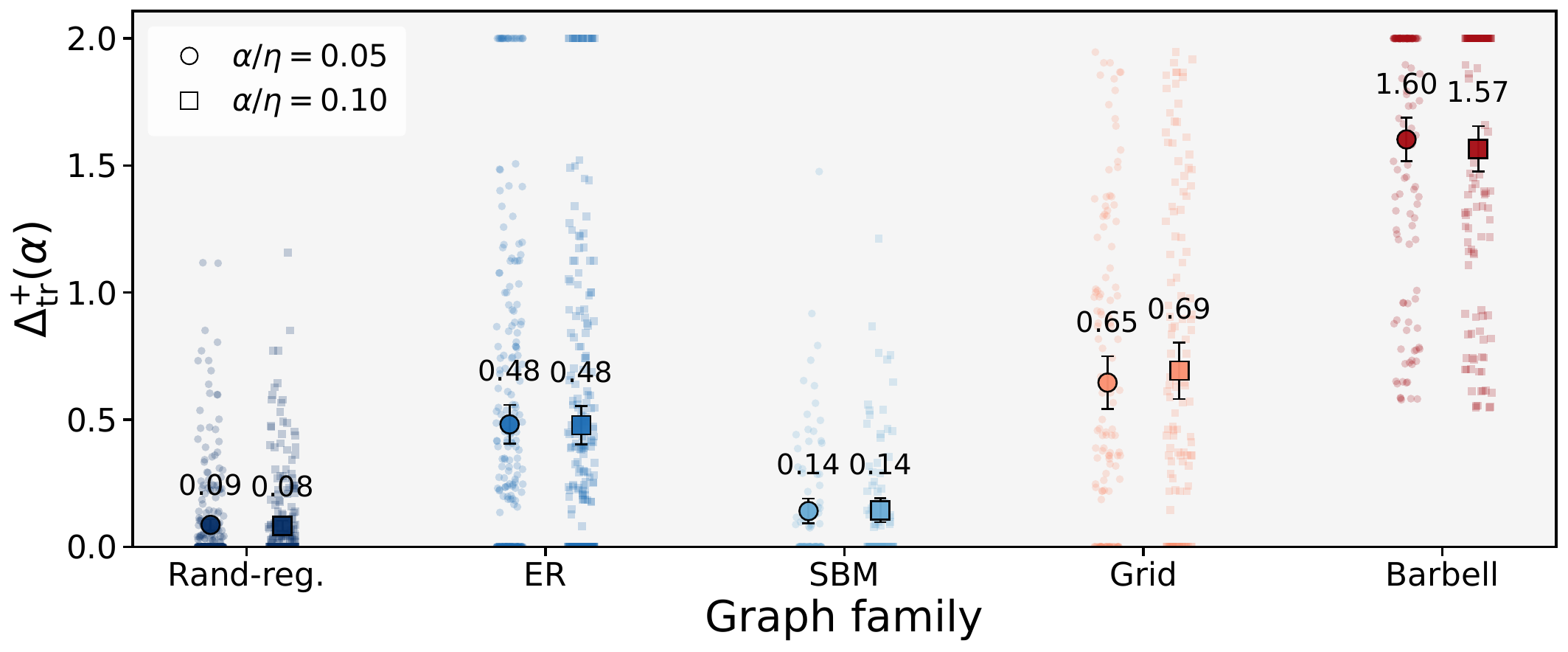}
    \caption{
    Expanded-cell surrogate-to-actual spectral-collision diagnostics for
    energy coordinates. Faint points show individual grouped configurations,
    and large markers with error bars show family-level means with \(95\%\)
    confidence intervals.
    }
    \label{fig:app_exp2_expanded_cell}
\end{figure}

\paragraph{Unified design-map diagnostics}
We further provide the binned design-map diagnostic corresponding to
Table~\ref{tab:exp3_design_diagnostics}.  To avoid overplotting from dense
configuration-level scatter points,
Figure~\ref{fig:app_exp3_design_heatmap_mean_error}
aggregates configurations with similar values of
\(B_{\rm conv}^{\Box}/\log n\) and \(I_H/\log n\).  The horizontal axis is the
normalized conservative box-budget diagnostic, the vertical axis is the
normalized hybrid collision information, and color indicates the mean empirical
localization error \(\operatorname{Err}^{*}(F)\) within each bin.  The dashed
reference lines mark the box-budget failure-side reference
\(B_{\rm conv}^{\Box}/\log n=1\) and the achievability-side threshold
\(I_H/\log n=1\).

\begin{figure}[htbp]
    \centering
    \includegraphics[width=0.98\columnwidth]{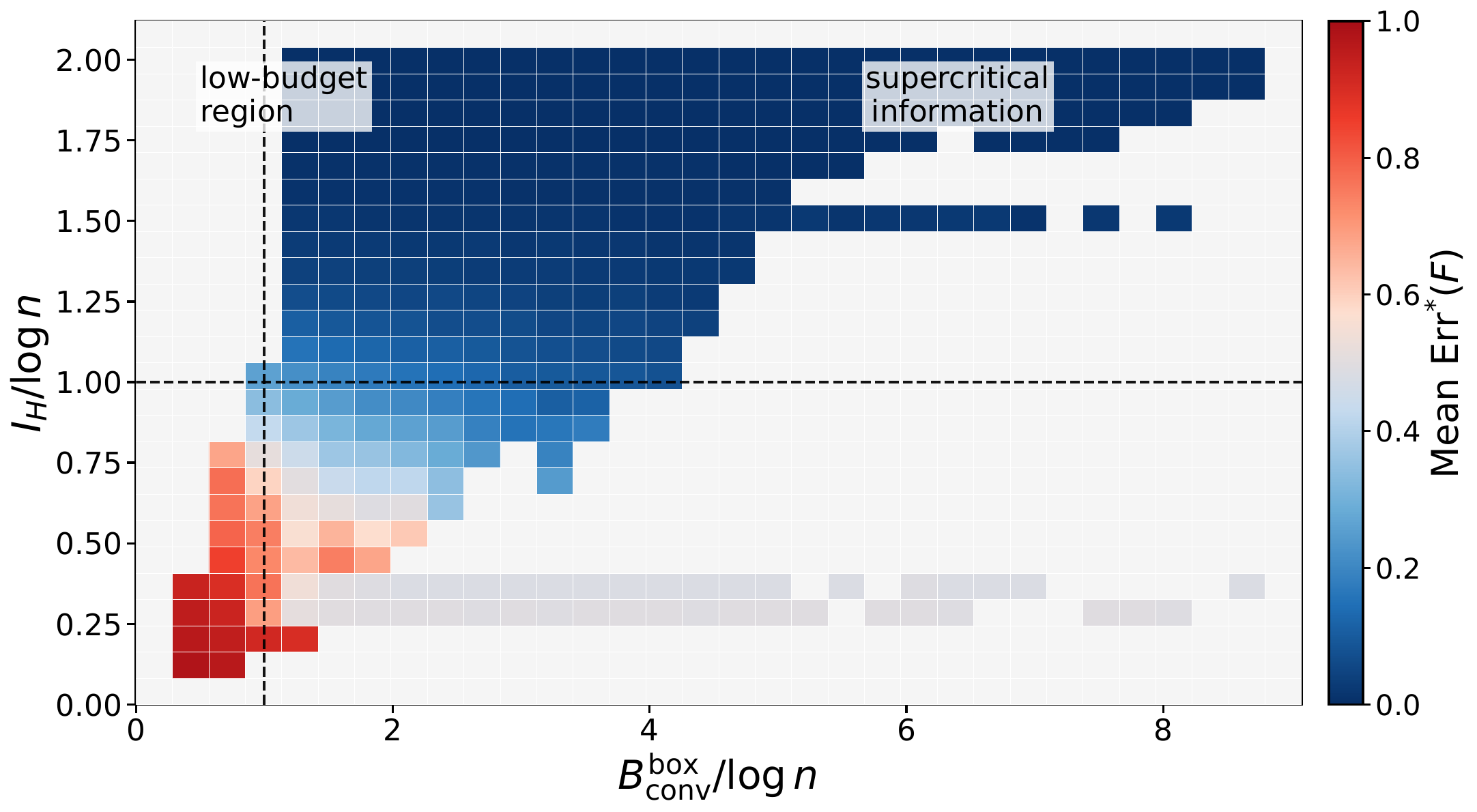}
    \caption{
    Binned unified graph-dependent design map for the Laplacian-energy hybrid
    positional encoding.  Each bin aggregates configurations with similar
    values of \(B_{\rm conv}^{\Box}/\log n\) and \(I_H/\log n\); color indicates
    the mean empirical localization error \(\operatorname{Err}^{*}(F)\) within the
    bin.  The dashed reference lines mark the box-budget reference
    \(B_{\rm conv}^{\Box}/\log n=1\) and the achievability-side threshold
    \(I_H/\log n=1\).
    }
    \label{fig:app_exp3_design_heatmap_mean_error}
\end{figure}

Table~\ref{tab:app_exp3_family_summary} reports the graph-setting breakdown
for the Laplacian-energy hybrid design diagnostics.  The random regular and
grid families are well organized by the information diagnostic and achieve
low average localization error.  The barbell graph is the most difficult
case: its mean information ratio is only \(0.309\), and no configuration
reaches \(\operatorname{Err}^{*}(F)\le0.1\).  This confirms that localization
is governed by graph-dependent collision geometry, not only by the raw choices
of \(k\), \(m\), and \(\eta\).

\begin{table}[htbp]
\centering
\scriptsize
\setlength{\tabcolsep}{3pt}
\caption{
Graph-setting summary for the Laplacian-energy hybrid design diagnostics.
}
\label{tab:app_exp3_family_summary}
\begin{tabular}{lrrrr}
\toprule
Graph setting & Mean err. & Median err. & Mean \(I_H/\log n\) & Succ. \\
\midrule
Random regular, \(r=3\)
& \(0.045\) & \(0.000\) & \(1.816\) & \(0.902\) \\
Grid
& \(0.051\) & \(0.004\) & \(1.651\) & \(0.865\) \\
Random regular, \(r=6\)
& \(0.085\) & \(0.000\) & \(1.688\) & \(0.818\) \\
Random regular, \(r=10\)
& \(0.107\) & \(0.000\) & \(1.625\) & \(0.776\) \\
SBM
& \(0.123\) & \(0.002\) & \(1.522\) & \(0.741\) \\
ER, average degree \(6\)
& \(0.170\) & \(0.004\) & \(1.443\) & \(0.675\) \\
Barbell
& \(0.490\) & \(0.491\) & \(0.309\) & \(0.000\) \\
\bottomrule
\end{tabular}
\end{table}

\subsection{Additional details for the UD structural task probes}
\label{app:ud_localization_details}
\label{app:ud_localization_diagnostics}

\paragraph{Treebanks and graph construction}
The UD experiments use English-EWT, Chinese-GSD, Spanish-GSD,
French-GSD, and German-GSD.  Sentences with length outside \([6,80]\)
are removed.  For each treebank, we use at most \(3000\) training
sentences and evaluate on the full official development and test splits.

Each sentence is converted into an unlabeled undirected dependency-tree
skeleton.  Tokens are nodes, dependency arcs are treated as undirected
edges, and lexical forms, dependency labels, edge directions, and token
attributes are excluded.  This construction matches the encoding-level
setting of the paper: the probes test what syntactic geometry can be
recovered from positional codes alone, rather than from lexical or
label-specific information.

\paragraph{Surface-position selection protocol}
The original surface-position probe is used to select configurations.
The grid is
\begin{equation}
\begin{aligned}
k &\in \{1,2,4,8\},\\
m &\in \{2,4,8,16\},\\
\eta &= 0.25,
\end{aligned}
\end{equation}
with three random-anchor trials and three probe seeds.

For each treebank, the HybridEnergy configuration with the best
development-set NMAE on the surface-position probe is selected.  The
selected \((k,m)\) values for English, Chinese, Spanish, French, and
German are
\begin{equation}
(1,16),\quad (4,4),\quad (8,4),\quad (4,4),\quad (1,16),
\end{equation}
respectively, with \(\eta=0.25\).

The selected HybridEnergy surface-position probe improves over NoPE by
\(0.0055\) NMAE on average, corresponding to a \(2.08\%\) relative
reduction, and improves Kendall-\(\tau\) by \(0.118\).  Its average NMAE
improvements over Distance-only and SpectralEnergy are \(1.27\%\) and
\(0.33\%\), respectively.  These gains are modest, which is expected
because undirected dependency-tree structure only weakly determines
linear word order.

\paragraph{Confirmatory depth and pairwise-distance probes}
The dependency-depth and pairwise dependency-distance probes are added
as more direct structural probes.  No new hyperparameter search is
performed for these probes.  Instead, the HybridEnergy configurations
selected by the surface-position development protocol are frozen and
reused.  This makes the depth and pairwise-distance evaluation
confirmatory rather than another tuning round.

The depth probe predicts normalized distance to the dependency root, and
therefore measures node-level syntactic geometry.  The pairwise-distance
probe predicts a clipped bucket of tree distance from a pairwise
combination of two positional codes.  These two probes are closer to the
graph geometry used to construct the positional encodings, and therefore
serve as task-relevant structural tests on real dependency-tree graphs.

\paragraph{Configuration-level localization diagnostics}
Figure~\ref{fig:app_exp4_ud_summary} gives the configuration-level
diagnostic view supporting the UD surface-position results.  Higher
\(I_H/\log n\) regimes produce larger average gains over the NoPE
baseline, and the selected best-performing configurations occupy the
low-collision, higher-rank-recovery region.  This supports the main
paper's interpretation that the collision information is informative not
only for exact node localization, but also for structural probe
performance on real dependency trees.

\begin{figure}[h]
    \centering
    \includegraphics[width=0.95\columnwidth]{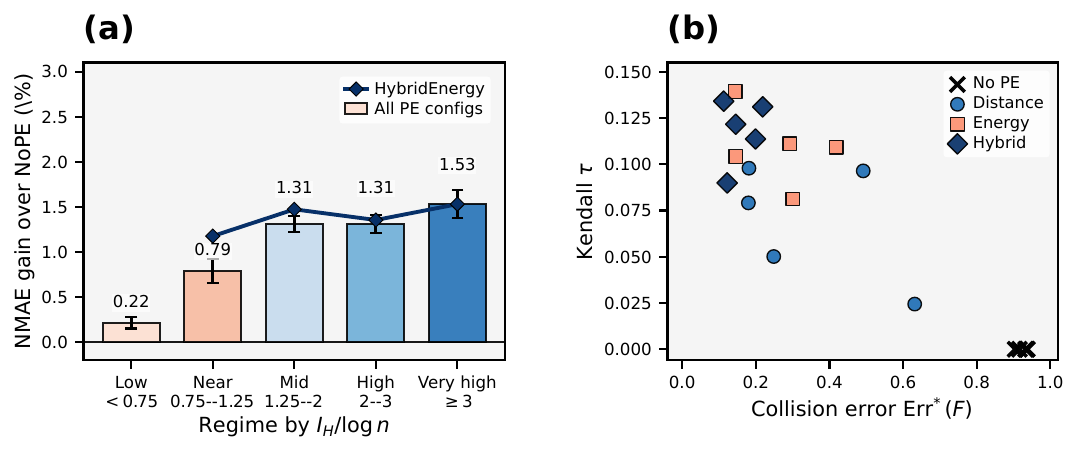}
    \caption{
    UD dependency-tree localization diagnostics.
    (a) NMAE gain over the NoPE baseline after grouping nontrivial PE
    configurations by \(I_H/\log n\). Bars summarize all PE
    configurations, and the line shows the HybridEnergy subset.
    (b) Best-configuration diagnostic map for each treebank and encoding
    family, comparing exact collision error \(\operatorname{Err}^{*}(F)\)
    with Kendall rank recovery.
    }
    \label{fig:app_exp4_ud_summary}
\end{figure}

\paragraph{PE-row derangement control}
Table~\ref{tab:app_ud_derangement_control_compact} reports the compact
derangement summary used to compute the alignment margins in the main
text.  For each sentence, positional codes are first computed from the
original dependency-tree skeleton.  The rows of the node-feature matrix
are then deranged within the same sentence while targets are kept
unchanged.  This preserves sentence length, feature dimension, PE
marginal statistics, and code budget, but removes the alignment between
each token and its own graph positional code.

For NMAE-based probes, \(\Delta_{\rm real}\) is the reduction over NoPE
in the real condition.  For the pairwise-distance probe,
\(\Delta_{\rm real}\) is the macro-F1 gain over NoPE.  The alignment
margin \(\Delta_{\rm align}\) subtracts the corresponding
deranged-condition gain.  HybridEnergy has the larger
alignment-specific margin on all three probes, indicating that its
advantages are not explained solely by sentence-level PE marginal
statistics.

\begin{table}[h]
\centering
\scriptsize
\setlength{\tabcolsep}{1.8pt}
\renewcommand{\arraystretch}{1.05}
\caption{
Compact PE-row derangement control on UD.  Position uses the original
surface-position control; depth and pairwise distance use the frozen
structural-probe protocol.  S and H denote SpectralEnergy and
HybridEnergy.  For position and depth, the metric is NMAE, lower better;
for pairwise distance, the metric is macro-F1, higher better.  NoPE is
the real-condition baseline used to compute the gains; its null value is
identical or numerically negligible after derangement.
}
\label{tab:app_ud_derangement_control_compact}
\resizebox{\columnwidth}{!}{%
\begin{tabular}{llcccccc}
\toprule
Probe
& Metric
& Enc.
& NoPE
& Real
& Null
& \(\Delta_{\rm real}\)
& \(\Delta_{\rm align}\)
\\
\midrule
Position
& NMAE \(\downarrow\)
& S
& 0.26386
& 0.26021
& 0.26034
& 0.00364
& 0.00011
\\
Position
& NMAE \(\downarrow\)
& H
& 0.26386
& 0.26053
& 0.26151
& 0.00332
& \textbf{0.00097}
\\
\midrule
Depth
& NMAE \(\downarrow\)
& S
& 0.2252
& 0.1910
& 0.2257
& 0.0342
& 0.0347
\\
Depth
& NMAE \(\downarrow\)
& H
& 0.2252
& \textbf{0.1748}
& 0.2257
& \textbf{0.0504}
& \textbf{0.0509}
\\
\midrule
Pairwise dist.
& Macro-F1 \(\uparrow\)
& S
& 0.0769
& 0.5748
& 0.1829
& 0.4979
& 0.3919
\\
Pairwise dist.
& Macro-F1 \(\uparrow\)
& H
& 0.0769
& \textbf{0.7911}
& \textbf{0.2029}
& \textbf{0.7142}
& \textbf{0.5881}
\\
\bottomrule
\end{tabular}%
}
\end{table}

\paragraph{Additional notes on UD design ablations}
The averaged UD ablation table is reported in the main text because it
is visually compact and directly supports the experimental narrative.
Here we clarify its interpretation.

The anchor block isolates the effect of structural reference-point
selection.  Root and farthest anchors both improve over random anchors,
and farthest anchors give the strongest surface-position recovery and
the lowest localization error.  This is consistent with the role of
anchor profiles in reducing distance-bucket ambiguity.

The signed-coordinate block separates within-tree collision control from
cross-sentence probe stability.  Signed Laplacian coordinates reduce
collisions and increase \(I_H/\log n\), but they do not improve probe
accuracy.  This suggests that exact within-tree separability is not
sufficient for cross-sentence structural prediction: the coordinate
system must also be stable across different dependency trees.

The quantization block verifies the expected diagnostic trend.  Coarser
\(\eta\) reduces \(I_H/\log n\) and mildly increases
\(\operatorname{Err}^{*}(F)\).  Probe performance remains stable for
\(\eta\leq 0.5\) and weakens mildly at \(\eta=1.0\), indicating that the
UD probes are not overly sensitive to small changes in quantization once
the code remains in a sufficiently resolved regime.

\end{document}